\documentclass[11pt]{article}
\usepackage{helvet}

\usepackage{fancyhdr}
\usepackage{lipsum} 

\usepackage{amssymb,amsbsy,amsfonts,amsmath,amsthm,xspace}
\usepackage{mathrsfs}
\usepackage{courier}
\usepackage[utf8]{inputenc}
\usepackage{tikz}
\usetikzlibrary{calc, matrix, arrows, shapes, positioning, fit, shapes.misc, shapes.geometric, calc, decorations.pathreplacing}
\usepackage{booktabs}
\usepackage{sectsty}
\sectionfont{\fontsize{11}{11}\selectfont}
\subsectionfont{\fontsize{11}{11}\selectfont}
\subsubsectionfont{\fontsize{11}{11}\selectfont}
\usepackage{graphicx}
\usepackage{caption}
\usepackage{setspace}
\usepackage{enumitem}
\usepackage{colortbl}

\usepackage{multirow}
\usepackage{epstopdf}
\usepackage{epsfig}
\usepackage{hyperref}
\usepackage{url}
\usepackage[toc,page]{appendix}
\usepackage{float}
\usepackage{natbib}
\usepackage{color}
\usepackage{verbatim}
\usepackage{authblk}
\usepackage{wrapfig}
\usepackage{algorithmic, algorithm}

\usepackage[normalem]{ulem}

\usepackage[margin=1in]{geometry}

\theoremstyle{plain}
\newtheorem{thm}{Theorem}

\newtheorem{defi}{Definition}

\newtheorem{coro}{Corollary}

\newtheorem{rem}{Remark}
\newtheorem{ex}{Example}

    \newtheorem{ass}{Assumption}

\theoremstyle{plain}

\newcommand{\e}{\mathbb{E}}
\newcommand{\p}{\mathbb{P}}

\newcommand{\mcal}{\mathcal{M}}

\newcommand{\vcal}{\mathcal{V}}

\newcommand{\ocal}{\mathcal{O}}
\newcommand{\bR}{\mathbb{R}}

\newcommand{\mbone}{{\mathbb{I}}}

\newcommand{\ecal}{\mathcal{E}}

\newcommand{\norm}[1]{\Vert{#1}\Vert}

\newcommand{\E}{\mathbb{E}} 
\newcommand{\lambdaMin}{\lambda_{\min}} 
\newcommand{\opnorm}[1]{\left\| #1 \right\|_{\mathrm{op}}} 

\usepackage{tikz}
\usepackage{xparse}
\usepackage{tabularx}
\usepackage{rotating}
\usepackage{makecell}

\NewDocumentCommand{\xtriangle}{O{.65}}{%
    \tikz[baseline = -0.4ex,scale=#1]{%
        \foreach \x in{90,-30,210}%
        {%
            \fill[](\x:0.2)circle(1.5pt);%
        }%
        \draw[thin] (-30:0.2) -- (90:0.2) -- (210:0.2) -- cycle;%
    }%
}%

\NewDocumentCommand{\xline}{O{.5}}{%
    \tikz[baseline = -0.8ex,scale=#1]{%
        \foreach \x in{-90,30}%
        {%
            \fill[](\x:0.2)circle(1.5pt);%
        }%
        \draw[thin] (-90:0.2) -- (30:0.2);%
    }%
}%

\NewDocumentCommand{\xsquare}{O{.5}}{%
    \tikz[baseline=-0.5ex,scale=#1]{%
        \foreach \x in {45, 135, -135, -45}%
        {%
            \fill[] (\x:0.2828) circle (1.5pt); 
        }%
        \draw[thin] (45:0.2828) -- (135:0.2828) -- (-135:0.2828) -- (-45:0.2828) -- cycle;%
    }%
}%

\NewDocumentCommand{\xthreestar}{O{.6}}{%
    \tikz[baseline=-0.8ex,scale=#1]{%
        \pgfmathsetmacro{\armLength}{0.4 / sqrt(3)}
        \pgfmathsetmacro{\angleA}{30}
        \pgfmathsetmacro{\angleB}{-90}
        \pgfmathsetmacro{\angleC}{150}
        \draw[thin] (0,-.03) -- (\angleA:\armLength)   (0,-.03) -- (\angleB:\armLength)   (0,-.03) -- (\angleC:\armLength);
        \fill[] (\angleA:\armLength) circle (1.5pt);
        \fill[] (\angleB:\armLength) circle (1.5pt);
        \fill[] (\angleC:\armLength) circle (1.5pt);
        \fill[] (0,-.03) circle (1.5pt);
    }%
}

\NewDocumentCommand{\xtwoline}{O{.6}}{%
    \tikz[baseline = -0.8ex,scale=#1]{%
        \foreach \x in{30,-90,150}%
        {%
            \fill[](\x:0.2)circle(1.5pt);%
        }%
        \draw[thin] (150:0.2) -- (-90:0.2) -- (30:0.2);%
    }%
}%

	\def\E{\mathbb{E}}

	\def\P{{\mathbb{P}}}

	\def\bZ{\mathbf{Z}}

	\def\P{\mathbb{P}}
	\def\fm{\mathfrak{m}}

 \newcommand{\be}{\begin{eqnarray}}
\newcommand{\ee}{\end{eqnarray}}
\newcommand{\ba}{\begin{eqnarray*}}
	\newcommand{\ea}{\end{eqnarray*}}
\newcommand{\bei}{\begin{itemize}}
	\newcommand{\beiftnt}{\begin{itemize}\footnotesize}
		\newcommand{\eei}{\end{itemize}}

\newtheorem{theorem}{Theorem}[section]
\newtheorem{lemma}{Lemma}[section]
\newtheorem{corollary}{Corollary}[section]

\usepackage{titlesec}

\titlespacing*{\section}
{0pt} 
{10pt} 
{5pt} 
\titlespacing*{\subsection}
{0pt} 
{6pt} 
{3pt} 

\def\text#1{\mbox{\rm #1}}

\usepackage{hyperref}

\hypersetup{
  colorlinks   = true, 
  urlcolor     = blue, 
  linkcolor    = blue, 
  citecolor   = blue 
}

\usetikzlibrary{patterns}

\usepackage[normalem]{ulem}
\renewcommand{\mid}{\,\vert\,}
\usepackage{subfigure}

\begin{document}

\begin{center}

\textbf{\large  GRAND: Graph Release with Assured Node Differential Privacy}\\
\medskip

Suqing Liu\\
University of Chicago\\
\medskip

Xuan Bi and Tianxi Li\\
University of Minnesota, Twin Cities

\end{center}

\begin{abstract}
Differential privacy is a well-established framework for safeguarding sensitive information in data. While extensively applied across various domains, its application to network data --- particularly at the node level --- remains underexplored. Existing methods for node-level privacy either focus exclusively on query-based approaches, which restrict output to pre-specified network statistics, or fail to preserve key structural properties of the network. In this work, we propose GRAND (Graph Release with Assured Node Differential privacy), which is, to the best of our knowledge, the first network release mechanism that releases networks while ensuring node-level differential privacy and preserving structural properties. Under a broad class of latent space models, we show that the released network asymptotically follows the same distribution as the original network. The effectiveness of the approach is evaluated through extensive experiments on both synthetic and real-world datasets.

\end{abstract}

\section{Introduction}

Rapid technological advances have rendered complex networks pervasive --- ranging from the Internet to international trade --- yet these systems inherently present critical privacy challenges, exemplified by the need to protect identities in organ donation \citep{hizo2010attitudes,marcus2023anonymity} and HIV transmission studies \citep{little2014using, abadie2021privacy}. Preemptively safeguarding such networks is uniquely difficult compared to conventional data sharing, as intricate structural dependencies, noise, and vast scales render simple mathematical models insufficient \citep{karwa2017sharing,hehir2022consistent} while imposing rigorous computational and communication efficiency requirements on privacy-preserving algorithms \citep{guo2023privacy}.

Differential privacy \citep[DP;][]{dwork2006differential} has established itself as a standard for strict privacy protection, particularly for vulnerable outliers with extreme values. Mechanisms such as the Laplace \citep{dwork2006calibrating,dwork2014algorithmic} and exponential \citep{mcsherry2007mechanism} methods are widely employed in learning tasks \citep{lei2018differentially,soto2022shape,lin2023differentially,xue2024optimal,ma2025locally}. Furthermore, extensive research into DP variations—including relaxed \citep{abadi2016deep,cai2019cost}, local \citep{rohde2018geometrizing}, random \citep{hall2012random}, Kullback-Leibler \citep{wang2016average}, and Gaussian \citep{dong2022gaussian} frameworks—has facilitated broad application across academia and industry \citep{kaissis2020secure,hie2018realizing,han2020breaking,santos2020differential,kenthapadi2018pripearl}.

Differential privacy in complex network data analysis has also been carefully explored \citep[e.g.,][]{karwa2016inference,rohde2018geometrizing,chang2024edge}. 
However, as network data typically show structures distinct from that of Euclidean data, the preservation of differential privacy for network data entails unique challenges. First, differential privacy preservation may come at the cost of altering a network's underlying utilities. This includes, but not limited to, crucial network properties, such as degree distributions, centrality and potentially community structures. Many existing methods may inadvertently alter the underlying distribution of edges and degrees. Meanwhile, edge-level modifications, such as edge addition or deletion, may also change node-level properties. In order to preserve a certain network property, therefore, many existing methods focus on statistic release rather than network release -- that is, the release of certain summary statistics (e.g., degree distribution) with privacy guarantee, rather than an entire differentially private network that has all properties preserved. 

Second, the adoption of differential privacy may appear at different levels. On the one hand, many existing works consider edge differential privacy \citep{nissim2007smooth}, which protects privacy of individual edges. The adoption of edge differential privacy is natural and shows good theoretical properties \citep{mulle2015privacy,karwa2017sharing,fan2020asymptotic,hehir2022consistent,chang2024edge}. On the other hand, the investigation and development of node differential privacy \citep{hay2009accurate} is relatively rare. Node differential privacy guarantees that a released network should remain robust with the change of any node. However, the alteration of a node usually entails the alteration of all associated edges. The preservation of network properties under node differential privacy is therefore considered substantially more challenging \citep{kasiviswanathan2013analyzing,hehir2022consistent}. In particular, the release of entire networks that follow node differential privacy has rarely been investigated in previous literature except for a few NP-hard theoretical algorithms \citep{borgs2015private,borgs2018revealing,chen2024private}.

In this work, we propose a novel mechanism, \textbf{GRAND} (\textbf{G}raph \textbf{R}elease with \textbf{A}ssured \textbf{N}ode \textbf{D}ifferential privacy), to release networks that guarantees \emph{node differential privacy}. Moreover, the proposed method is shown to preserve statistical network properties under the general latent space models, which include many widely used statistical network models as special cases. To the best of our knowledge, our method is the first practical approach that can \textbf{(1) release a network (rather than a summary statistic), (2) ensure node differential privacy of the released network, and (3) is computationally feasible with utility preservation guarantees}.

The paper proceeds as follows: Section~\ref{sec:contribution} reviews existing literature and outlines our contributions, followed by formal definitions in Section~\ref{sec:definition}. The GRAND mechanism and its theoretical properties are detailed in Sections~\ref{sec:method} and \ref{sec:theory}, respectively. We evaluate performance via simulations in Section~\ref{sec:sim} and real-world applications in Section~\ref{sec:data}, concluding with discussions in Section~\ref{sec:discussion}.

\section{Network Differential Privacy and Our Contributions} 
\label{sec:contribution}

\noindent\textbf{\emph{Node differential privacy vs. edge differential privacy.}}
Broadly speaking, existing works on network differential privacy can be categorized into two branches \citep{abawajy2016privacy,li2023private}. 
One branch is edge differential privacy \citep[EDP;][]{nissim2007smooth}. Intuitively, an EDP mechanism ensures that one cannot infer the pairwise relation between any two nodes in the network. Many methods have been proposed and discussed regarding the protection of EDP, which protect the privacy of individual edges. 
This can be achieved through mechanisms such as randomized responses \citep{mulle2015privacy,hehir2022consistent}, method of moments \citep{chang2024edge}, or noise injection \citep{karwa2016inference,fan2020asymptotic}.
The other branch is node differential privacy \citep[NDP;][]{hay2009accurate}. An NDP mechanism (to be formally defined later) ensures  that one cannot infer an individual node in the network.

Notice that, a node's information includes its pairwise relations to other nodes, and hence node differential privacy is more strict than edge differential privacy. As a result, node-level privacy protection is considerably more challenging than edge-level privacy. Applying edge-level DP methods to node-level settings usually leads to degenerate results \citep{kasiviswanathan2013analyzing,hehir2022consistent,chang2024edge}. 

However, we advocate that preserving node differential privacy can be critical in many scenarios, because nodes can represent individual users, customers, and patients, whose identities are sensitive and vulnerable to de-anonymization. Encoding node differential privacy, while preserving the underlying network's properties is the key motivation of this paper. 


\noindent\textbf{\emph{Network release vs. statistic release.}}
The adoption of differential privacy in network analysis, either at the node level or at the edge level, involves introducing random noise, which may lead to unintended consequences. First, existing differentially private mechanisms that are designed for Euclidean data may inadvertently change the underlying structure or properties of networks. For example, adding or deleting edges may lead to the alteration of the degree distribution, transitivity, and the number of components. In general, releasing a differentially private network while preserving all original properties can be very challenging. 

One viable approach is statistic release. That is, rather than releasing a network, one releases a particular network summary statistic that satisfies differential privacy, while preserving the accuracy of this particular statistic. This line of works include, but not limited to, the release of triangle counts \citep{liu2022collecting}, node degrees \citep{sivasubramaniam2020differentially,yan2021directed,wang2022two}, and network centralities \citep{laeuchli2022analysis} with the edge differential privacy guarantee, and the release of triangle counts \citep{blocki2013differentially}, the number of connected components \citep{kalemaj2023node,jain2024time}, edge density \citep{ullman2019efficiently}, and degree distributions \citep{day2016publishing,raskhodnikova2016lipschitz, macwan2018node} with the node differential privacy guarantee.

Nevertheless, statistic release has many critical drawbacks. For example, one has to design different methods for different statistics, and each such task can be sophisticated. Moreover, for many commonly used network statistics, such as centralities and closeness, no release mechanisms are currently available.  More importantly, statistic release completely eliminates the possibility of analyzing multiple network statistics simultaneously. This is because when each statistic is processed separately, the dependence between statistics would be lost. And it has been widely known that, given the complexity of network data, one usually needs inference of multiple network statistics jointly for informative analysis \citep{qi2024multivariate}.

With the aforementioned limitations, it is then evident that directly \emph{releasing a network} would be substantially more flexible for subsequent use. Ideally, if the released network is almost the same as the original network, known as an \emph{informative release} \citep{wasserman2010statistical}, we can use the privatized network in an arbitrary way in downstream tasks. In particular,  when the released network has most network properties maintained the same, it essentially subsumes most statistic release. Along this path, many have explored the novel ways of generating synthetic networks under edge differential privacy such as \cite{karwa2017sharing,qin2017generating,hehir2022consistent,guo2023privacy,chang2024edge}. For node differential privacy, the only known methods in the literature
\citep{borgs2015private,borgs2018revealing,chen2024private} are NP-hard, thus are only of theoretical interest. 

Moving beyond the prioritization of network release strategies, we also consider the use of \textit{data-perturbation} mechanisms (as opposed to synthetic network generation). 
Data-perturbation mechanisms inject noise into the original network.
A key feature is that it retains a bijective (one-to-one) correspondence between the anonymized nodes and the original individuals. This avoids the construction of completely artificial nodes as used in synthetic networks, giving substantial advantages with respect to data expandability, decentralizability and robustness (details in Section~\ref{sec:discussion}).  For this reason, we explicitly prioritize the data-perturbation paradigm.

\noindent\textbf{\emph{Our contributions.}} The preceding discussion provides the motivation for a \emph{data-perturbation} mechanism that \emph{releases networks} (rather than statistics) with \emph{node-level differential privacy} and guarantees of \emph{preserving network structures}. This is expected to be a crucial tool in practice, while such a tool is not yet available in the literature. 
In this paper, we introduce a computationally feasible method that achieves all the above objectives simultaneously. To the best of our knowledge, we are among the first to achieve this goal in the literature.

\section{Node Differential Privacy and Latent Space Models} \label{sec:definition}



\noindent\textbf{\emph{Notations.}} For any positive integer $K$, we denote the set $\{1, \ldots, K\}$ by $[K]$. Consider a network of $N$ individuals (nodes), where the edges represent interactions or relationships between nodes, such as friendship, collaboration, or transactions \citep{newman2018networka}. Denote the node set by $[N]$. We assume the network is unweighted and undirected. The network can be uniquely represented by an $N\times N$ adjacency matrix $A$ where $A_{ij}=1$ if individuals $i$ and $j$ are connected, and $A_{ij}=0$ otherwise. Denote the node and edge sets of network $A$ by $\vcal(A)$ and $\ecal(A)$, respectively.
For any vector $z\in\bR^d$, we use $\norm{z}$ to denote its Euclidean norm.

\subsection{Node Differential Privacy} 

Suppose that the network information is sensitive (e.g., a cancer patient network) and many nodes in $A$ may not be willing to release their identity or even their presence in the network. Private information may still be vulnerable to leakage if one simply releases the network to the public, even with de-anonymization \citep{narayanan2008robust}. In many cases, one may infer a node's identity according to certain connectivity patterns of a de-anonymized network. Alternatively, a node's identity may be revealed due to adversarial attacks. 

We adopt differential privacy \citep{dwork2006differential} as the standard for privacy protection. One commonly used definition of is $\varepsilon$-differential privacy, where a parameter $\varepsilon$, referred to as the \emph{privacy budget}, is prespecified, such that a small $\varepsilon$ indicates a tight budget, associated with a stringent privacy protection protocol. Technically, $\varepsilon$-differential privacy requires that the alteration of any entry in the original dataset only leads to a small change of the output data's distribution, quantified by the ratio of distributions before and after the alteration being bounded by $e^{\varepsilon}$.

On network data, however, the definition of one data entry may be construed differently (e.g., as a node or as an edge). There have been multiple ways to define differential privacy on network data, and a detailed review can be found in \cite{jiang2021applications}. Specifically, we focus on node differential privacy (NDP), formally defined as below.

\begin{defi}[Node Differential Privacy] \label{def:existing_NDP}
Let $\varepsilon>0$ be the privacy budget. A network data releasing mechanism $\mcal$ satisfies \textbf{node $\varepsilon$-differential privacy} if for any measurable set $\Psi$ of the network sample space, we have
\begin{equation}\label{eq:DF}
\p(\mcal(A)\in \Psi) \le e^{\varepsilon} \p(\mcal(A')\in \Psi)
\end{equation}
for any two networks $A$ and $A'$ with adjacency node sets $\vcal(A)$ and $\vcal(A')$: $A$ and $A'$ are only different in one row and the corresponding column.
\end{defi}

Analogously, the edge level differential privacy is defined by requiring \eqref{eq:DF} for any two networks $A$ and $A'$ with only one difference in $\ecal(A)$ and $\ecal(A')$. The above definitions have been also discussed in several previous frameworks \citep{hay2009accurate,kasiviswanathan2013analyzing,karwa2016inference,imola2021locally,hehir2022consistent,guo2023privacy, chang2024edge}. Intuitively, one should not be able to infer an individual's identity from the released network data under the definition of node level differential privacy. In contrast, edge level privacy refers to protection against inferring the existence of an edge. It is not difficult to see that node-level DP provides strictly stronger privacy guarantees than edge-level DP \citep{kasiviswanathan2013analyzing}. Technically, for an individual $i \in \vcal(A)$, NDP protects $i$'s relationship with all other individuals $j$'s in the network, $j \in \vcal(A)$, $j \neq i$. In other words, the output $\mcal(A)$ should remain ``similar" when an entire row and column of $A$ has been changed (to an arbitrary degree). Our goal is to design \emph{a network-releasing mechanism $\mcal$} satisfying the NDP definition above. Meanwhile, the released network $\mcal(A)$ would provably preserve network properties of the original $A$ under the privacy budget.

\subsection{General Latent Space Models} \label{secsec:framework}

We now proceed to introduce a class of statistical models for network data, referred to as the \emph{latent space model}. We first introduce the design of our GRAND mechanism based on this model. And later we will demonstrate that our NDP guarantee still holds without the model. The idea of latent space has been widely used in random networks, and was first formally introduced by \cite{hoff2002latent}. 

\begin{defi}[General Latent Space Model]\label{defi:generic}
We say that $A$ is a network generated from the general latent space model if there exists a distribution $F$ (unknown) on $\mathbb{R}^d$ and a known symmetric generative function $W: \mathbb{R}^d\times \mathbb{R}^d\to[0,1]$ such that $A$ can be generated as below:
$$
Z_1,\ldots,Z_N \stackrel{\mathrm{i.i.d.}}{\sim} F, \quad\quad A_{ij}\stackrel{\mathrm{ind.}}{\sim}\text{Bernoulli}(W(Z_i,Z_j)), \quad i>j.
$$
\end{defi}
\noindent Here $Z_1,...,Z_N$ are latent vectors corresponding to nodes $1,\ldots,N$, respectively, and $d$ denotes the dimension of the latent space. The above model encapsulates a series of popular models as special cases. We illustrate a few examples here.

\begin{ex}[Inner Product Latent Space Model \citep{hoff2002latent}]\label{example:hoff} We say that $A$ is a network generated from an inner product latent space model if there exist distributions $F_X$ on $\mathbb{R}^d$ and $F_{\alpha}$ on $\mathbb{R}$, such that $A$ is generated as follows:
$$
(X_1, \alpha_1),\ldots,(X_N,\alpha_N) \stackrel{\mathrm{i.i.d.}}{\sim} F_X\times F_{\alpha}, \quad\quad A_{ij}\stackrel{\mathrm{ind.}}{\sim}\text{Bernoulli}( W((X_i,\alpha_i),(X_j,\alpha_j))), \quad i>j,
$$ 
\begin{equation}\label{eq:LSM}
W((X_i,\alpha_i),(X_j,\alpha_j)) = \sigma(X_i^\top X_j+\alpha_i+\alpha_j)
\end{equation}
and $\sigma$ is the sigmoid function. Taking $Z_i = (X_i, \alpha_i)$, this model is a special case of the general latent space model.
\end{ex}

\begin{ex}[Random Dot Product Graph \citep{young2007random}]\label{example:RDPG} We say $A$ is a network generated from a generalized random dot product graph (RDPG) if there exists a distribution $F$ on $\mathbb{R}^d$, such that $A$ follows the generative procedure below:
$$
Z_1,\ldots,Z_N \stackrel{\mathrm{i.i.d.}}{\sim} F, \quad\quad A_{ij}\stackrel{\mathrm{ind.}}{\sim}\text{Bernoulli}( Z_i^\top Z_j ), \quad i>j,
$$
where the support of $F$ guarantees that the generalized inner product is a probability. 
\end{ex}
The RDPG model can be further modified following \cite{rubin2022statistical} resulting in the so-called generalized RDPG (gRDPG). Many other widely used random network models in literature are special cases of the above models. For example, the stochastic block model (SBM) \citep{holland1983stochastic} and its variants \citep{karrer2011stochastic,airoldi2008mixed,sengupta2018block,li2018hierarchical,jin2019optimal,li2023community}, as well as various $\beta$-models \citep{chatterjee2011random,chen2021analysis} can be seen as special cases of the inner product latent space model. 


\begin{rem}
In most of the latent space models, the latent vectors are identifiable only up to some operator $\mathcal{O}^d$, depending on the specific format of $W$. For example, under the inner product latent space model or the RDPG, $\mathcal{O}^d$ can be any $d\times d$ orthogonal transformation. Such operators are usually not crucial when using these models, and would not result in difficulties in our analysis. For notational simplicity, in this paper, when we discuss the recovery of latent vectors or their distributions, the recovery can be up to such an unidentifiable operator.
\end{rem}


\begin{rem}
In the literature, some models introduce an additional parameter, which formulates the density of the network. However, the definition of this parameter differs by models; for instance, it appears as a scaling parameter in the RDPG and gRDPG frameworks \citep{athreya2021estimation, rubin2022statistical} and as the intercept term in the inner product latent space model \citep{li2023statistical}. To maintain a uniform format of our framework and focus only on the privacy aspect of networks, we will not introduce such a parameter. Our theory can be readily extended to incorporate such a parameter and we will briefly mention the corresponding minor changes in Section~\ref{sec:theory}.
\end{rem}

\section{The Proposed Method} \label{sec:method}

\subsection{The Prototype Node DP Mechanism: An Oracle Scenario} \label{secsec:DIP}

We introduce the principle of our design regarding how it may preserve network properties. The prototypical idea is to (i) acquire a latent vector for each node, (ii) privatize the latent vectors while preserving their distribution, and (iii) generate a private network through using the private latent vectors. The private network is expected to have the same network properties as the original one, since share the same latent vector distribution.

Under the latent space model in Definition~\ref{defi:generic}, recall that the latent vectors satisfy $Z_i = (Z_{i1},\ldots,Z_{id})^\top \in \mathbb{R}^d$, $i \in [N]$ and are an i.i.d. sample from $F$. To introduce our prototype design, we assume that we know two pieces of \emph{oracle} information.
\begin{itemize}
    \item One is the joint cumulative distribution function (CDF) $F$ of the true latent vectors. From $F$, we can also derive all conditional distributions. Let $F^{l \mid 1:(l-1)}(\cdot\mid Z_{i1},\ldots,Z_{i,l-1})$ be the conditional CDF of $Z_{il}$ given $(Z_{i1},\ldots,Z_{i,l-1})$. When the context is clear, we write $F^{l\mid 1:(l-1)}$ as $F^l$ for simplicity. In particular, let the marginal CDF of the first coordinate be $F^1$.
    \item The other is the true latent vectors $Z_i$, $i\in [N]$. Furthermore, let $\bZ_l = (Z_{1l},\ldots,Z_{Nl})^\top$ be the $l$-th latent vector of all nodes for $l\in [d]$. For each $l$, we can treat $\bZ_l$ as a sequence of univariate observations.
\end{itemize}

We next introduce a distribution-invariant privacy mechanism (DIP) proposed in \cite{bi2023distribution} to perturb the univariate data to achieve differential privacy, which will then be applied to all coordinates by conditioning. Specifically, applying $F^1$ to $\bZ_1$, we note that $F^1(Z_{i1})$ follows $\text{Uniform}(0,1)$.
We then perturb $F^1(Z_{i1})$ by adding an independent noise $e_{i1} \stackrel{\mathrm{i.i.d.}}{\sim} 
\text{Laplace}(0,1/\varepsilon)$ for $i\in [N]$. Following the standard Laplace mechanism, we know that $F^1(Z_{i1})+e_{i1}$ satisfies $\varepsilon$-differential privacy \citep{dwork2006differential}. Let $G$ be the CDF of $F^1(Z_{i1})+e_{i1}$ whose expression can be exactly obtained via convolution. Applying $G$ to $F^1(Z_{i1})+e_{i1}$ also results in a uniform random variable. The privatization mechanism $\fm^1$ for $\bZ_1$ can be written as
$$\tilde{Z}_{i1} \equiv \fm_1(Z_{i1}) =(F^1)^{-1}\circ G(F^1(Z_{i1})+e_{i1}), \quad i=1,\ldots,N,$$ 
where $\circ$ denotes function composition. We can see that $\tilde{Z}_{i1}\stackrel{\mathrm{i.i.d.}}{\sim} F^1$. Meanwhile, $\tilde{Z}_{i1}$ is differentially private \citep{dwork2006differential}.


For $l \ge 2$, we apply the same strategy to privatize each latent vector sequentially, using probability chain rule \citep{schum2001evidential}. Let $F^l(\cdot) := F^l(\cdot\mid Z_{i1},\ldots,Z_{i,l-1})$ and $\tilde{F}^l(\cdot) := F^l(\cdot\mid \tilde{Z}_{i1},\ldots,\tilde{Z}_{i,l-1})$.
\be \label{eq:mv_DIP}
\tilde{Z}_{il} \equiv \fm_{l}(Z_{il}\mid\tilde{Z}_{i1},\ldots,\tilde{Z}_{i,l-1}) = (\tilde{F}^l)^{-1}\circ G(F^l(Z_{il})+e_{il}), 
\ee
where $\fm_{l}(\cdot)$ denotes the privatization
process for $\bZ_l$, $l = 2,\ldots,d$. Notice that, since $Z_i \stackrel{\mathrm{i.i.d.}}{\sim} F$, we have $F^l(Z_{il}) \stackrel{\mathrm{i.i.d.}}{\sim} \text{Uniform}(0,1)$. On the other hand, the use of $(\tilde{F}^l)^{-1}$ in \eqref{eq:mv_DIP} is to guarantee differential privacy of the preceding variables and preserve that $\tilde{Z}_i \stackrel{\mathrm{i.i.d.}}{\sim} F$.
For notational simplicity, we write
\be \label{eq:dip_general}
\tilde{Z}_i = \fm(Z_i;\mathbf{e}_i,F),
\ee
where $\tilde{Z}_i = (\tilde{Z}_{i1},\ldots,\tilde{Z}_{id})^\top$, $\mathbf{e}_i = (e_{i1}, \ldots, e_{id})^\top$, and $\fm(\cdot)$ denotes the sequential application of $\fm_1,\ldots,\fm_d$. Here $\tilde{Z}_i$ is the differentially private perturbation of $Z_i$. 

Subsequently, we can follow Definition~\ref{defi:generic} and use the same generative function $W$ to generate the private adjacency matrix $\tilde{A}$. Then given the fact that $\tilde{Z}_i$ follows $F$ and that $W$ is known, $\tilde{A}$ should follow the same distribution as $A$. This ensures that network properties of $\tilde{A}$, especially those that can be presented as summary statistics, remains the same. Meanwhile, as $\tilde{Z}_i$ satisfies $\varepsilon$-differential privacy, $\tilde{A}$ follows NDP.

The above procedure describes the high-level idea of the proposed method. 
However, in practice, we do not know either $F$ or $\{{Z}_i\}$, thus these have to be estimated from data. The crux of this lies in the privacy requirement. This privacy requirement becomes nontrivial for network data, which requires special designs to be introduced next.

\subsection{The Proposed Network Release with Node-wise Estimation} \label{secsec:TNR}

Our goal is to estimate the latent distribution CDF $F$ and the latent vectors $Z_i$'s for the prototype procedure. It seems natural to consider the following procedure: we estimate the latent vectors by a standard network estimation method and then estimate the CDF from the estimated latent vectors, and these will be used as the ``plug-in" objects in the prototype procedure. Nevertheless, such a straightforward approach will not meet the DP requirement for the following three reasons.
\begin{enumerate}
    \item \emph{Privacy spill-over effect on networks}: In contrast to standard multivariate data where instances are assumed to be i.i.d. and many processes can be applied to each instance independently, estimating network models involves the information of pairwise relations. That means, the estimation of one node $i$ relies on the information of its relations with all other nodes $i'\ne i$. Thus, when a standard estimation method is used, the resulting estimator of $Z_i$ will also contain (private) information of other nodes, for which the DP mechanisms may fail. In other words, the alteration of any node $i$ will lead to the alteration of not only $Z_i$, but also other $Z_{i'}$'s, such that the released output may not satisfy the probability bound \eqref{eq:DF} in Definition \ref{def:existing_NDP}. This is, in our view, the major challenge in handling node-level DP in network data. 
    \item In our prototype procedure (e.g., see \eqref{eq:mv_DIP}), the outermost layer of operation depends on the inverse CDF. This is the final layer of operation without additional privacy protection. If the estimated CDF is based on the nodes to be released, the differential privacy of these nodes may not be guaranteed..
    \item In practice, fitting any network models almost always requires additional tuning or model selection. Rigorously speaking, once such procedure involves the to-be-released data, DP cannot be guaranteed.
\end{enumerate}

To rectify these challenges, we assume that we have a network $A$ between $N=n+m$ nodes available. We would like to hold out the network of $m$ nodes while privatizing and releasing only the network between the remaining $n$ nodes. By doing this, we will be able to use the hold-out network as the reference to obtain estimates about $F$ and $Z_i$'s and also for tuning if needed while ensuring the node DP for the released network between the $n$ nodes. Figure \ref{fig:obs_adjmat} gives a high-level flowchart demonstrating the proposed method, whose details will be introduced soon. Our data splitting practice is also seen in the graph neural network literature \citep[e.g.,][]{kipf2016semi}, where a sub-network is used for model training, and another sub-network is used for testing. 
In statistics, the splitting strategy is also used by \cite{chen2014network}.

Consider a network of size $N=n+m$. Without loss of generality, we assume the first $n$ nodes are designated for release. The adjacency matrix $A \in \{0,1\}^{(n+m) \times (n+m)}$ can be partitioned into blocks $A^{kr}$ for $k,r \in \{1,2\}$, where $A^{11}$ is an $n \times n$ matrix, $A^{12}$ is an $n \times m$ matrix, $A^{21}$ is an $m \times n$ matrix, and $A^{22}$ is an $m \times m$ matrix, as shown in Figure \ref{fig:obs_adjmat}. The proposed scheme will then have the entire adjacency matrix $A$ as the input, and a private version of $A^{11}$, namely $\tilde{A}^{11}$, as the output.

Assuming the network is generated from a general latent space model, we apply one of the standard model estimation procedures based on the specific assumed model to $A^{22}$. For example, if an inner product latent space model is assumed, one can apply the likelihood-based gradient descent algorithm of \cite{ma2020universal,li2023statistical}. If the RDPG is assumed, then the adjacency spectral embedding (ASE) of \cite{sussman2014consistent} can be used. This step would result in an estimate of the latent positions for the hold-out nodes, denoted by $\{\hat{Z}_i\}_{i=n+1}^{n+m}$. From this step, we can estimate the corresponding CDF $F$ for the latent vectors $\{\hat{Z}_i\}_{i=n+1}^{n+m}$. This can be achieved through following standard practices such as empirical CDFs, kernel smoothing, splines, or $K$-nearest neighbors \citep{wasserman2006all}. Denote this estimated $F$ by $\hat{F}$.

\begin{figure}[H]
  \centering
\includegraphics[width=\textwidth]{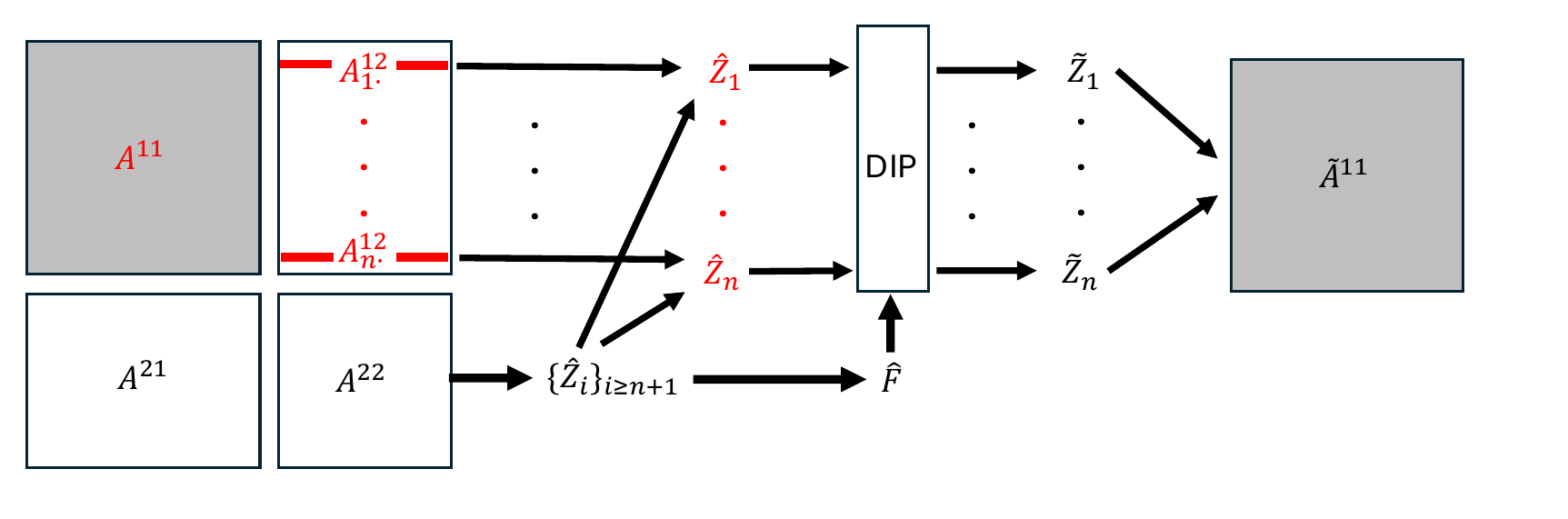}
\caption{\label{fig:obs_adjmat}The proposed privacy-preserving scheme illustrated: An $N \times N$ adjacency matrix $A$ is used as the input, and a private version of a $n \times n$ subnetwork $A^{11}$, namely $\tilde{A}^{11}$, is generated as the output. All privacy vulnerable quantities are colored in \textcolor{red}{red}. Each of the latent vectors of nodes $i \le n$ is separately estimated by a node-wise estimation with the help of hold-out estimates $\{\hat{Z}_i\}_{i> n}$. }
\end{figure}

Next, we introduce a \textbf{node-wise estimation procedure} to obtain an estimate for the latent vectors of the to-be-released nodes. This is a crucial step to ensure valid node DP for our method. Specifically, for each node $i \in [n]$, we use the information of the $i$th row of $A^{12}$, denoted by $A^{12}_{i*}$, following the latent space model. The estimation can be done by solving the following optimization problem involving the negative loglikelihood:
\be \label{eq:LF}
\hat{Z}_i = \mathop{\arg\min}_{Z_i\in \bR^d} \sum_{j=n+1}^{n+m} \{-A_{ij} \log W(Z_i,\hat{Z}_j) - (1-A_{ij}) \log (1-W(Z_i,\hat{Z}_j))\}, \quad 1\le i\le n.
\ee
Problem \eqref{eq:LF} is simply a pseudo-likelihood estimation procedure, by focusing only on one node at a time, while fixing the estimated latent vectors of the hold-out nodes. Note that the above estimation procedure ensures that, after fixing the estimation on the hold-out set, the estimation of $\hat{Z}_i$'s is separable for each $i\in [n]$. That is, the estimation of each $\hat{Z}_i$ in \eqref{eq:LF} uses information of its own row, thus \emph{it avoids the privacy spill-over effect}. 

For most commonly used $W$, this problem is easy to solve using standard algorithms. In many particular cases, the problem can be significantly simplified. For example, under the inner product latent space model, \eqref{eq:LF} becomes a simple logistic regression problem:
\be \label{eq:logistic}
\mathop{\arg\min}_{(X_i, \alpha_i)\in\bR^d} \sum_{j=n+1}^{n+m} \{-A_{ij} \log \sigma(X_i^\top\hat{X}_j + \alpha_i + \hat{\alpha}_j) - (1-A_{ij}) \log (1-\sigma(X_i^\top\hat{X}_j + \alpha_i + \hat{\alpha}_j))\}.
\ee

Other reasonable objectives can be used as alternatives to the negative log-likelihood in \eqref{eq:LF}, and these alternatives may further simplify the algorithm. For example, under the RDPG, it is easier to use the sum of squared errors as an M-estimation criterion rather than the likelihood. Thus the estimation becomes a node-wise OLS linear regression $\mathop{\arg\min}_{Z_i\in\bR^d}\sum_{j=n+1}^{n+m} (A_{ij} - Z_i^\top\hat{Z}_j)^2$.

With the estimated $\hat{F}$ and $\{\hat{Z}_i\}_{i\in [n]}$, we can apply the prototypical procedure in Section~\ref{secsec:DIP}. In particular, $\hat{F}$ does not utilize any information from the to-be-released nodes, thus can be used to replace $F$. Using \eqref{eq:dip_general}, for each $i \in [n]$, we have
\be \label{eq:TNR}
\tilde{Z}_i = \fm(\hat{Z}_i;e_i,\hat{F}).
\ee

Subsequently, we generate the private network $\tilde{A}^{11}$ based on $\{\tilde{Z}_i\}_{i\in [n]}$ following the Bernoulli. 
Algorithm \ref{alg:TNR} provides a summary of the proposed method.

Note that if one has node-level attributes $U_i$ in addition to the network, we can simply let $\hat{Z}_{i,\rm new} = (\hat{Z}_i,U_i)$ and equation \eqref{eq:TNR} can directly be applicable to $\hat{Z}_{i,\rm new}$.
Consequently, our method and the theoretical results can be trivially extended to handle networks with node attributes. We briefly outline this extension in Appendix \ref{sec:attributes}. 

\begin{algorithm}[h]
\caption{\label{alg:TNR}The GRAND mechanism by node-wise estimation}
\begin{algorithmic}[1]\label{algo:main}
{
\item[\textbf{Input:}] A network $A$ of size $N=n+m$ in which the first $n$ nodes are to be privatized, and the latent dimension $d$. 
\STATE Partition $A$ as in Figure~\ref{fig:obs_adjmat}.
\STATE Specify a general latent space model and the corresponding generative function $W$.
\STATE Estimate the hold-out latent vectors $\{\hat{Z}_i\}_{i=n+1}^{n+m}$ based on $A^{22}$ using a standard network model estimation procedure. 
\STATE Estimate the CDF $\hat{F}$ based on $\{\hat{Z}_i\}_{i=n+1}^{n+m}$.
\STATE Estimate the latent vectors $\{\hat{Z}_i\}_{i=1}^{n}$ by the node-wise estimation procedure \eqref{eq:LF} using $\{\hat{Z}_i\}_{i=n+1}^{n+m}$ and $A^{12}$.
\STATE Apply procedure \eqref{eq:TNR} to privatize latent vectors of the to-be-released nodes, producing $\{\tilde{Z}_i\}_{i=1}^{n}$.
\STATE Generate the private adjacency matrix $\tilde{A}^{11}$ from $\{\tilde{Z}_i\}_{i=1}^{n}$ using $W$.
\item[\textbf{Output:}] A privatized $\tilde{A}^{11}$ of size $n$.
}\end{algorithmic}
\end{algorithm}

\begin{rem}[Privacy of the Holdout Set]
Similar to the requirements discussed in \cite{bi2023distribution}, the holdout set should not be accessed, altered, queried, or released, and should be deleted immediately once Algorithm \ref{alg:TNR} is implemented. More details can be found in \cite{bi2023distribution}.

We acknowledge that the holdout set does not receive the same formal node DP guarantee as the released set. However, we argue that this concession is necessary to resolve a fundamental conflict inherent to network data: the \emph{privacy spill-over effect}. Intuitively, preserving network structures (e.g, by model fitting) requires exploiting pairwise relations. Conversely, node differential privacy mandates that every single node should be separable in privatization operations. These two objectives are intrinsically contradictory. 
Consequently, previous attempts in literature to achieve node DP on the entire graph have resulted in either a prohibitive loss of utility (network structures) or computational intractability.

Our method reconciles this tension by using the holdout set as a structural anchor. The proposed node-wise estimation (Equation \eqref{eq:LF}) decouples the released nodes in model fitting, thereby asymptotically breaking the trade-off between privacy and utility for the released network. We feel that the introduction of such a holdout set is a reasonable and necessary price to pay for the first computationally feasible mechanism that delivers both assured Node Differential Privacy and structural consistency. In Section \ref{sec:discussion}, we discuss possible future directions to further improve the privacy protection for the holdout set.
\end{rem}


\begin{rem}[Tradeoff for the Private Network Size]
It can be seen that $A^{11}$ is not used in estimation. This is to avoid the privacy spill-over issue. Therefore, from a statistical estimation perspective, the estimation is essentially based on a size-$(m+1)$ network rather than a size-$N$ network. which will be explicitly verified in our theory and numerical results. The statistical efficiency reduction from size $N$ to size $m$ is the price we pay for the node-level privacy. Meanwhile, the computational bottleneck of the algorithm is on the standard model estimation procedure in Step 3 based on $m$. The node-wise estimation is typically trivial and can be completely in parallel. Intuitively, we hope to maximize the released size $n$, subject to $m = N-n$ being large enough for model estimation. In practice, this choice may also depend on other factors, such as the target size of the released networks and computational resources. In our numerical studies, we use $n=0.5N$ for simplicity.
\end{rem}

\begin{rem}
It is not difficult to extend the procedure for directed networks. In the directed case, the latent space model would associate with each node a ``sending" latent vector and a ``receiving" latent vector. We will have to estimate both latent vectors for each node, using both $A^{12}$ and $A^{21}$. The rest operations would remain similar. 
\end{rem}


\section{Theoretical Guarantees}\label{sec:theory}

We now study the theoretical guarantees for the released network from our method. The very basic result is that our method gives differential privacy at the node level. 

\begin{thm}\label{thm:NDP}
    Suppose Algorithm~\ref{algo:main} is used to process the input network under the privacy budget $\varepsilon$. The resulting $\tilde{A}^{11}$  is node $\varepsilon$-differentially private.
\end{thm}
Note that, even though we use the latent space model to motivate our method, the node DP guarantee \textbf{does not rely on} the latent space model. That means, even if the network data is not generated by latent space models, our method still protects privacy.

However, the latent space model assumption is needed to guarantee the preservation of network properties.  
As implemented in Algorithm \ref{alg:TNR}, the estimation procedure for all $\tilde{Z}_i$'s involves the common $A^{22}$. That is, $\tilde{Z}_i$'s may have marginal dependence and can no longer be assumed to be i.i.d.
This is in contrast to the true $Z_i$'s. Therefore, in the rest of this section, \textbf{none of the theoretical results rely on the $\tilde{Z}_i$'s being i.i.d.}, which is one of the major theoretical contribution of this work.

As the focus of this section,  we will show that the privatized $\tilde{Z}_i$'s are guaranteed to asymptotically maintain the original distribution of $Z_i$'s. Such consistency indicates that the distributional properties of the original network are asymptotically preserved. In particular, as an indication of latent space consistency, we will show that the released private network $\tilde{A}$ can preserve motif counts of the original network for any given motif.

\subsection{Preservation of Network Properties} \label{sec:distn_consistency}

As discussed, to study the preservation of network properties, we will assume that the true network data is generated from a latent space model with regularity conditions.

\begin{ass}[Latent Space Model]\label{ass:LSM}
    Suppose the network $A$ is generated from the general latent space model (Definition~\ref{defi:generic}) with a fixed dimension $d$.
\end{ass}

\begin{ass}\label{ass:continuous}
The true latent space distribution satisfies the following properties.
    \begin{itemize}
    \item The joint CDF \( F \) is a continuous distribution on a compact support \( S \). The joint density \( f(x_1,\dots,x_d) \) of \( F \) is continuous and bounded by a constant \( C_{\mathrm{up}} > 0 \).
        \item If $d>1$, for each \( l=2, \dots, d \): The conditional CDF \( F_{l\mid 1:(l-1)}(x\mid u) \) is Lipschitz continuous in $(x,u)$ and strictly increasing in $x$. The marginal density \( f_{1:{(l-1)}}(u) \) is continuous, bounded and lower bounded by a positive constant.
\end{itemize}
\end{ass}


\begin{ass}[Kernel Smoothing for CDF Estimation]\label{ass:smoothedDIP}
    For each $1\le l\le d$, the estimation of the conditional CDF $F^{l\mid 1:(l-1)}$
    is constructed from $\{\hat{Z}_i\}_{i> n}$ by a kernel estimation using a bounded, Lipschitz continuous kernel \(K:\bR^{l-1}\to[0, \infty)\) (with a constant $L_{K}$) such that 
\(\int_{\bR^{l-1}} K(u)\,\mathrm du = 1\) and bandwidth $h = (\log m)^{-c}$ for a constant $c>0$,
defined as
\begin{equation}\label{eq:kernel-estimator}
  \hat{F}^{l \mid 1:(l-1)}(x \mid x_1,\dots,x_{l-1})
=  \frac{
    \sum_{i=1}^m \mathbf{1}_{\{\hat{Z}_{n+i,l} \le x\}}
      K\left(\tfrac{x_1 - \hat{Z}_{n+i,1}}{h}, \dots, \tfrac{x_{l-1} - \hat{Z}_{n+i,l-1}}{h}\right)
  }{
    \sum_{i=1}^m
      K\left(\tfrac{x_1 - \hat{Z}_{n+i,1}}{h}, \dots, \tfrac{x_{l-1} - \hat{Z}_{n+i,l-1}}{h}\right)
  }.    
\end{equation}
\end{ass}

\begin{ass}[Latent Embedding Error]\label{ass:embedding-concentration}
   Suppose $\hat{Z}_i$, $i \in [N]$ are the estimated latent vectors in Algorithm~\ref{algo:main}. There exists a constant $c>0$, such that (subject to the unidentifiable operator $\mathcal{O}^d$)
   \[\max_{1\le i\le N}\norm{\hat{Z}_{i}-Z_i} \le \delta_{m} =  o((\log m)^{-c})\]
    with probability approaching 1 as $m\to \infty$. 
\end{ass}

Recall that we use $\{\tilde{Z}_i\}_{i\in [n]}$  to generate the released network $\tilde{A}^{11}$. Given the latent vectors, the generating mechanism is always independent Bernoulli following Definition~\ref{defi:generic}. Therefore, to preserve network properties, the primary guarantee we need is that the resulting distribution of the $\tilde{Z}_i$'s should be roughly the same as the true distribution in a proper sense. 
\begin{thm}[Individual Latent Distribution Consistency]\label{thm:main-marginal}
    Suppose assumptions~\ref{ass:LSM}--\ref{ass:embedding-concentration} hold. Let $\tilde{F}^{(i)}$ be the joint CDF of the random vector $\tilde{Z}_i$. For any $1\le i \le n$, as $m\to\infty$,
    $$\sup_{x\in \bR^d}\big|\tilde{F}^{(i)}(x) - F(x)\big| \xrightarrow{\mathbb P} 0.$$
\end{thm}
On the one hand, we can say that the privatized $\tilde{Z}_i$ has the asymptotically correct distribution individually. On the other hand, due to the fact that $\{\tilde{Z}_i\}_{i\in [n]}$ cannot be treated as an i.i.d. sample from $F$ collectively, the latent distribution consistency is for each individual node $i$, rather than a uniform convergence that holds simultaneously for all $i \in [n]$. 
Fortunately, although the uniform convergence is not achievable, 
the privatized latent vectors still give a similar collective behavior for the resulting network. Specifically, we now show that the empirical CDF of the privatized latent vectors is asymptotically consistent.

\begin{thm}[Latent CDF Consistency]\label{thm:main-CDF}
    Suppose assumptions~\ref{ass:LSM}--\ref{ass:embedding-concentration} hold. Let $\tilde{F}_n$ be the empirical CDF of the privatized latent vectors based on $\tilde{Z}_{i}\in \bR^d$, $i\in [n]$ where
\[
\tilde{F}_n(x) = \frac{1}{n}\sum_{i=1}^n \mathbf{1}_{\{\tilde{Z}_{i1} \leq x_1, \dots, \tilde{Z}_{id} \leq x_d\}}.
\]
 Then $\tilde{F}_n$ converges uniformly to $F$ in probability. That is, as $n, m \to \infty$,
\[
\sup_{x \in \bR^d} \big|\tilde{F}_n(x) - F(x)\big| \xrightarrow{\mathbb P} 0.
\]
\end{thm}

Assumption~\ref{ass:embedding-concentration} for Theorems~\ref{thm:main-marginal} and \ref{thm:main-CDF} requires that the latent vectors can be accurately recovered for all nodes $i\in [N]$. For nodes in the hold-out set, this is not difficult by using the standard network estimation methods. For nodes $i\in [n]$, such a requirement needs additional study of the node-wise estimation procedure. Next, we show that under both the inner product latent space model and the RDPG, the proposed estimation satisfies \ref{ass:embedding-concentration}.

\begin{thm}[Error Bound for Node-wise Estimation]\label{thm:model-consistency}
Suppose $n = o(e^m)$. For a constant $c>0$, consider the following models.
\begin{enumerate}
    \item The network is generated from the inner product latent space model  and the latent vector set $\{\hat{Z}_i\}_{i\in [n]}$ is estimated by \eqref{eq:logistic}.
    \item The network is generated from the RDPG model and  the latent vector set $\{\hat{Z}_i\}_{i\in [n]}$ is estimated by the OLS linear regression.
\end{enumerate}
If the true latent distribution $F$ has a bounded domain and a positive definite second-order moment: $\Sigma = \e[ZZ^\top]\succeq \mu I_d$ for some constant $\mu>0$, while the hold-out estimates (up to an unidentifiable transformation) satisfy
$$\max_{n+1\le i\le n+m}\norm{\hat{Z}_i-Z_i} = \delta_m = o(1)$$
with probability approaching 1, 
the estimates $\{\hat{Z}_{i}\}_{i\in [N]}$ from Algorithm~\ref{algo:main} satisfy
\begin{equation}\label{eq:thm4}
\max_{i\in [n]}\norm{\hat{Z}_i-Z_i} \le C\max\left\{\delta_m, \frac{\sqrt{\log{m}}+\sqrt{\log n}}{\sqrt{m}}\right\}
\end{equation}
for a constant $C>0$ with probability approaching 1.
\end{thm}

With the above results, under both models, we know the assumption~\ref{ass:embedding-concentration} can be satisfied using commonly used estimation approaches.
\begin{coro}\label{coro:main-specific-models}
    Suppose Assumptions~\ref{ass:LSM}--\ref{ass:smoothedDIP} hold. In either of the following cases:
    \begin{enumerate}
        \item Under the inner product latent space model, suppose that the estimator of \cite{li2023statistical} is used to estimate the hold-out latent vectors  in Step 3 of Algorithm~\ref{algo:main}, and that the model satisfies the regularity conditions in \cite{li2023statistical}.
        \item Under the RDPG model, the method of \cite{rubin2022statistical} is used to estimate the hold-out latent vectors in Step 3 of Algorithm~\ref{algo:main}.
    \end{enumerate}
    The conclusions of Theorems~\ref{thm:main-marginal} and \ref{thm:main-CDF} hold. 
\end{coro}

\begin{rem}
    As mentioned before, we can adapt the main theoretical results to accommodate additional network sparsity parameters, but using different parameterizations for different models. For example, for the RDPG model, we can instead assume the connection probability to be $\rho_{n+m}Z_i^\top Z_j$ for some $\rho_{n+m}\to 0$ as $n, m \to \infty$, as used in \cite{athreya2017statistical}. The other assumptions on about $Z_i$'s can remain the same. Theorems~\ref{thm:main-marginal} and \ref{thm:main-CDF} still remain valid, as long as $m\rho_{m+n}\ge (\log{m})^{c'}$ for a sufficiently large constant $c'>0$ without significantly changing the standard density requirement \citep{athreya2017statistical}. The $\sqrt{m}$ term in \eqref{eq:thm4} will be changed to $\sqrt{m\rho_{n+m}}$.
\end{rem}



\subsection{Preservation of Network Moments}

In this subsection, we show that, as a result of the latent distribution consistency, the privatized $\tilde{A}$ maintains the moments/motif counts of the original network.

Network moments are commonly used to measure various aspects of network structures for inference tasks \citep{milo2002network,bickel2011method,bhattacharyya2015subsampling,maugis2020testing,levin2019bootstrapping,zhang2022edgeworth,qi2024multivariate}, as they allow comparison across networks of different sizes and node sets. More broadly, many other useful network statistics, though not directly expressed as network moments, can be written as functions of network moments of multiple motifs, such as the clustering coefficient. Therefore, the guarantee of network moment recovery can indicate the valid of many downstream analyses under NDP.

Following \cite{qi2024multivariate}, a graph $R$ is a subgraph of $A$, written as $R \subset  A$, if $\vcal(R) \subset \vcal(A) $ and  $\ecal(R) \subset \ecal(A)$.   Two graphs $R$ and $A$ are isomorphic, denoted by $R \cong A$, when there exists a bijective function $\phi$: $\vcal(R) \to \vcal(A)$ such that $(v_i,v_j)\in \ecal(R)$ if and only if edge $[\phi(v_i), \phi(v_j)] \in \ecal(A)$.   A motif refers to a (usually simple) graph, such as an edge ($\xline$), a 2-star/V-shape ($\xtwoline$), a triangle ($\xtriangle$), or a 3-star ($\xthreestar$),  which forms the building blocks of larger graphs. Here we denote a motif by $R$, with $|V(R)| = r$ representing the number of nodes. We focus exclusively on connected motifs. For a network $A$ and a motif $R$, the motif density of $R$ in $A$ is defined as the normalized number of subgraphs of $A$ that are isomorphic to $R$:
\begin{equation}
\label{eq:MotifCountp}
    X_R(A)=\big|\{S: S \subset A, S \cong  R\}\big|/\binom{n}{r} = \sum_{1\le i_1 < \cdots < i_r \le n}\mbone(A_{[i_1, \ldots, i_r]} \cong  R)/\binom{n}{r},
\end{equation}
in which \(A_{[i_1, \cdots, i_r]}\) is the subgraph of $A$ induced by nodes \(\{i_1, \cdots, i_r\}\). We call $X_R(A)$ the network moment of $A$ with respect to motif $R$. For example, when $R$ is an edge, then $X_R(A)$ is the edge density of $A$. Network moments are summary statistics of the whole network structures. Intuitively, under the latent space model, they are determined by the collective behavior of the latent vectors. And therefore, preservation of latent vector distributions should also indicate the preservation of network moments.

\begin{thm}[Consistency of Network Moments]\label{thm:main-moments}
    Let $R$ be a fixed motif on $r$ nodes that does not depend on $n$. Suppose $X_R(A^{11})$ is the motif density of $R$ in network $A^{11}$, as defined in \eqref{eq:MotifCountp}. Similarly, let $X_R(\tilde{A}^{11})$ be motif density of $R$ in the privatized network $\tilde{A}^{11}$ from Algorithm~\ref{algo:main}. When Assumptions~\ref{ass:LSM}--\ref{ass:embedding-concentration} hold, we have, as $n, m\to \infty$,
    \[X_R(A^{11})-X_R(\tilde{A}^{11}) \xrightarrow{\mathbb P} 0.\]
\end{thm}

It is easy to see that the claim can be extended to any function of network motif counts such as the clustering coefficient. Theorem~\ref{thm:main-moments} provides crucial convenience in applications. For example, one can compare multiple unmatchable networks by comparing the distribution of their moments \cite{qi2024multivariate}, which can be done privately using GRAND-privatized networks.

\section{Simulation Experiments}\label{sec:sim}

In this section, we introduce simulation experiments to evaluate the performance of the proposed method. In the experiments, we fix the node DP budget, and assess the extent to which the privatized network can preserve the structural properties of the original network. We consider the two commonly used network models, inner product latent space model and RDPG, as described in Section \ref{secsec:framework}.  For the inner product latent space model, our data generating method is a generalization of the mechanism of \cite{li2023statistical} for more heterogeneity.  We sample latent vectors $X_i$'s from a mixture of truncated Gaussian distributions and generate $\alpha_i$'s from a uniform distribution. We then rescale all $\alpha_i$'s together so the resulting network from the model would have the desired density according to our previous configuration. For the RDPG model, we sample latent vectors $Z_1,\ldots,Z_N \stackrel{\mathrm{i.i.d.}}{\sim} \text{Uniform}(0,1)^d$. The $Z_i$'s are then rescaled to ensure the desired network density. We always randomly sample half of the nodes to be the hold-out data, and privatize the network between the other half of the nodes.

\emph{Experiment configuration.} Under both models, we set certain key parameters as follows.
We vary the released size and hold-out size to be $n=m \in \{2000, 4000\}$, and network density $\rho \in \{0.025, 0.05, 0.1\}$. For example, for a network of size 2000, the expected average degree is 50, 100 and 200, respectively. Also, we set the dimension $d\in \{3,6\}$ and the privacy budget $\varepsilon \in \{1, 2, 5, 10\}$. Typically, $\varepsilon=1$ corresponds to a strong privacy protection in practice, while $\varepsilon=10$ gives much weaker protection.

\emph{Methods in comparison.} In addition to our proposed method, we include two additional benchmark methods for comparison. (a) The first one is naively applying the Laplace mechanism of \cite{dwork2006differential} to $\{\hat{Z}_i\}_{i \in [n]}$ and then generating the network from the resulting latent vectors (``\emph{Laplace}"). Following a similar argument as in Theorem~\ref{thm:NDP}, it also satisfies node DP, but offers no guarantee for structural preservation. The comparison with this benchmark illustrates how important it is to consider the latent distribution consistency in the design of the proposed method in practice. (b) We also consider a non-private method: We directly estimate the network models from the size $m$ network using a standard network estimation, rather than the node-wise estimation method. Specifically, we use the gradient descent model fitting of \cite{ma2020universal, li2023statistical} for the inner product model and the adjacency spectral embedding \citep{sussman2014consistent} for the RDPG. After that, we generate a new size-$n$ network from the estimated model without introducing any privatization step as the result (in this case, we have $m=n$). We call this the ``\emph{Hat}" network. The ``Hat" method corresponds to the scenario where we do not impose DP at all and only use a standard approach to estimate the network model and generate new data. Comparing the ``Hat" network with the original network gives us a non-private benchmark.
As discussed in Section~\ref{sec:method}, our model estimation before the privatization is essentially based on $m+1\approx m$ nodes. Therefore, the ``Hat" estimate can be seen as a comparable one with our method when $\varepsilon\to \infty$. The comparison between our method and ``Hat" can show the price of structural preservation we pay for the privacy.

\emph{Performance evaluation metrics.}
We measure the structural preservation by examining the distributions of five node-level statistics across the $n$ nodes in each released network, and compare these distributions with those of the original network. Specifically, for each local statistic, denote its value at node $i$ by $T_i(A^{11})$ on network $A^{11}$, and its value at node $i$ on the released network $\tilde{A}^{11}$ as $T_i(\tilde{A}^{11})$. We want to compare the distribution of $\{T_i(\tilde{A}^{11})\}_{1\le i \le n}$ with that of $\{T_i(A^{11})\}_{1\le i \le n}$. We focus on the following five local statistics as representative of network properties. \textbf{1) Node degree}: $T_i(A) = \sum_j A_{ij}$. \textbf{2) V-shape count}: $T_i(A) = \binom{\sum_j A_{ij}}{2}$. It measures how many V-shape motifs involve node $i$ as the center. To be consistent with the motif count definition, this also includes triangles. \textbf{3) Triangle count}: $T_i(A) = \frac{1}{2}\sum_{j,k}A_{ij}A_{jk}A_{ki}$. It measures how many triangles involve node $i$. \textbf{4) Eigen centrality}: $T_i(A) = v(A)_i$ where $v(A)$ is the eigenvector corresponding to the largest eigenvalue of $A$. It is a popular measure of how central node $i$ is in the network based on the spectral structure. \textbf{5) Harmonic centrality}:  $T_i(A) = 1/\sum_{j\ne i}d_A(i,j)$ where $d_{A}(i,j)$ is the geodesic distance between node $i$ and node $j$ in the network. With these definitions, for each pair of $A^{11}$ and $\tilde{A}^{11}$, we can visualize the resulting distributions of the five statistics and compare them. 
However, to aggregate the results in a meaningful way, we also have to introduce the metrics to measure the difference between a pair of distributions.  In our study, we use the Wasserstein distance metric. For the highly skewed distributions (such as V-shape, triangle counts),  directly calculating the distances becomes numerically unstable, so we applied a logarithm transformation to the data and evaluate the recovery of the log-transformed distributions. 

We independently repeat each experiment 100 times and take the average distances as the final results. Due to the space limit, we present the results only for $\rho = 0.05, n=4000$ in this section. Additional results for other configurations of $\rho=0.025, 0.1, n=2000$ are included in Appendix~\ref{app:additional-sim}. The observed patterns are similar to this section. Tables  \ref{tab:LSM-RDPG-4000-d3-rho05-noSE} and \ref{tab:LSM-RDPG-4000-d6-rho05-noSE} give the distribution preservation errors under both models and all privacy budget levels when $d=3$ and $6$. The results are summarized as follows. 

\begin{figure}[h]
\centering
\captionsetup{type=table}%
\caption{%
  Comparison of node-level statistic distribution preservation for $\rho=0.05$ with $n=4000$ and $d=3$ under the inner product latent space model (LSM) and random dot product graph (RDPG) model. Values are the average Wasserstein distance over 100 replications. The average distances are more than ten times larger than their corresponding standard errors; therefore, standard errors are omitted for brevity.
}
\label{tab:LSM-RDPG-4000-d3-rho05-noSE}
\begin{singlespace}{%
  \small
  \begin{tabular}{cc|ccc|ccc}
    \toprule
    \multirow{2}{*}{\raisebox{-0.5ex}{Metric}} & \multirow{2}{*}{\raisebox{-0.5ex}{$\varepsilon$}} & \multicolumn{3}{c|}{LSM} & \multicolumn{3}{c}{RDPG} \\
    \cmidrule(lr){3-5} \cmidrule(lr){6-8}
    & & Hat & GRAND & Laplace & Hat & GRAND & Laplace \\
    \midrule
    \multirow{4}{*}{\makecell{Node\\Degree}} & 1 & 0.010 & 0.030 & 2.340 & 0.009 & 0.019 & 2.053 \\
    & 2 & 0.010 & 0.030 & 2.276 & 0.009 & 0.021 & 1.430 \\
    & 5 & 0.010 & 0.023 & 1.949 & 0.009 & 0.018 & 0.535 \\
    & 10 & 0.010 & 0.018 & 1.327 & 0.009 & 0.016 & 0.487 \\
    \midrule
    \multirow{4}{*}{\makecell{V-shape\\Count}} & 1 & 0.020 & 0.061 & 4.696 & 0.018 & 0.038 & 4.121 \\
    & 2 & 0.020 & 0.060 & 4.568 & 0.018 & 0.042 & 2.872 \\
    & 5 & 0.020 & 0.047 & 3.913 & 0.018 & 0.037 & 1.080 \\
    & 10 & 0.020 & 0.036 & 2.666 & 0.018 & 0.031 & 1.005 \\
    \midrule
    \multirow{4}{*}{\makecell{Triangle\\Count}} & 1 & 0.013 & 0.081 & 6.768 & 0.019 & 0.047 & 6.532 \\
    & 2 & 0.013 & 0.079 & 6.633 & 0.019 & 0.050 & 5.117 \\
    & 5 & 0.013 & 0.066 & 5.933 & 0.019 & 0.042 & 2.091 \\
    & 10 & 0.013 & 0.054 & 4.491 & 0.019 & 0.035 & 1.385 \\
    \midrule
    \multirow{4}{*}{\makecell{Eigen\\Centrality}} & 1 & 0.007 & 0.028 & 0.257 & 0.020 & 0.023 & 0.125 \\
    & 2 & 0.007 & 0.029 & 2.080 & 0.020 & 0.023 & 0.178 \\
    & 5 & 0.007 & 0.028 & 0.115 & 0.020 & 0.022 & 0.299 \\
    & 10 & 0.007 & 0.029 & 0.041 & 0.020 & 0.023 & 0.305 \\
    \midrule
    \multirow{4}{*}{\makecell{Harmonic\\Centrality}} & 1 & 2.046 & 4.445 & 884.790 & 1.162 & 2.385 & 659.590 \\
    & 2 & 2.046 & 4.301 & 836.624 & 1.162 & 2.558 & 345.161 \\
    & 5 & 2.046 & 3.298 & 621.660 & 1.162 & 2.300 & 133.620 \\
    & 10 & 2.046 & 2.374 & 324.391 & 1.162 & 1.975 & 124.544 \\
    \bottomrule
  \end{tabular}%
}\end{singlespace}
\captionsetup{type=figure}%
\end{figure}
\begin{itemize}
\item When focusing on the proposed GRAND method, it can be seen that as $\varepsilon$ increases, the overall performance becomes better, following the potential privacy-utility tradeoff.

\begin{figure}[h]
\centering
\captionsetup{type=table}%
\caption{%
  Comparison of node-level statistic distribution preservation for $\rho=0.05$ with $n=4000$ and $d=6$ under the inner product latent space model (LSM) and random dot product graph (RDPG) model. Values are the average Wasserstein distance over 100 replications.  The average distances are more than ten times larger than their corresponding standard errors; therefore, standard errors are omitted for brevity.
}
\label{tab:LSM-RDPG-4000-d6-rho05-noSE}
\begin{singlespace}{%
  \small
  \begin{tabular}{cc|ccc|ccc}
    \toprule
    \multirow{2}{*}{\raisebox{-0.5ex}{Metric}} & \multirow{2}{*}{\raisebox{-0.5ex}{$\varepsilon$}} & \multicolumn{3}{c|}{LSM} & \multicolumn{3}{c}{RDPG} \\
    \cmidrule(lr){3-5} \cmidrule(lr){6-8}
    & & Hat & GRAND & Laplace & Hat & GRAND & Laplace \\
    \midrule
    \multirow{4}{*}{\makecell{Node\\Degree}} & 1 & 0.012 & 0.034 & 2.348 & 0.010 & 0.016 & 2.259 \\
    & 2 & 0.012 & 0.030 & 2.334 & 0.010 & 0.015 & 2.038 \\
    & 5 & 0.012 & 0.026 & 2.238 & 0.010 & 0.015 & 1.074 \\
    & 10 & 0.012 & 0.022 & 1.957 & 0.010 & 0.015 & 0.562 \\
    \midrule
    \multirow{4}{*}{\makecell{V-shape\\Count}} & 1 & 0.024 & 0.069 & 4.712 & 0.020 & 0.032 & 4.532 \\
    & 2 & 0.024 & 0.061 & 4.682 & 0.020 & 0.031 & 4.090 \\
    & 5 & 0.024 & 0.054 & 4.491 & 0.020 & 0.031 & 2.158 \\
    & 10 & 0.024 & 0.045 & 3.929 & 0.020 & 0.030 & 1.134 \\
    \midrule
    \multirow{4}{*}{\makecell{Triangle\\Count}} & 1 & 0.015 & 0.073 & 6.765 & 0.023 & 0.039 & 6.993 \\
    & 2 & 0.015 & 0.068 & 6.734 & 0.023 & 0.038 & 6.501 \\
    & 5 & 0.015 & 0.064 & 6.531 & 0.023 & 0.036 & 4.226 \\
    & 10 & 0.015 & 0.059 & 5.921 & 0.023 & 0.034 & 2.036 \\
    \midrule
    \multirow{4}{*}{\makecell{Eigen\\Centrality}} & 1 & 0.009 & 0.028 & 0.409 & 0.022 & 0.024 & 0.234 \\
    & 2 & 0.009 & 0.032 & 0.308 & 0.022 & 0.025 & 0.117 \\
    & 5 & 0.009 & 0.031 & 0.177 & 0.022 & 0.028 & 0.255 \\
    & 10 & 0.009 & 0.035 & 0.080 & 0.022 & 0.023 & 0.355 \\
    \midrule
    \multirow{4}{*}{\makecell{Harmonic\\Centrality}} & 1 & 1.953 & 4.435 & 898.208 & 1.042 & 1.670 & 834.276 \\
    & 2 & 1.953 & 3.904 & 885.700 & 1.042 & 1.617 & 666.549 \\
    & 5 & 1.953 & 3.362 & 808.732 & 1.042 & 1.596 & 239.014 \\
    & 10 & 1.953 & 2.785 & 618.093 & 1.042 & 1.528 & 120.293 \\
    \bottomrule
  \end{tabular}%
}\end{singlespace}
\captionsetup{type=figure}%
\end{figure}
\item Comparing our method with the Hat method, we can see that it indeed yields worse preservation. This is also expected since it reflects the cost of achieving privacy. Noticeably, for weak privacy protection, such as $\varepsilon=10$, the performance of our method becomes comparable to the non-private Hat method. Recall that the Hat method is based on an estimation of a size-$m$ network. This verifies our theory that the node-wise estimation basically gives the accuracy as a standard estimation of size-$m$ network. 

\item Comparing our method with the Laplace mechanism, where both methods satisfy the same strictness of node DP, it is evident that our method achieves much better preservation of network properties. The Laplace mechanism blindly introduces the noises but fails to preserve the network structures.
\item The patterns are consistent across configurations (see Appendix~\ref{app:additional-sim}), demonstrating the effectiveness of GRAND across a wide range of scenarios.
\end{itemize}

\section{Examples of real-world data sets}\label{sec:data}

We now demonstrate how well the released network preserves the network properties, fixing $\varepsilon$, on real-world networks. In addition to the Hat method and the Laplace mechanism, we also include the network statistics of the original network (True) in comparisons.

\emph{Caltech Facebook network.} The first example data set is the Facebook social network between students in California Institute of Technology (Caltech), collected by \cite{traud2012social}. Each node in the network is a student and the edges are the Facebook connections between students. We process the data set following \cite{wang2016discussion} and \cite{li2018hierarchical} using the 2-core algorithm. The resulting data set contains 734 students with an average degree of 45.29. Similar to our previous experiments, we randomly select half of the students to hold out. The network of the other half of the students  ($n=367$) is then privatized and released. The inner product latent space model is used for model fitting, and the latent dimension $d=6$ is selected by the edge cross-validation method of \cite{li2016network} on the hold-out network. We enforce a strong privacy requirement at the level of $\varepsilon=1$. The distributions of the local statistics in the true and released network are calculated. 

\begin{figure}[h]
\vspace{-8mm}
\centering
\subfigure{\includegraphics[width=0.27\textwidth]{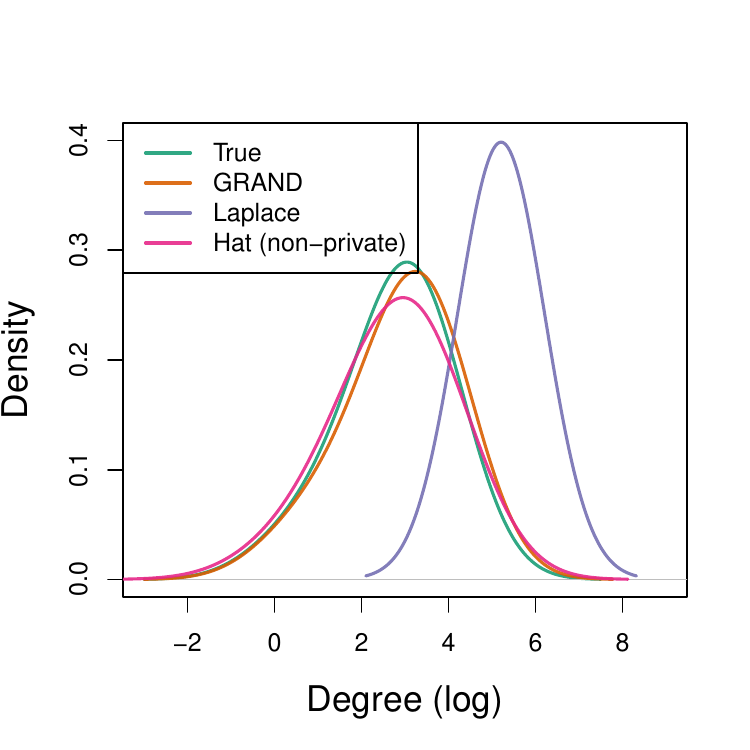}}
\subfigure{\includegraphics[width=0.27\textwidth]{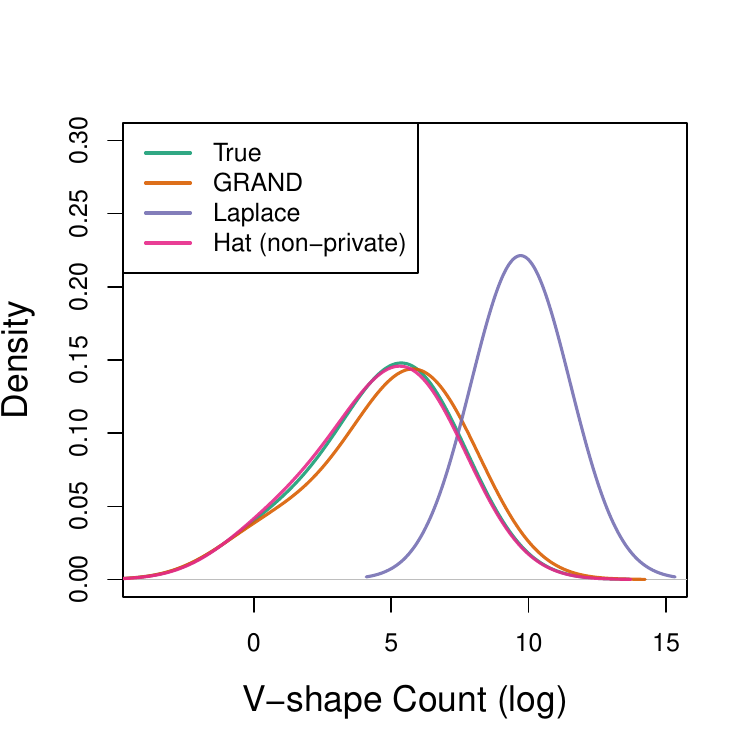}}
\subfigure{\includegraphics[width=0.27\textwidth]{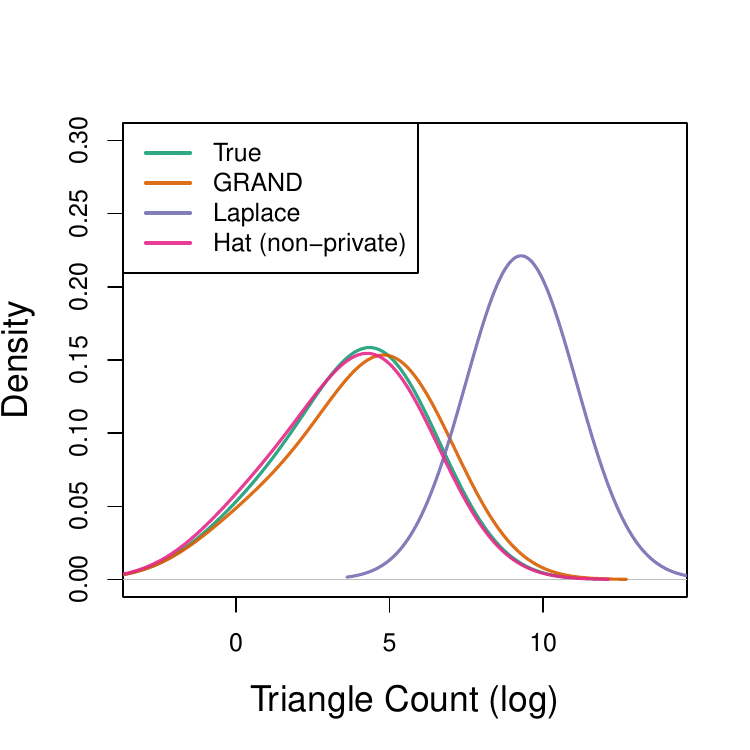}}\vspace{-8mm}
\vspace{-5mm}
\subfigure{\includegraphics[width=0.27\textwidth]{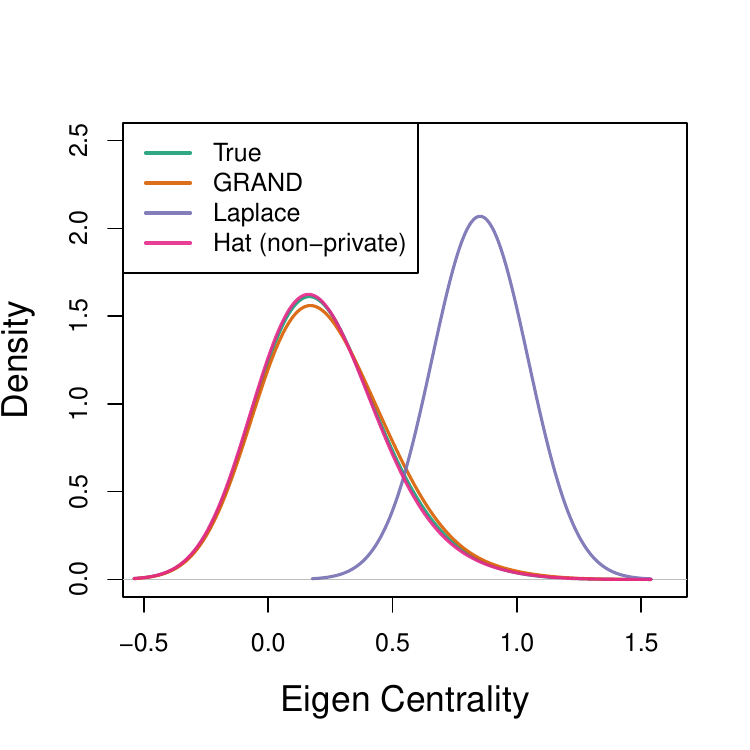}}
\subfigure{\includegraphics[width=0.27\textwidth]{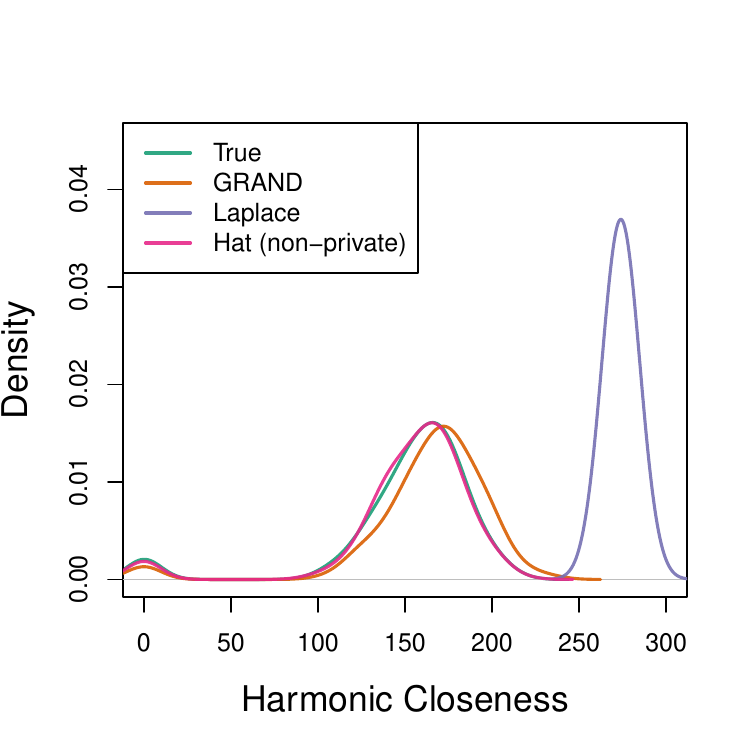}}

\caption{The distributions of five local statistics of the privatized Caltech social network with privacy budget $\varepsilon=1$.}
\label{fig:CaltechExample}
\vspace{-0.5cm}
\end{figure}

Figure~\ref{fig:CaltechExample} displays the resulting distributions of the five local statistics in four networks as introduced in Section~\ref{sec:sim}: the original network (True), the network released GRAND, the network generated using the naive Laplace mechanism (Laplace) and the non-private network from standard model estimation (Hat). It can be seen that the privatized network from GRAND matches the true network well in all of the five metrics. It also substantially outperforms the naive Laplace method, which completely misses the pattern. Meanwhile, GRAND maintains a slightly deviated but similar performance compared to the non-private Hat network. In particular, the harmonic closeness of the original network exhibits a bi-modal pattern, with one tiny lower mode (on the left end of the figure panel). This subtle pattern is also well captured by GRAND. 

\begin{figure}[h]
\vspace{-8mm}
\centering
\subfigure{\includegraphics[width=0.27\textwidth]{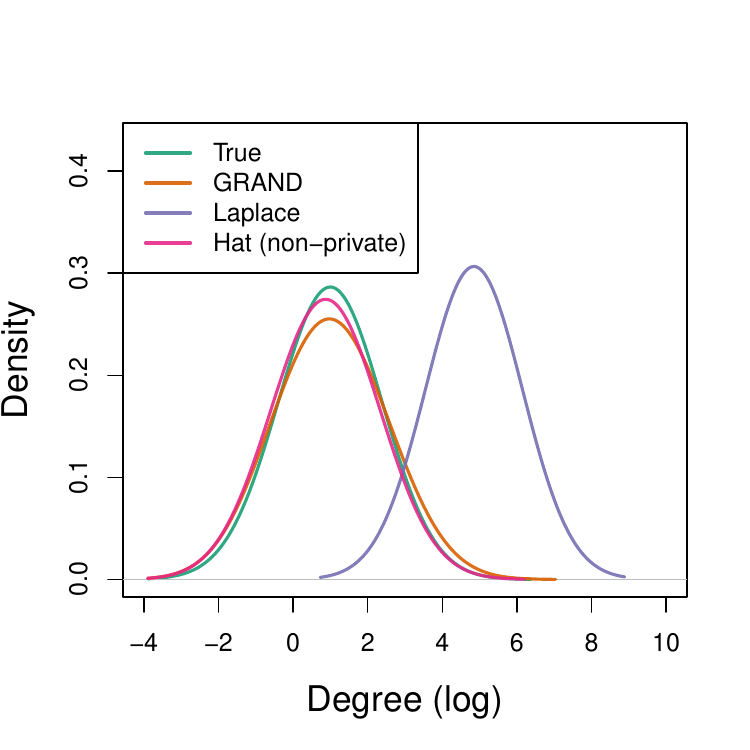}}
\subfigure{\includegraphics[width=0.27\textwidth]{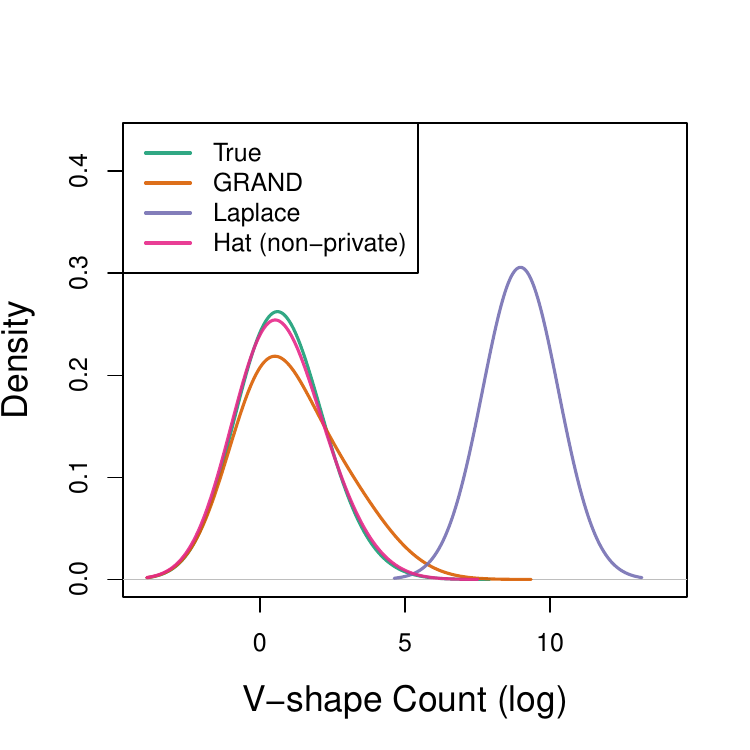}}
\subfigure{\includegraphics[width=0.27\textwidth]{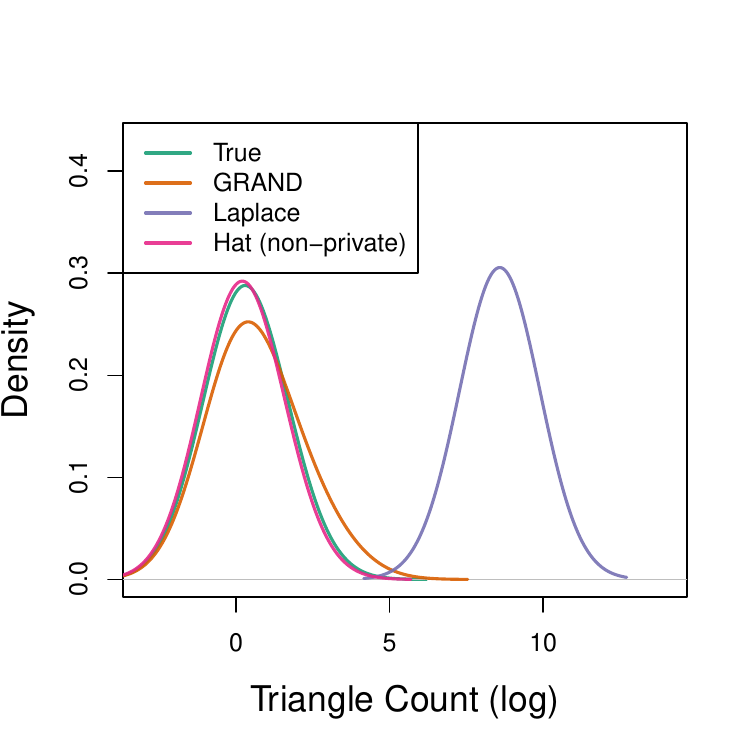}}\vspace{-8mm}
\vspace{-5mm}
\subfigure{\includegraphics[width=0.27\textwidth]{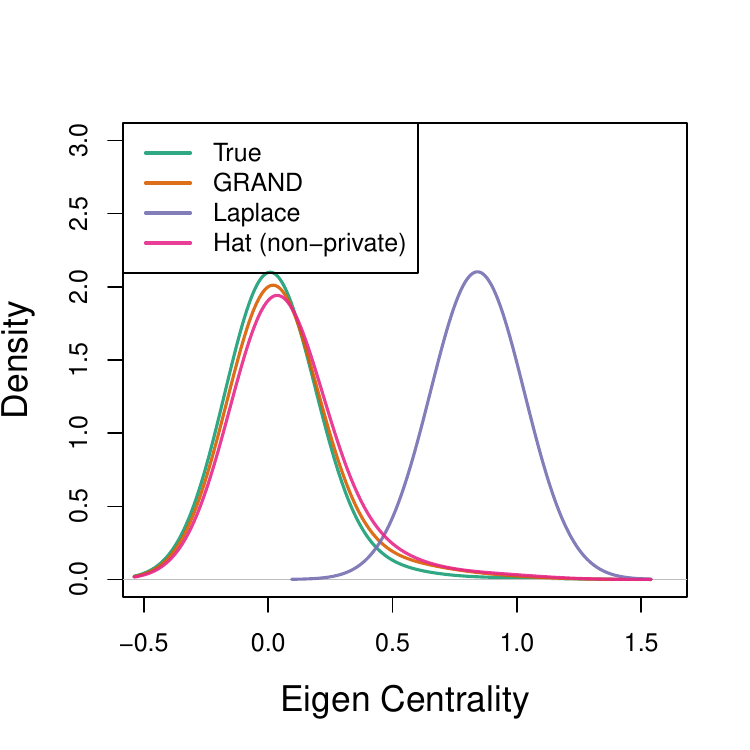}}
\subfigure{\includegraphics[width=0.27\textwidth]{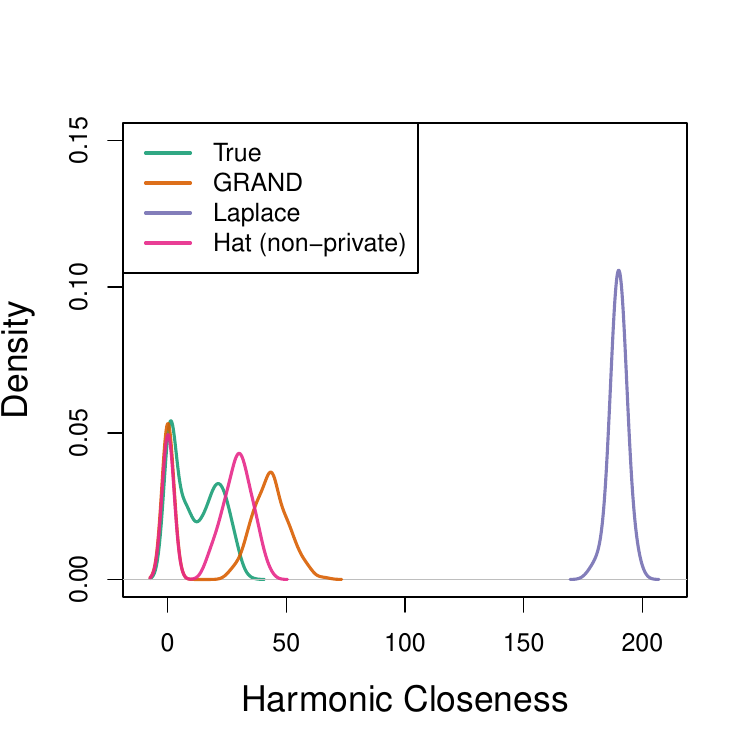}}

\caption{The distributions of five local properties of the privatized statisticians' coauthorship network with privacy budget $\varepsilon=1$.}
\label{fig:StatNetExample}
\vspace{-0.5cm}
\end{figure}

\emph{Statisticians' collaboration network.} The second example is about the collaboration network between statisticians based on publication data originally collected by \cite{ji2016coauthorship} and processed in \cite{li2020high}. Each node in this network is a statistician and an edge indicates that the two statisticians coauthored at least one paper during the data collection period. We use the same procedure to process the data set as before, and the resulting network has 509 nodes with an average degree of 4.24. This network is much sparser than the Caltech network, indicating a more difficult model fitting. We still hold out half of the nodes and privatize the network structure of the other half using privacy budget $\varepsilon=1$. The inner product latent space model is used and $d=4$ is selected by cross-validation on the hold-out data.

The distributions of the five local statistics in the resulting networks are shown in Figure~\ref{fig:StatNetExample}. The released network from GRAND matches the true network reasonably well, and maintains a performance that is very close to that of the Hat network, with a small deviation due to the incorporation of privacy guarantees. Similar to the  Caltech example, the true network also exhibits a bi-modal pattern for the harmonic centrality, but the two modes in this case are much closer to each other, which significantly increases the difficulty in preserving them under the introduced perturbations for privacy. The privatized network from GRAND, though does not perfectly recover the magnitudes, still captures the bi-modal pattern.

\emph{Additional evaluations on social network data.} We also evaluate GRAND on 107 social networks from \cite{ghasemian2019evaluating} and a privatization task for attributed network, demonstrating similarly competitive performance. These can be found in Appendix~\ref{appendix:500Network} and \ref{sec:attributes}.

\section{Discussion}
\label{sec:discussion}

This paper defines node-level differential privacy and introduces a novel privatization mechanism, named GRAND, to achieve node-level differential privacy. To conclude our paper, we want to add more discussions on a few aspects of our work.

\paragraph{Advantages of Perturbation-Based Privacy Protection} The GRAND framework functions primarily as a \textit{data-perturbation} mechanism. The released network retains a node-level bijective correspondence between the original network and the perturbed one. While our framework could theoretically support synthetic generation—by sampling new latent vectors from the estimated distribution $\hat{F}$ to construct a network of arbitrary size—such an approach yields artificial nodes that lack correspondence to real-world entities. The proposed data-perturbation paradigm exhibit several advantages that are critical in practice.

\emph{Extensibility and decentralizability.} Real-world data dissemination may involve multiple waves. For instance, the National Longitudinal Study of Adolescent to Adult Health (Add Health) \citep{harris2013add} comprises six waves of data collection spanning over 30 years, with new attributes and topological updates continuously appended to the same cohort regularly. A synthetic mechanism is fundamentally incompatible with this longitudinal pattern as nodes generated in previous waves may not match those generated later.
In contrast, by maintaining stable node identities, GRAND ensures the private network remains compatible with future data releases and external attribute merging.
Another key advantage is that a perturbation mechanism allows for data combination from multiple parties, where each party can
separately apply perturbation to their respective datasets. The combined data still follow the DP requirement (with the sum of the privacy budget used by each party) \citep{dwork2014algorithmic}.

\emph{Robustness to structural over-protection.} 
While node-level DP shields individual topological connections, a utility-maximizing mechanism should preserve meso-scopic features—such as community structure—that characterize groups rather than identifying individuals. A fundamental limitation of synthetic graph mechanisms is the decoupling of structure from identity: the synthetic nodes lose their correspondence to intrinsic ground-truth attributes, rendering original labels (like community membership) meaningless. In contrast, perturbation mechanisms like GRAND operate on the original node set, maintaining the structural alignment required to recover these group-level properties. In Figure~\ref{fig:SBM} of Appendix~\ref{app:additional-sim}, we demonstrate this advantage using the Stochastic Block Model \citep{holland1983stochastic}, a canonical model community structures. We applied spectral clustering to the GRAND-perturbed networks to recover community labels and compared them against the ground truth. GRAND successfully facilitates high recovery accuracy for reasonable $\varepsilon$ regimes. This confirms that GRAND is more robust to the structural over-protection, preserving meso-scopic signals under the privacy constraint, which is inevitable in synthetic approaches.


\paragraph{Future Directions} While laying a foundational cornerstone for solving the node DP problem, this framework also opens numerous avenues for future research. One promising direction is to generalize the proposed method to encode more systematic privacy protection for the hold-out network. As previously discussed, using such hold-out data currently appears unavoidable to achieve the desired privacy guarantees within the current scope of generality. One viable solution is to adopt a $(\varepsilon,\delta)$-differentially private holdout set with theoretical underpinnings and a transfer learning model to apply the learned CDF to the release set.
Another potentially useful future direction is to extend the current framework to accommodate more general network models such as the grahpon model \citep{bickel2009nonparametric}.

\bigskip

\bibliography{main}{}

\begin{thebibliography}{103}
\providecommand{\natexlab}[1]{#1}
\providecommand{\url}[1]{\texttt{#1}}
\expandafter\ifx\csname urlstyle\endcsname\relax
  \providecommand{\doi}[1]{doi: #1}\else
  \providecommand{\doi}{doi: \begingroup \urlstyle{rm}\Url}\fi

\bibitem[Abadi et~al.(2016)Abadi, Chu, Goodfellow, McMahan, Mironov, Talwar, and Zhang]{abadi2016deep}
M.~Abadi, A.~Chu, I.~Goodfellow, H.~B. McMahan, I.~Mironov, K.~Talwar, and L.~Zhang.
\newblock Deep learning with differential privacy.
\newblock In \emph{Proceedings of the 2016 ACM SIGSAC Conference on Computer and Communications Security}, pages 308--318, 2016.

\bibitem[Abadie et~al.(2021)Abadie, Fisher, and Dombrowski]{abadie2021privacy}
R.~Abadie, C.~Fisher, and K.~Dombrowski.
\newblock Privacy, confidentiality and anonymity: Understandings from people who inject drugs enrolled in a study of social networks and hiv risk.
\newblock \emph{Journal of empirical research on human research ethics: JERHRE}, 16\penalty0 (3):\penalty0 304, 2021.

\bibitem[Abawajy et~al.(2016)Abawajy, Ninggal, and Herawan]{abawajy2016privacy}
J.~H. Abawajy, M.~I.~H. Ninggal, and T.~Herawan.
\newblock Privacy preserving social network data publication.
\newblock \emph{IEEE communications surveys \& tutorials}, 18\penalty0 (3):\penalty0 1974--1997, 2016.

\bibitem[Airoldi et~al.(2008)Airoldi, Blei, Fienberg, and Xing]{airoldi2008mixed}
E.~M. Airoldi, D.~M. Blei, S.~E. Fienberg, and E.~P. Xing.
\newblock Mixed membership stochastic blockmodels.
\newblock \emph{Journal of Machine Learning Research}, 9\penalty0 (Sep):\penalty0 1981--2014, 2008.

\bibitem[Athreya et~al.(2018)Athreya, Fishkind, Tang, Priebe, Park, Vogelstein, Levin, Lyzinski, Qin, and Sussman]{athreya2017statistical}
A.~Athreya, D.~E. Fishkind, M.~Tang, C.~E. Priebe, Y.~Park, J.~T. Vogelstein, K.~Levin, V.~Lyzinski, Y.~Qin, and D.~L. Sussman.
\newblock Statistical inference on random dot product graphs: a survey.
\newblock \emph{Journal of Machine Learning Research}, 18\penalty0 (226):\penalty0 1--92, 2018.

\bibitem[Athreya et~al.(2021)Athreya, Tang, Park, and Priebe]{athreya2021estimation}
A.~Athreya, M.~Tang, Y.~Park, and C.~E. Priebe.
\newblock On estimation and inference in latent structure random graphs.
\newblock \emph{Statistical Science}, 36\penalty0 (1):\penalty0 68--88, 2021.

\bibitem[Bhattacharyya and Bickel(2015)]{bhattacharyya2015subsampling}
S.~Bhattacharyya and P.~J. Bickel.
\newblock Subsampling bootstrap of count features of networks.
\newblock \emph{The Annals of Statistics}, 43\penalty0 (6):\penalty0 2384--2411, 2015.

\bibitem[Bi and Shen(2023)]{bi2023distribution}
X.~Bi and X.~Shen.
\newblock Distribution-invariant differential privacy.
\newblock \emph{Journal of econometrics}, 235\penalty0 (2):\penalty0 444--453, 2023.

\bibitem[Bickel and Chen(2009)]{bickel2009nonparametric}
P.~J. Bickel and A.~Chen.
\newblock A nonparametric view of network models and newman--girvan and other modularities.
\newblock \emph{Proceedings of the National Academy of Sciences}, 106\penalty0 (50):\penalty0 21068--21073, 2009.

\bibitem[Bickel et~al.(2011)Bickel, Chen, and Levina]{bickel2011method}
P.~J. Bickel, A.~Chen, and E.~Levina.
\newblock The method of moments and degree distributions for network models.
\newblock \emph{The Annals of Statistics}, 39\penalty0 (5):\penalty0 2280--2301, 2011.

\bibitem[Billingsley(2013)]{billingsley2013convergence}
P.~Billingsley.
\newblock \emph{Convergence of probability measures}.
\newblock John Wiley \& Sons, 2013.

\bibitem[Blocki et~al.(2013)Blocki, Blum, Datta, and Sheffet]{blocki2013differentially}
J.~Blocki, A.~Blum, A.~Datta, and O.~Sheffet.
\newblock Differentially private data analysis of social networks via restricted sensitivity.
\newblock In \emph{Proceedings of the 4th conference on Innovations in Theoretical Computer Science}, pages 87--96, 2013.

\bibitem[Borgs et~al.(2015)Borgs, Chayes, and Smith]{borgs2015private}
C.~Borgs, J.~Chayes, and A.~Smith.
\newblock Private graphon estimation for sparse graphs.
\newblock \emph{Advances in Neural Information Processing Systems}, 28, 2015.

\bibitem[Borgs et~al.(2018)Borgs, Chayes, Smith, and Zadik]{borgs2018revealing}
C.~Borgs, J.~Chayes, A.~Smith, and I.~Zadik.
\newblock Revealing network structure, confidentially: Improved rates for node-private graphon estimation.
\newblock In \emph{2018 IEEE 59th Annual Symposium on Foundations of Computer Science (FOCS)}, pages 533--543. IEEE, 2018.

\bibitem[Cai et~al.(2019)Cai, Wang, and Zhang]{cai2019cost}
T.~T. Cai, Y.~Wang, and L.~Zhang.
\newblock The cost of privacy: Optimal rates of convergence for parameter estimation with differential privacy.
\newblock \emph{arXiv preprint arXiv:1902.04495}, 2019.

\bibitem[Chang et~al.(2024)Chang, Hu, Kolaczyk, Yao, and Yi]{chang2024edge}
J.~Chang, Q.~Hu, E.~D. Kolaczyk, Q.~Yao, and F.~Yi.
\newblock Edge differentially private estimation in the $\beta$-model via jittering and method of moments.
\newblock \emph{Annals of Statistics}, 2024.

\bibitem[Chatterjee et~al.(2011)Chatterjee, Diaconis, and Sly]{chatterjee2011random}
S.~Chatterjee, P.~Diaconis, and A.~Sly.
\newblock Random graphs with a given degree sequence.
\newblock \emph{The Annals of Applied Probability}, pages 1400--1435, 2011.

\bibitem[Chen et~al.(2024)Chen, Ding, d'Orsi, Hua, Liu, and Steurer]{chen2024private}
H.~Chen, J.~Ding, T.~d'Orsi, Y.~Hua, C.-H. Liu, and D.~Steurer.
\newblock Private graphon estimation via sum-of-squares.
\newblock In \emph{Proceedings of the 56th Annual ACM Symposium on Theory of Computing}, pages 172--182, 2024.

\bibitem[Chen and Lei(2018)]{chen2014network}
K.~Chen and J.~Lei.
\newblock Network cross-validation for determining the number of communities in network data.
\newblock \emph{Journal of the American Statistical Association}, 113\penalty0 (521):\penalty0 241--251, 2018.

\bibitem[Chen et~al.(2021)Chen, Kato, and Leng]{chen2021analysis}
M.~Chen, K.~Kato, and C.~Leng.
\newblock Analysis of networks via the sparse $\beta$-model.
\newblock \emph{Journal of the Royal Statistical Society Series B: Statistical Methodology}, 83\penalty0 (5):\penalty0 887--910, 2021.

\bibitem[Day et~al.(2016)Day, Li, and Lyu]{day2016publishing}
W.-Y. Day, N.~Li, and M.~Lyu.
\newblock Publishing graph degree distribution with node differential privacy.
\newblock In \emph{Proceedings of the 2016 International Conference on Management of Data}, pages 123--138, 2016.

\bibitem[Dong et~al.(2022)Dong, Roth, and Su]{dong2022gaussian}
J.~Dong, A.~Roth, and W.~J. Su.
\newblock Gaussian differential privacy.
\newblock \emph{Journal of the Royal Statistical Society: Series B (Statistical Methodology)}, 84\penalty0 (1):\penalty0 3--37, 2022.

\bibitem[Dwork(2006)]{dwork2006differential}
C.~Dwork.
\newblock Differential privacy.
\newblock In \emph{The 33rd International Colloquium on Automata, Languages and Programming}, pages 1--12. Springer, 2006.

\bibitem[Dwork et~al.(2006)Dwork, McSherry, Nissim, and Smith]{dwork2006calibrating}
C.~Dwork, F.~McSherry, K.~Nissim, and A.~Smith.
\newblock Calibrating noise to sensitivity in private data analysis.
\newblock In \emph{Proceedings of the 3rd Theory of Cryptography Conference}, pages 265--284, 2006.

\bibitem[Dwork et~al.(2014)Dwork, Roth, et~al.]{dwork2014algorithmic}
C.~Dwork, A.~Roth, et~al.
\newblock The algorithmic foundations of differential privacy.
\newblock \emph{Foundations and Trends{\textregistered} in Theoretical Computer Science}, 9\penalty0 (3--4):\penalty0 211--407, 2014.

\bibitem[Fan et~al.(2020)Fan, Zhang, and Yan]{fan2020asymptotic}
Y.~Fan, H.~Zhang, and T.~Yan.
\newblock Asymptotic theory for differentially private generalized $\beta$-models with parameters increasing.
\newblock \emph{arXiv preprint arXiv:2002.12733}, 2020.

\bibitem[Ghasemian et~al.(2019)Ghasemian, Hosseinmardi, and Clauset]{ghasemian2019evaluating}
A.~Ghasemian, H.~Hosseinmardi, and A.~Clauset.
\newblock Evaluating overfit and underfit in models of network community structure.
\newblock \emph{IEEE Transactions on Knowledge and Data Engineering}, 32\penalty0 (9):\penalty0 1722--1735, 2019.

\bibitem[Ghasemian et~al.(2020)Ghasemian, Hosseinmardi, Galstyan, Airoldi, and Clauset]{ghasemian2020stacking}
A.~Ghasemian, H.~Hosseinmardi, A.~Galstyan, E.~M. Airoldi, and A.~Clauset.
\newblock Stacking models for nearly optimal link prediction in complex networks.
\newblock \emph{Proceedings of the National Academy of Sciences}, 117\penalty0 (38):\penalty0 23393--23400, 2020.

\bibitem[Guo et~al.(2023)Guo, Li, Chang, and Ma]{guo2023privacy}
X.~Guo, X.~Li, X.~Chang, and S.~Ma.
\newblock Privacy-preserving community detection for locally distributed multiple networks.
\newblock \emph{arXiv preprint arXiv:2306.15709}, 2023.

\bibitem[Hall et~al.(2012)Hall, Rinaldo, and Wasserman]{hall2012random}
R.~Hall, A.~Rinaldo, and L.~Wasserman.
\newblock Random differential privacy.
\newblock \emph{Journal of Privacy and Confidentiality}, 4\penalty0 (2):\penalty0 43--59, 2012.

\bibitem[Han et~al.(2020)Han, Nebelung, Haarburger, Horst, Reinartz, Merhof, Kiessling, Schulz, and Truhn]{han2020breaking}
T.~Han, S.~Nebelung, C.~Haarburger, N.~Horst, S.~Reinartz, D.~Merhof, F.~Kiessling, V.~Schulz, and D.~Truhn.
\newblock Breaking medical data sharing boundaries by using synthesized radiographs.
\newblock \emph{Science Advances}, 6\penalty0 (49):\penalty0 eabb7973, 2020.

\bibitem[Harris(2013)]{harris2013add}
K.~M. Harris.
\newblock The add health study: Design and accomplishments.
\newblock 2013.

\bibitem[Hay et~al.(2009)Hay, Li, Miklau, and Jensen]{hay2009accurate}
M.~Hay, C.~Li, G.~Miklau, and D.~Jensen.
\newblock Accurate estimation of the degree distribution of private networks.
\newblock In \emph{2009 Ninth IEEE International Conference on Data Mining}, pages 169--178. IEEE, 2009.

\bibitem[Hehir et~al.(2022)Hehir, Slavkovi{\'c}, and Niu]{hehir2022consistent}
J.~Hehir, A.~Slavkovi{\'c}, and X.~Niu.
\newblock Consistent spectral clustering of network block models under local differential privacy.
\newblock \emph{The Journal of privacy and confidentiality}, 12\penalty0 (2), 2022.

\bibitem[Hie et~al.(2018)Hie, Cho, and Berger]{hie2018realizing}
B.~Hie, H.~Cho, and B.~Berger.
\newblock Realizing private and practical pharmacological collaboration.
\newblock \emph{Science}, 362\penalty0 (6412):\penalty0 347--350, 2018.

\bibitem[Hizo-Abes et~al.(2010)Hizo-Abes, Young, Reese, McFarlane, Wright, Cuerden, Garg, Network, et~al.]{hizo2010attitudes}
P.~Hizo-Abes, A.~Young, P.~P. Reese, P.~McFarlane, L.~Wright, M.~Cuerden, A.~X. Garg, D.~N. O. R.~D. Network, et~al.
\newblock Attitudes to sharing personal health information in living kidney donation.
\newblock \emph{Clinical Journal of the American Society of Nephrology}, 5\penalty0 (4):\penalty0 717--722, 2010.

\bibitem[Hoff et~al.(2002)Hoff, Raftery, and Handcock]{hoff2002latent}
P.~D. Hoff, A.~E. Raftery, and M.~S. Handcock.
\newblock Latent space approaches to social network analysis.
\newblock \emph{Journal of the American Statistical Association}, 97\penalty0 (460):\penalty0 1090--1098, 2002.

\bibitem[Holland et~al.(1983)Holland, Laskey, and Leinhardt]{holland1983stochastic}
P.~W. Holland, K.~B. Laskey, and S.~Leinhardt.
\newblock Stochastic blockmodels: First steps.
\newblock \emph{Social Networks}, 5\penalty0 (2):\penalty0 109--137, 1983.

\bibitem[Imola et~al.(2021)Imola, Murakami, and Chaudhuri]{imola2021locally}
J.~Imola, T.~Murakami, and K.~Chaudhuri.
\newblock Locally differentially private analysis of graph statistics.
\newblock In \emph{30th USENIX security symposium (USENIX Security 21)}, pages 983--1000, 2021.

\bibitem[Jain et~al.(2024)Jain, Smith, and Wagaman]{jain2024time}
P.~Jain, A.~Smith, and C.~Wagaman.
\newblock Time-aware projections: Truly node-private graph statistics under continual observation.
\newblock \emph{arXiv preprint arXiv:2403.04630}, 2024.

\bibitem[Ji and Jin(2016)]{ji2016coauthorship}
P.~Ji and J.~Jin.
\newblock Coauthorship and citation networks for statisticians.
\newblock \emph{The Annals of Applied Statistics}, 10\penalty0 (4):\penalty0 1779--1812, 2016.

\bibitem[Jiang et~al.(2021)Jiang, Pei, Yu, Yu, Gong, and Cheng]{jiang2021applications}
H.~Jiang, J.~Pei, D.~Yu, J.~Yu, B.~Gong, and X.~Cheng.
\newblock Applications of differential privacy in social network analysis: A survey.
\newblock \emph{IEEE Transactions on Knowledge and Data Engineering}, 35\penalty0 (1):\penalty0 108--127, 2021.

\bibitem[Jin et~al.(2021)Jin, Ke, and Luo]{jin2019optimal}
J.~Jin, Z.~T. Ke, and S.~Luo.
\newblock Optimal adaptivity of signed-polygon statistics for network testing.
\newblock \emph{The Annals of Statistics}, 49\penalty0 (6):\penalty0 3408--3433, 2021.

\bibitem[Kaissis et~al.(2020)Kaissis, Makowski, R{\"u}ckert, and Braren]{kaissis2020secure}
G.~A. Kaissis, M.~R. Makowski, D.~R{\"u}ckert, and R.~F. Braren.
\newblock Secure, privacy-preserving and federated machine learning in medical imaging.
\newblock \emph{Nature Machine Intelligence}, 2\penalty0 (6):\penalty0 305--311, 2020.

\bibitem[Kalemaj et~al.(2023)Kalemaj, Raskhodnikova, Smith, and Tsourakakis]{kalemaj2023node}
I.~Kalemaj, S.~Raskhodnikova, A.~Smith, and C.~E. Tsourakakis.
\newblock Node-differentially private estimation of the number of connected components.
\newblock In \emph{Proceedings of the 42nd ACM SIGMOD-SIGACT-SIGAI Symposium on Principles of Database Systems}, pages 183--194, 2023.

\bibitem[Karrer and Newman(2011)]{karrer2011stochastic}
B.~Karrer and M.~E. Newman.
\newblock Stochastic blockmodels and community structure in networks.
\newblock \emph{Physical Review E}, 83\penalty0 (1):\penalty0 016107, 2011.

\bibitem[Karwa and Slavkovi{\'c}(2016)]{karwa2016inference}
V.~Karwa and A.~Slavkovi{\'c}.
\newblock Inference using noisy degrees: Differentially private $\beta$-model and synthetic graphs.
\newblock \emph{Annals of Statistics}, 44\penalty0 (1):\penalty0 87--112, 2016.

\bibitem[Karwa et~al.(2017)Karwa, Krivitsky, and Slavkovi{\'c}]{karwa2017sharing}
V.~Karwa, P.~N. Krivitsky, and A.~B. Slavkovi{\'c}.
\newblock Sharing social network data: differentially private estimation of exponential family random-graph models.
\newblock \emph{Journal of the Royal Statistical Society Series C: Applied Statistics}, 66\penalty0 (3):\penalty0 481--500, 2017.

\bibitem[Kasiviswanathan et~al.(2013)Kasiviswanathan, Nissim, Raskhodnikova, and Smith]{kasiviswanathan2013analyzing}
S.~P. Kasiviswanathan, K.~Nissim, S.~Raskhodnikova, and A.~Smith.
\newblock Analyzing graphs with node differential privacy.
\newblock In \emph{Theory of Cryptography: 10th Theory of Cryptography Conference, TCC 2013, Tokyo, Japan, March 3-6, 2013. Proceedings}, pages 457--476. Springer, 2013.

\bibitem[Kenthapadi and Tran(2018)]{kenthapadi2018pripearl}
K.~Kenthapadi and T.~T. Tran.
\newblock Pripearl: A framework for privacy-preserving analytics and reporting at {LinkedIn}.
\newblock In \emph{Proceedings of the 27th ACM International Conference on Information and Knowledge Management}, pages 2183--2191, 2018.

\bibitem[Kipf and Welling(2017)]{kipf2016semi}
T.~N. Kipf and M.~Welling.
\newblock Semi-supervised classification with graph convolutional networks.
\newblock In \emph{{ICLR '17 - International Conference on Learning Representations}}, 2017.

\bibitem[Laeuchli et~al.(2022)Laeuchli, Ram{\'\i}rez-Cruz, and Trujillo-Rasua]{laeuchli2022analysis}
J.~Laeuchli, Y.~Ram{\'\i}rez-Cruz, and R.~Trujillo-Rasua.
\newblock Analysis of centrality measures under differential privacy models.
\newblock \emph{Applied Mathematics and Computation}, 412:\penalty0 126546, 2022.

\bibitem[Le and Li(2022)]{le2022linear}
C.~M. Le and T.~Li.
\newblock Linear regression and its inference on noisy network-linked data.
\newblock \emph{Journal of the Royal Statistical Society Series B}, 84\penalty0 (5):\penalty0 1851--1885, 2022.

\bibitem[Lei and Rinaldo(2014)]{lei2014consistency}
J.~Lei and A.~Rinaldo.
\newblock Consistency of spectral clustering in stochastic block models.
\newblock \emph{The Annals of Statistics}, 43\penalty0 (1):\penalty0 215--237, 2014.

\bibitem[Lei et~al.(2018)Lei, Charest, Slavkovic, Smith, and Fienberg]{lei2018differentially}
J.~Lei, A.-S. Charest, A.~Slavkovic, A.~Smith, and S.~Fienberg.
\newblock Differentially private model selection with penalized and constrained likelihood.
\newblock \emph{Journal of the Royal Statistical Society Series A: Statistics in Society}, 181\penalty0 (3):\penalty0 609--633, 2018.

\bibitem[Levin and Levina(2019)]{levin2019bootstrapping}
K.~Levin and E.~Levina.
\newblock Bootstrapping networks with latent space structure.
\newblock \emph{arXiv preprint arXiv:1907.10821}, 2019.

\bibitem[Li et~al.(2023{\natexlab{a}})Li, Xu, and Zhu]{li2023statistical}
J.~Li, G.~Xu, and J.~Zhu.
\newblock Statistical inference on latent space models for network data.
\newblock \emph{arXiv preprint arXiv:2312.06605}, 2023{\natexlab{a}}.

\bibitem[Li and Le(2023)]{li2023network}
T.~Li and C.~M. Le.
\newblock Network estimation by mixing: Adaptivity and more.
\newblock \emph{Journal of the American Statistical Association}, pages 1--16, 2023.

\bibitem[Li et~al.(2020{\natexlab{a}})Li, Levina, and Zhu]{li2016network}
T.~Li, E.~Levina, and J.~Zhu.
\newblock Network cross-validation by edge sampling.
\newblock \emph{Biometrika}, 107\penalty0 (2):\penalty0 257--276, 2020{\natexlab{a}}.

\bibitem[Li et~al.(2020{\natexlab{b}})Li, Qian, Levina, and Zhu]{li2020high}
T.~Li, C.~Qian, E.~Levina, and J.~Zhu.
\newblock High-dimensional gaussian graphical models on network-linked data.
\newblock \emph{Journal of Machine Learning Research}, 21:\penalty0 74--1, 2020{\natexlab{b}}.

\bibitem[Li et~al.(2022)Li, Lei, Bhattacharyya, Van~den Berge, Sarkar, Bickel, and Levina]{li2018hierarchical}
T.~Li, L.~Lei, S.~Bhattacharyya, K.~Van~den Berge, P.~Sarkar, P.~J. Bickel, and E.~Levina.
\newblock Hierarchical community detection by recursive partitioning.
\newblock \emph{Journal of the American Statistical Association}, 117\penalty0 (538):\penalty0 951--968, 2022.

\bibitem[Li et~al.(2023{\natexlab{b}})Li, Levina, and Zhu]{li2023community}
T.~Li, E.~Levina, and J.~Zhu.
\newblock Community models for networks observed through edge nominations.
\newblock \emph{Journal of Machine Learning Research}, 24\penalty0 (282):\penalty0 1--36, 2023{\natexlab{b}}.

\bibitem[Li et~al.(2023{\natexlab{c}})Li, Purcell, Rakotoarivelo, Smith, Ranbaduge, and Ng]{li2023private}
Y.~Li, M.~Purcell, T.~Rakotoarivelo, D.~Smith, T.~Ranbaduge, and K.~S. Ng.
\newblock Private graph data release: A survey.
\newblock \emph{ACM Computing Surveys}, 55\penalty0 (11):\penalty0 1--39, 2023{\natexlab{c}}.

\bibitem[Lin et~al.(2023)Lin, Paquette, and Kolaczyk]{lin2023differentially}
S.~Lin, E.~Paquette, and E.~D. Kolaczyk.
\newblock Differentially private linear regression with linked data.
\newblock \emph{arXiv preprint arXiv:2308.00836}, 2023.

\bibitem[Little et~al.(2014)Little, Kosakovsky~Pond, Anderson, Young, Wertheim, Mehta, May, and Smith]{little2014using}
S.~J. Little, S.~L. Kosakovsky~Pond, C.~M. Anderson, J.~A. Young, J.~O. Wertheim, S.~R. Mehta, S.~May, and D.~M. Smith.
\newblock Using hiv networks to inform real time prevention interventions.
\newblock \emph{PloS one}, 9\penalty0 (6):\penalty0 e98443, 2014.

\bibitem[Liu et~al.(2022)Liu, Zhao, Liu, Zhao, Chen, and Li]{liu2022collecting}
Y.~Liu, S.~Zhao, Y.~Liu, D.~Zhao, H.~Chen, and C.~Li.
\newblock Collecting triangle counts with edge relationship local differential privacy.
\newblock In \emph{2022 IEEE 38th International Conference on Data Engineering (ICDE)}, pages 2008--2020. IEEE, 2022.

\bibitem[Ma et~al.(2025)Ma, Jiang, Zhao, Yang, and Yu]{ma2025locally}
Y.~Ma, F.~Jiang, Z.~Zhao, H.~Yang, and Y.~Yu.
\newblock Locally private nonparametric contextual multi-armed bandits.
\newblock \emph{arXiv preprint arXiv:2503.08098}, 2025.

\bibitem[Ma et~al.(2020)Ma, Ma, and Yuan]{ma2020universal}
Z.~Ma, Z.~Ma, and H.~Yuan.
\newblock Universal latent space model fitting for large networks with edge covariates.
\newblock \emph{Journal of Machine Learning Research}, 21\penalty0 (4):\penalty0 1--67, 2020.

\bibitem[Macwan and Patel(2018)]{macwan2018node}
K.~R. Macwan and S.~J. Patel.
\newblock Node differential privacy in social graph degree publishing.
\newblock \emph{Procedia computer science}, 143:\penalty0 786--793, 2018.

\bibitem[Marcus et~al.(2023)Marcus, Berner, Hadaya, and Hurst]{marcus2023anonymity}
K.~Marcus, D.~Berner, K.~Hadaya, and S.~Hurst.
\newblock Anonymity in kidney paired donation: A systematic review of reasons.
\newblock \emph{Transplant International}, 36:\penalty0 10913, 2023.

\bibitem[Maugis et~al.(2020)Maugis, Olhede, Priebe, and Wolfe]{maugis2020testing}
P.-A. Maugis, S.~Olhede, C.~Priebe, and P.~Wolfe.
\newblock Testing for equivalence of network distribution using subgraph counts.
\newblock \emph{Journal of Computational and Graphical Statistics}, 29\penalty0 (3):\penalty0 455--465, 2020.

\bibitem[McSherry and Talwar(2007)]{mcsherry2007mechanism}
F.~McSherry and K.~Talwar.
\newblock Mechanism design via differential privacy.
\newblock In \emph{48th Annual IEEE Symposium on Foundations of Computer Science}, pages 94--103, 2007.

\bibitem[Milo et~al.(2002)Milo, Shen-Orr, Itzkovitz, Kashtan, Chklovskii, and Alon]{milo2002network}
R.~Milo, S.~Shen-Orr, S.~Itzkovitz, N.~Kashtan, D.~Chklovskii, and U.~Alon.
\newblock Network motifs: simple building blocks of complex networks.
\newblock \emph{Science}, 298\penalty0 (5594):\penalty0 824--827, 2002.

\bibitem[M{\"u}lle et~al.(2015)M{\"u}lle, Clifton, and B{\"o}hm]{mulle2015privacy}
Y.~M{\"u}lle, C.~Clifton, and K.~B{\"o}hm.
\newblock Privacy-integrated graph clustering through differential privacy.
\newblock In \emph{EDBT/ICDT Workshops}, volume 1330, pages 247--254, 2015.

\bibitem[Narayanan and Shmatikov(2008)]{narayanan2008robust}
A.~Narayanan and V.~Shmatikov.
\newblock Robust de-anonymization of large sparse datasets.
\newblock In \emph{2008 IEEE Symposium on Security and Privacy (sp 2008)}, pages 111--125. IEEE, 2008.

\bibitem[Newman(2018)]{newman2018networka}
M.~Newman.
\newblock Network structure from rich but noisy data.
\newblock \emph{Nature Physics}, page~1, 2018.

\bibitem[Nissim et~al.(2007)Nissim, Raskhodnikova, and Smith]{nissim2007smooth}
K.~Nissim, S.~Raskhodnikova, and A.~Smith.
\newblock Smooth sensitivity and sampling in private data analysis.
\newblock In \emph{Proceedings of the thirty-ninth annual ACM symposium on Theory of computing}, pages 75--84, 2007.

\bibitem[Paluck et~al.(2016)Paluck, Shepherd, and Aronow]{paluck2016changing}
E.~L. Paluck, H.~Shepherd, and P.~M. Aronow.
\newblock Changing climates of conflict: A social network experiment in 56 schools.
\newblock \emph{Proceedings of the National Academy of Sciences}, 113\penalty0 (3):\penalty0 566--571, 2016.

\bibitem[Qi et~al.(2024)Qi, Li, and Zhou]{qi2024multivariate}
M.~Qi, T.~Li, and W.~Zhou.
\newblock Multivariate inference of network moments by subsampling.
\newblock \emph{arXiv preprint arXiv:2409.01599}, 2024.

\bibitem[Qin et~al.(2017)Qin, Yu, Yang, Khalil, Xiao, and Ren]{qin2017generating}
Z.~Qin, T.~Yu, Y.~Yang, I.~Khalil, X.~Xiao, and K.~Ren.
\newblock Generating synthetic decentralized social graphs with local differential privacy.
\newblock In \emph{Proceedings of the 2017 ACM SIGSAC conference on computer and communications security}, pages 425--438, 2017.

\bibitem[Raskhodnikova and Smith(2016)]{raskhodnikova2016lipschitz}
S.~Raskhodnikova and A.~Smith.
\newblock Lipschitz extensions for node-private graph statistics and the generalized exponential mechanism.
\newblock In \emph{2016 IEEE 57th Annual Symposium on Foundations of Computer Science (FOCS)}, pages 495--504. IEEE, 2016.

\bibitem[Rohde and Steinberger(2018)]{rohde2018geometrizing}
A.~Rohde and L.~Steinberger.
\newblock Geometrizing rates of convergence under local differential privacy constraints.
\newblock \emph{Annals of Statistics}, page forthcoming, 2018.

\bibitem[Rubin-Delanchy et~al.(2022)Rubin-Delanchy, Cape, Tang, and Priebe]{rubin2022statistical}
P.~Rubin-Delanchy, J.~Cape, M.~Tang, and C.~E. Priebe.
\newblock A statistical interpretation of spectral embedding: The generalised random dot product graph.
\newblock \emph{Journal of the Royal Statistical Society Series B: Statistical Methodology}, 84\penalty0 (4):\penalty0 1446--1473, 2022.

\bibitem[Santos-Lozada et~al.(2020)Santos-Lozada, Howard, and Verdery]{santos2020differential}
A.~R. Santos-Lozada, J.~T. Howard, and A.~M. Verdery.
\newblock How differential privacy will affect our understanding of health disparities in the united states.
\newblock \emph{Proceedings of the National Academy of Sciences}, 117\penalty0 (24):\penalty0 13405--13412, 2020.

\bibitem[Schum(2001)]{schum2001evidential}
D.~A. Schum.
\newblock \emph{The Evidential Foundations of Probabilistic Reasoning}.
\newblock Northwestern University Press, 2001.

\bibitem[Sengupta and Chen(2018)]{sengupta2018block}
S.~Sengupta and Y.~Chen.
\newblock A block model for node popularity in networks with community structure.
\newblock \emph{Journal of the Royal Statistical Society: Series B (Statistical Methodology)}, 80\penalty0 (2):\penalty0 365--386, 2018.

\bibitem[Sivasubramaniam et~al.(2020)Sivasubramaniam, Li, and He]{sivasubramaniam2020differentially}
H.~Sivasubramaniam, H.~Li, and X.~He.
\newblock Differentially private sublinear average degree approximation, 2020.

\bibitem[Soto et~al.(2022)Soto, Bharath, Reimherr, and Slavkovi{\'c}]{soto2022shape}
C.~Soto, K.~Bharath, M.~Reimherr, and A.~Slavkovi{\'c}.
\newblock Shape and structure preserving differential privacy.
\newblock \emph{Advances in Neural Information Processing Systems}, 35:\penalty0 24693--24705, 2022.

\bibitem[Sussman et~al.(2014)Sussman, Tang, and Priebe]{sussman2014consistent}
D.~L. Sussman, M.~Tang, and C.~E. Priebe.
\newblock Consistent latent position estimation and vertex classification for random dot product graphs.
\newblock \emph{IEEE transactions on pattern analysis and machine intelligence}, 36\penalty0 (1):\penalty0 48--57, 2014.

\bibitem[Traud et~al.(2012)Traud, Mucha, and Porter]{traud2012social}
A.~L. Traud, P.~J. Mucha, and M.~A. Porter.
\newblock Social structure of {F}acebook networks.
\newblock \emph{Phys. A}, 391\penalty0 (16):\penalty0 4165--4180, Aug 2012.

\bibitem[Ullman and Sealfon(2019)]{ullman2019efficiently}
J.~Ullman and A.~Sealfon.
\newblock Efficiently estimating erdos-renyi graphs with node differential privacy.
\newblock \emph{Advances in Neural Information Processing Systems}, 32, 2019.

\bibitem[Vershynin(2018)]{vershynin_2018}
R.~Vershynin.
\newblock \emph{High-Dimensional Probability: An Introduction with Applications in Data Science}.
\newblock Cambridge Series in Statistical and Probabilistic Mathematics. Cambridge University Press, 2018.
\newblock \doi{10.1017/9781108231596}.

\bibitem[Wang et~al.(2024)Wang, Le, and Li]{wang2024perturbation}
J.~Wang, C.~M. Le, and T.~Li.
\newblock Perturbation-robust predictive modeling of social effects by network subspace generalized linear models.
\newblock \emph{arXiv preprint arXiv:2410.01163}, 2024.

\bibitem[Wang et~al.(2022)Wang, Yan, Jiang, and Leng]{wang2022two}
Q.~Wang, T.~Yan, B.~Jiang, and C.~Leng.
\newblock Two-mode networks: Inference with as many parameters as actors and differential privacy.
\newblock \emph{Journal of Machine Learning Research}, 23\penalty0 (292):\penalty0 1--38, 2022.
\newblock URL \url{http://jmlr.org/papers/v23/20-1255.html}.

\bibitem[Wang and Rohe(2016)]{wang2016discussion}
S.~Wang and K.~Rohe.
\newblock Discussion of ``coauthorship and citation networks for statisticians".
\newblock \emph{The Annals of Applied Statistics}, 10\penalty0 (4):\penalty0 1820--1826, 2016.

\bibitem[Wang et~al.(2016)Wang, Lei, and Fienberg]{wang2016average}
Y.-X. Wang, J.~Lei, and S.~E. Fienberg.
\newblock On-average kl-privacy and its equivalence to generalization for max-entropy mechanisms.
\newblock In \emph{Privacy in Statistical Databases: UNESCO Chair in Data Privacy, International Conference, PSD 2016, Dubrovnik, Croatia, September 14--16, 2016, Proceedings}, pages 121--134. Springer, 2016.

\bibitem[Wasserman(2006)]{wasserman2006all}
L.~Wasserman.
\newblock \emph{All of nonparametric statistics}.
\newblock Springer Science \& Business Media, 2006.

\bibitem[Wasserman and Zhou(2010)]{wasserman2010statistical}
L.~Wasserman and S.~Zhou.
\newblock A statistical framework for differential privacy.
\newblock \emph{Journal of the American Statistical Association}, 105\penalty0 (489):\penalty0 375--389, 2010.

\bibitem[Xue et~al.(2024)Xue, Lin, and Yu]{xue2024optimal}
G.~Xue, Z.~Lin, and Y.~Yu.
\newblock Optimal estimation in private distributed functional data analysis.
\newblock \emph{arXiv preprint arXiv:2412.06582}, 2024.

\bibitem[Yan(2021)]{yan2021directed}
T.~Yan.
\newblock Directed networks with a differentially private bi-degree sequence.
\newblock \emph{Statistica Sinica}, 31\penalty0 (4):\penalty0 2031--2050, 2021.

\bibitem[Yin et~al.(2006)Yin, Zhao, and Wei]{yin2006asymptotic}
C.~Yin, L.~Zhao, and C.~Wei.
\newblock Asymptotic normality and strong consistency of maximum quasi-likelihood estimates in generalized linear models.
\newblock \emph{Science in China Series A}, 49:\penalty0 145--157, 2006.

\bibitem[Young and Scheinerman(2007)]{young2007random}
S.~J. Young and E.~R. Scheinerman.
\newblock Random dot product graph models for social networks.
\newblock In \emph{International Workshop on Algorithms and Models for the Web-Graph}, pages 138--149. Springer, 2007.

\bibitem[Zhang and Xia(2022)]{zhang2022edgeworth}
Y.~Zhang and D.~Xia.
\newblock Edgeworth expansions for network moments.
\newblock \emph{The Annals of Statistics}, 50\penalty0 (2):\penalty0 726--753, 2022.

\end{thebibliography}
\bibliographystyle{abbrvnat}

\newpage

\begin{appendix}


\section*{Appendix}

The appendix includes proofs of the theoretical results and additional numerical experiments for the paper.

\section{Proof of Theorem~\ref{thm:NDP}}
\begin{proof}
  
  Due to the node-wise estimation procedure, note that the change of any node $i \in [n]$ does not change the estimate of $\hat{Z}_j$, $j\ne i$, $j\in [n]$. That means, the estimates $\hat{Z}_i$'s are separable. By applying Theorem 2 of \cite{bi2023distribution}, $\tilde{Z}_i$ acquired via Equation \eqref{eq:TNR} is differentially private. Since the generative function $W$ is assumed to be known. We have each node of $\tilde{A}^{11}$ also being differentially private.
\end{proof}

\section{Proof of Theorem~\ref{thm:main-marginal} (Marginal Convergence)}

The DIP procedure, given $\hat{Z}_i$, is to apply univariate DIP transformation by columns. In each step, we use the approximated (conditional) CDF, $\hat{F}_m^{j\mid 1:(j-1)}$ of the $j$-th variable, conditioning on the previous $j-1$ variables, for privatization: for each $i$, we generate $\mathbf e_i$ from $\text{Laplace}(0,1/\varepsilon)^d$ independently, and then compute 
$$\tilde{Z}_{ij} = (\hat{F}_m^{j\mid 1:(j-1)})^{-1}(G(\hat{F}_m^{j\mid 1:(j-1)}(\hat{Z}_{ij})+e_{ij})).$$

To prove the consistency for the latent distribution, we need the following key steps: 1) Prove the uniform convergence of the estimated $\hat{F}$ and its inverse based on $\{\hat{Z}_{i}\}_{i>n}$; 2) Prove the convergence of marginal distribution of each $\tilde{Z}_i$. The potential challenges are from the fact that $\{\hat{Z}_{i}\}_{i>n}$ are not precisely $\{Z_{i}\}_{i>n}$ and they are not independent. Moreover, each $\hat{Z}_i$, $i\le n$ is also dependent on $\{\hat{Z}_{i}\}_{i>n}$. Therefore, the analysis of the convergence of these quantities have to take these dependence into consideration.

\subsection{Crucial tools}

The following lemma provides a crucial tool in our main proof.

\begin{lemma}[Uniform Convergence of Empirical CDF with Perturbed Data]\label{lemma:ecdf_convergence}
Let \( Q = \{ Q_{i} \}_{i=1}^m \) be i.i.d. random variables from a distribution \( F \) on \( \bR \), 
and let \( F_m(x) \) be the empirical CDF computed from \( Q \). Let \( \hat{Q} = \{ \hat{Q}_{i} \}_{i=1}^m \)
be a perturbed version of \(Q\), which can be potentially dependent. Let 
\( \hat{F}_m(x) \) denote the empirical CDF based on the perturbed data \( \hat{Q} \).

Assume the following regularity conditions:
\begin{itemize}
    \item The distribution \( F \) is a continuous distribution with density $f$, where $f$ is bounded by a constant \( C_{\mathrm{up}} > 0 \),
  i.e., \( f(y) \leq C_{\mathrm{up}} \).
  \item The perturbed data satisfies \( \sup_{1\le i\le m}\big|\hat{Q}_{i} - Q_{i}\big| \leq \delta_m \) for some $\delta_m$ that is associated with $m$.
\end{itemize}
Then, the empirical CDF \( \hat{F}_m(x) \), satisfies the following
uniform convergence bound:
\[
\sup_{x \in \bR} \big|\hat{F}_m(x) - F(x)\big| \leq \ocal_\P\left( \delta_m \right) + \ocal_\P\biggl( \sqrt{\frac{1}{m}} \biggr).
\]
\end{lemma}

\begin{proof}
We decompose the total error as follows:
\[
\sup_{x \in \bR} \big|\hat{F}_m(x) - F(x)\big| \leq \underbrace{\sup_{x \in \bR} \big|\hat{F}_m(x) - F_m(x)\big|}_{\textsf{Perturbation Error}} + \underbrace{\sup_{x \in \bR} \big|F_m(x) - F(x)\big|}_{\textsf{Sampling Error}}.
\]

From Glivenko-Cantelli theorem, we have:
\[
\sup_{x \in \bR} \big|F_m(x) - F(x)\big| = \ocal_\P\biggl( \sqrt{\frac{1}{m}} \biggr).
\]

For any \( x \in \bR \), the empirical CDFs are defined as:
\[
F_m(x) = \frac{1}{m} \sum_{i=1}^m \mathbf{1}_{\{ Q_i \leq x \}}, \quad \hat{F}_m(x) = \frac{1}{m} \sum_{i=1}^m \mathbf{1}_{\{ \hat{Q}_i \leq x \}}.
\]
The difference between \( \hat{F}_m(x) \) and \( F_m(x) \) arises when \( \mathbf{1}_{\{ \hat{Q}_i \leq x \}} \neq \mathbf{1}_{\{ Q_i \leq x \}} \). Define
\[
D_i(x) = \big| \mathbf{1}_{\{ \hat{Q}_i \leq x \}} - \mathbf{1}_{\{ Q_i \leq x \}} \big|.
\]
We have
\[
\big|\hat{F}_m(x) - F_m(x)\big| \le \frac{1}{m} \sum_{i=1}^m D_i(x).
\]

Since \( \big| \hat{Q}_i - Q_i \big| \leq \delta_m \), \( D_i(x) \) can be non-zero only if \( Q_i \in H(x) \), where \( H(x) \) is the interval:
\[
H(x) = \left\{ y \in \bR: | y - x |\leq \delta_m \right\}.
\]

Therefore,
\[
D_i(x) \leq \mathbf{1}_{H(x)}(Q_i).
\]

We have
\[
\E[ D_i(x) ] \leq \P( Q_i \in H(x) ) \leq C_{\mathrm{up}} \cdot (2\delta_m),
\]
which is uniform in $x$, leading to
\[
\sup_{x \in \bR} \E[ D_i(x) ] \leq C_{\mathrm{up}} \cdot (2\delta_m).
\]

For a universal control, we call the results from empirical process theory for VC classes. That is, we have:

\[
\E\left[ \sup_{x \in \bR} \left| \frac{1}{m} \sum_{i=1}^m D_i(x) - \E[D_i(x)] \right| \right] \leq K \sqrt{\frac{1}{m}}.
\]
Thus,
\[
\sup_{x \in \bR} \big|\hat{F}_m(x) - F_m(x)\big| \leq \sup_{x\in\bR} \E[D_i(x)] + \ocal_\P\biggl( \sqrt{\frac{1}{m}} \biggr) = \ocal( \delta_m ) + \ocal_\P\biggl( \sqrt{\frac{1}{m}} \biggr). 
\]

Therefore,
\[
\sup_{x \in \bR} \big|\hat{F}_m(x) - F(x)\big| \leq \ocal_\P( \delta_m ) + \ocal_\P\biggl( \sqrt{\frac{1}{m}} \biggr).
\]
\end{proof}

\begin{corollary}\label{coro:convergence}
Under the conditions of Lemma~\ref{lemma:ecdf_convergence}, if $\delta_m \to 0$ as $m\to \infty$, then the approximate \(\hat{F}_m\) based on $\hat{Q}_1, \cdots, \hat{Q}_m$ in our problem satisfies that \(\hat{F}_m\) converges to \(F\) uniformly in probability:
  \[
  \sup_{x \in \bR} \big| \hat{F}_m(x) - F(x) \big| \xrightarrow{\P} 0 \quad \text{as } m \to \infty.
  \]

\end{corollary}

The next tool we need is about a conditional distribution. In particular, suppose a $d$-dimensional random variable $X = (X_1,\ldots,X_d)^\top \sim F$, we aim to estimate the conditional CDF
\[
 F^{d \mid 1:(d-1)}(x \mid x_1,\dots,x_{d-1})
 =
 \P(X_d \le x \mid X_1 = x_1,\dots,X_{d-1}=x_{d-1}).
\]
This is because we have to estimate the conditional CDF when applying the DIP procedure by following the probability chain rule. In the following result, we will use a kernel estimator. Let \(K:\bR^{d-1}\to[0, \infty)\) be a bounded, continuous kernel with 
\(\int_{\bR^{d-1}} K(u)\,\mathrm du = 1\), 
and let \(h>0\) be a bandwidth parameter. Define the \emph{kernel-based} conditional CDF estimator:
\[
 \hat{F}_{m}^{d \mid 1:(d-1)}(x \mid x_1,\dots,x_{d-1})
 :=
 \frac{
  \sum_{i=1}^m \mathbf{1}_{\{\hat{Q}_{i,d} \le x\}}
   K\left(\tfrac{x_1 - \hat{Q}_{i,1}}{h}, \dots, \tfrac{x_{d-1} - \hat{Q}_{i,d-1}}{h}\right)
 }{
  \sum_{i=1}^m
   K\left(\tfrac{x_1 - \hat{Q}_{i,1}}{h}, \dots, \tfrac{x_{d-1} - \hat{Q}_{i,d-1}}{h}\right)
 }.
\]
We assume that the denominator is nonzero, or use a small regularizing constant if needed. Then the following result can be seen as a generalization of Lemma~\ref{lemma:ecdf_convergence}.

\begin{lemma}[Uniform Convergence of Kernel-based Conditional Empirical CDF with Perturbed Data]\label{lem:conditional-cdf-convergence}
Let \( Q = \{ Q_i \}_{i=1}^m \) be i.i.d. $d$-dimensional random vectors from a distribution \( F \) on \( \bR^d \), where \( Q_i = (Q_{i,1}, \dots, Q_{i,d}) \).
Let \( \hat{Q} = \{ \hat{Q}_i \}_{i=1}^m \) be a perturbed version of \(Q\), which can be potentially dependent.
Let \( K:\bR^{d-1}\to[0, \infty) \) be a bounded (by a constant $K_{\max}$) and Lipschitz continuous kernel with \( \int_{\bR^{d-1}} K(u)\,\mathrm du = 1 \).
Let \( h>0 \) be a bandwidth parameter. Define the \emph{kernel-based} conditional CDF estimator:
\[
\hat{F}_{m}^{d \mid 1:(d-1)}(x \mid u)
 :=
\frac{
\sum_{i=1}^m \mathbf{1}_{\{\hat{Q}_{i,d} \le x\}}
 K\left(\tfrac{u - \hat{Q}_{i,1:(d-1)}}{h}\right)
 }{
 \sum_{i=1}^m
 K\left(\tfrac{u - \hat{Q}_{i,1:(d-1)}}{h}\right)
},
\]
where \( u = (x_1,\dots,x_{d-1}) \) and \( \hat{Q}_{i,1:(d-1)} = (\hat{Q}_{i,1},\dots,\hat{Q}_{i,d-1}) \). We assume that the denominator is nonzero for any \(u\) in the domain of interest, or use a small regularizing constant if needed.

Assume the following regularity conditions:
\begin{itemize}
\item The joint density \( f(x_1,\dots,x_d) \) of \( F \) is continuous and bounded by a constant \( C_{\mathrm{up}} > 0 \).
\item The marginal density of the conditioning variables, \( f_{1:(d-1)}(u) \), is continuous and bounded. Furthermore, there exists a compact set \( \mathcal{U} \subset \bR^{d-1} \) such that \( \inf_{u \in \mathcal{U}} f_{1:(d-1)}(u) > 0 \).
\item The conditional CDF \( F^{d \mid 1:(d-1)}(x \mid u) \) is Lipschitz continuous with respect to \( u = (x_1,\dots,x_{d-1}) \).
\item The perturbed data satisfies \( \sup_{1\le i\le m}\|\hat{Q}_i - Q_i\| \leq \delta_m \) where $\norm{\cdot}$ is the Euclidean norm.
\end{itemize}
If \( \delta_m = o((\log m)^{-c} )\) for some constant $c>0$ and we choose $h = (\log m)^{-c}$, we have 

\[
 \sup_{u \in \mathcal{U},\, x \in \bR}
\big|\hat{F}_{m}^{d \mid 1:(d-1)}(x \mid u)
-
 F^{d \mid 1:(d-1)}(x \mid u)
\big|
=
o_\p(1).
\]
Hence \(\hat{F}_{m}^{d \mid 1:(d-1)}\) converges \emph{uniformly} to \(F^{d \mid 1:(d-1)}\) in probability over \( \mathcal{U} \times \bR \).
\end{lemma}

\begin{proof}
Let $d' = d-1$. Define the oracle estimator $\tilde{F}_{m}^{d \mid 1:d'}(x \mid u)$ to be the same kernel estimator but constructed using the unperturbed data \(\{Q_i\}\). By the triangle inequality, the error decomposes as:
\begin{align*}
 \sup_{u\in\mathcal U,\,x\in\bR} \big|\hat{F}_{m}^{d \mid 1:d'}(x \mid u) - F^{d \mid 1:d'}(x \mid u)\big|
 &\le
 \underbrace{\sup_{u\in\mathcal U,\,x\in\bR} \big|
  \hat{F}_{m}^{d \mid 1:d'}(x \mid u) - \tilde{F}_{m}^{d \mid 1:d'}(x \mid u)
 \big|}_{(\textsf{I})}\\
 &\quad +
 \underbrace{\sup_{u\in\mathcal U,\,x\in\bR} \big|
  \tilde{F}_{m}^{d \mid 1:d'}(x \mid u) - F^{d \mid 1:d'}(x \mid u)
 \big|}_{(\textsf{II})}.
\end{align*}

\paragraph*{Term (II): Standard Nonparametric Convergence.}
Term (II) represents the uniform error of a standard kernel-based conditional CDF estimator. Under the assumptions that the marginal density $f_{1:d'}$ is bounded away from zero on $\mathcal{U}$ and the conditional CDF is Lipschitz, standard results on ukernel estimators (e.g., from \cite{wasserman2006all} for rates of kernel regression) provide the following rate:
\[
 \sup_{u \in \mathcal{U},\, x \in \bR}
 \big|
  \tilde{F}_{m}^{d \mid 1:d'}(x \mid u)
  - F^{d \mid 1:d'}(x \mid u)
 \big|
 =
 \ocal_\P\left(
  h + \sqrt{\frac{\log m}{m h^{d'}}}
 \right).
\]
Here, $h$ corresponds to the smoothing bias and the second term corresponds to the uniform stochastic fluctuation (including the logarithmic factor for uniformity).

\paragraph*{Term (I): Perturbation Analysis.}
We analyze the effect of the data perturbation. For fixed $(x,u)$, write the estimators as ratios:
\[
 \hat{F}_{m}^{d \mid 1:d'}(x \mid u) = \frac{A(x,u)}{B(u)},
 \quad
 \tilde{F}_{m}^{d \mid 1:d'}(x \mid u) = \frac{C(x,u)}{D(u)},
\]
where
\[
 A(x,u)
 =
 \sum_{i=1}^m \mathbf{1}_{\{\hat{Q}_{i,d} \le x\}}
  K\biggl(\frac{u - \hat{Q}_{i,1:d'}}{h}\biggr),
 \quad
 B(u)
 =
 \sum_{i=1}^m
  K\biggl(\frac{u - \hat{Q}_{i,1:d'}}{h}\biggr),
\]
\[
 C(x,u)
 =
 \sum_{i=1}^m \mathbf{1}_{\{Q_{i,d} \le x\bigr\}}
  K\biggl(\frac{u - Q_{i,1:d'}}{h}\biggr),
 \quad
 D(u)
 =
 \sum_{i=1}^m
  K\biggl(\frac{u - Q_{i,1:d'}}{h}\biggr).
\]

The difference is bounded by:
\[
 \left| \frac{A}{B} - \frac{C}{D} \right|
 = \left| \frac{(A-C)D - C(B-D)}{B D} \right|
 \le \frac{|A-C|}{|B|} + \frac{|C|}{|B|}\frac{|B-D|}{|D|}.
\]
We proceed by establishing the stochastic orders of the components uniformly over $\mathcal{U} \times \bR$.

\begin{itemize}
 \item \textbf{Controlling $B$ and $D$:}
 The sum $D(u)$ relates to the kernel density estimator $\hat{f}(u)$ via $D(u) = m h^{d'} \hat{f}(u)$. Since $\hat{f}$ converges uniformly to $f_{1:d'}$ and $\inf_{\mathcal{U}} f_{1:d'} = c_f > 0$, we have uniformly in probability:
 \[
  D(u) \ge \frac{c_f}{2} m h^{d'}\implies D(u) = \Omega_\P(m h^{d'}).
 \]
 Since $\delta_m \to 0$, we will show $|B-D|$ is of lower order, implying $B(u) = \Theta_\P(m h^{d'})$ as well. Note also that $|C(x,u)| \le D(u)$, so the ratio $|C/D| \le 1$.

 \item \textbf{Controlling $|B-D|$:}
 Using the Lipschitz property of $K$ (with constant $L_K$) and the bound $\|\hat{Q}_i - Q_i\| \le \delta_m$:
 \[
  |B(u) - D(u)| \le \sum_{i=1}^m \left| K\biggl(\frac{u - \hat{Q}_{i,1:d'}}{h}\biggr) - K\biggl(\frac{u - Q_{i,1:d'}}{h}\biggr) \right|.
 \]
 The Lipschitz difference is bounded by $L_K\|\hat{Q}_{i,1:d'} - Q_{i,1:d'}\|/h \le L_K \delta_m/h$.
 Importantly, this difference is non-zero only if the data point falls within the effective support of the kernel. Thhus this summation can be controlled in exactly the same manner as detailed in the proof of Lemma~\ref{lemma:ecdf_convergence}. Briefly speaking, since $K$ has compact support, there exists a constant $R > 0$ such that the support is contained within a hypercube $[-R, R]^{d'}$.
 
 Consequently, the term for index $i$ is non-zero only if the unperturbed point $Q_{i,1:d'}$ falls within a neighborhood of $u$ with radius roughly $Rh + \delta_m$. Since $\delta_m = o(h)$, this effective support is contained within a hypercube $\mathcal{C}_{u,h}$ centered at $u$ with side length proportional to $h$. The volume of this hypercube is $\ocal(h^{d'})$.

 Let $N_h(u)$ be the number of data points $Q_{i,1:d'}$ falling into this local neighborhood. Since the marginal density $f_{1:d'}$ is bounded by $C_{\mathrm{up}}$, the probability that a single observation falls into this region is bounded by:
 \[
  \mathbb{P}(Q_{i,1:d'} \in \mathcal{C}_{u,h}) \le C_{\mathrm{up}} \cdot \mathsf{Vol}(\mathcal{C}_{u,h}) = \ocal(h^{d'}).
 \]
 Thus, the number of active points $N_h(u)$ behaves as a sum of Bernoulli trials with success probability $\ocal(h^{d'})$. The expected count is $\ocal(m h^{d'})$. Since the density is bounded, the count is uniformly bounded in probability by the same order, i.e., $\sup_{u} N_h(u) = \ocal_\P(m h^{d'})$.
 
 Summing the Lipschitz error bound over these active points yields:
 \[
  \sup_{u \in \mathcal{U}} |B(u) - D(u)| \le N_h(u) \frac{L_K \delta_m}{h} = \ocal_\P\left( m h^{d'} \frac{\delta_m}{h} \right) = \ocal_\P( m h^{d'-1} \delta_m ).
 \]
When we have $\delta_m = o(h)$, this term is in a lower order compared with $D(u)$ and thus $B(u)$ is the same order of $D(u)$.
 \item \textbf{Controlling $|A-C|$:}
 We decompose the numerator difference:
\begin{align*}
  |A - C|  & =\left|\sum_{i=1}^m \mathbf{1}_{\{\hat{Q}_{i,d} \le x\}}
  K\biggl(\frac{u - \hat{Q}_{i,1:d'}}{h}\biggr) - \sum_{i=1}^m \mathbf{1}_{\{Q_{i,d} \le x\}}
  K\biggl(\frac{u - Q_{i,1:d'}}{h}\biggr)\right|\\
  & \le \sum_{i=1}^m\mathbf{1}_{\{\hat{Q}_{i,d} \le x\}}\left|K\biggl(\frac{u - \hat{Q}_{i,1:d'}}{h}\biggr)-K\biggl(\frac{u - Q_{i,1:d'}}{h}\biggr)\right|\\ 
  & \quad +  \sum_{i=1}^m \big| \mathbf{1}_{\{\hat{Q}_{i,d} \le x\}} - \mathbf{1}_{\{Q_{i,d} \le x\}} \big| K\biggl(\frac{u - Q_{i,1:d'}}{h}\biggr)\\
  & =\Delta_{\mathrm{kernel}} + \Delta_{\mathrm{ind}}.
 \end{align*}
 $\Delta_{\mathrm{kernel}}$ can be directly bounded by $|B-D|$, since the indicator is $\le 1$. Thus, 
 \[
 \Delta_{\mathrm{kernel}} = \ocal_\P(m h^{d'-1} \delta_m).
 \]
Combining and merging the bounds, we have
 \[
  \sup_{u\in\mathcal U,\, x\in\bR} |A(x,u) - C(x,u)| = \ocal_\P( m h^{d'-1} \delta_m ).
 \]
\end{itemize}

Substituting these rates back into Term (I), we get 
\[
 \frac{|A-C|}{|B|} = \frac{\ocal_\P(m h^{d'-1} \delta_m)}{\Theta_\P(m h^{d'})} = \ocal_\P\left( \frac{\delta_m}{h} \right),
\]
\[
 \frac{|B-D|}{|D|} = \frac{\ocal_\P(m h^{d'-1} \delta_m)}{\Theta_\P(m h^{d'})} = \ocal_\P\left( \frac{\delta_m}{h} \right).
\]

Therefore, we have Term (I) to be $\ocal_\P(\delta_m/h)$.

Combining Term (I) and Term (II) yields
\[
 \sup_{u\in\mathcal U,\, x\in\bR} \big|\hat{F}_{m}^{d\mid 1:d'}(x\mid u) - F^{d\mid 1:d'}(x\mid u)\big| = \ocal_\P\left( h + \sqrt{\frac{\log m}{m h^{d-1}}} + \frac{\delta_m}{h} \right).
\]

Picking $h = (\log m)^{-c}$ and $\delta_m = o((\log m)^{-c})$ ensures the vanishing error.

\end{proof}

\begin{theorem}\label{thm:inverse-convergence}
Let \( F \) be a continuous and strictly increasing cumulative distribution function (CDF) (or conditional CDF) on \( \bR \), and let \( F^{-1} \) denote its inverse function. Assume the conditions of Corollary~\ref{coro:convergence} hold. Let \( \hat{F}_m \) be the smoothed CDF estimator of Corollary~\ref{coro:convergence}. Then, for any closed interval \( [a, b] \subset (0, 1) \), the inverse functions \( \hat{F}_m^{-1} \) converge uniformly in probability to \( F^{-1} \) on \( [a, b] \); that is
\[
\sup_{u \in [a, b]} \big| \hat{F}_m^{-1}(u) - F^{-1}(u) \big| \xrightarrow{\P} 0 \quad \text{as } m \to \infty.
\]
Moreover, for any closed interval \( [a, b] \subset (0, 1) \), the sequence \( \{ \hat{F}_m^{-1} \} \) is uniformly equicontinuous in probability on \( [a, b] \); that is, for any \( \varepsilon, \eta > 0 \), there exist \( \delta > 0 \) and \( M \in \mathbb{N} \) such that for all \( m \geq M \),
\[
\P\left( \sup_{\substack{u, v \in [a, b] \\ |u - v| < \delta}} \big| \hat{F}_m^{-1}(u) - \hat{F}_m^{-1}(v) \big| \geq \varepsilon \right) < \eta.
\]
\end{theorem}

\begin{proof}
We aim to show that for any \( \varepsilon > 0 \),
\[
\P\left( \sup_{u \in [a, b]} \big| \hat{F}_m^{-1}(u) - F^{-1}(u) \big| > \varepsilon \right) \xrightarrow{m \to \infty} 0.
\]
Since \( F^{-1} \) is continuous and strictly increasing on \( (0, 1) \), it is uniformly continuous on the closed interval \( [a, b] \). Therefore, for any \( \varepsilon > 0 \), there exists \( \delta > 0 \) such that for all \( u, v \in [a, b] \),
\[
| u - v | < \delta \implies \big| F^{-1}(u) - F^{-1}(v) \big| < \frac{\varepsilon}{2}.
\]
Define the event
\[
\mathscr{A}_m = \left\{ \sup_{x \in \bR} \big| \hat{F}_m(x) - F(x) \big| \leq \delta \right\}.
\]
Since \( \hat{F}_m \) converges to \( F \) uniformly in probability, we have \( \P(\mathscr{A}_m) \xrightarrow{m \to \infty} 1 \).

On the event \( \mathscr{A}_m \), for all \( x \in \bR \),
\[
\big| \hat{F}_m(x) - F(x) \big| \leq \delta.
\]
In particular, at \( x = F^{-1}(u) \),
\[
\big| \hat{F}_m(F^{-1}(u)) - F(F^{-1}(u)) \big| = \big| \hat{F}_m(F^{-1}(u)) - u \big| \leq \delta,
\]
which implies \( \hat{F}_m(F^{-1}(u)) \in [u - \delta, u + \delta] \).

For \( x \leq F^{-1}(u - \delta) \), since \( F \) is strictly increasing,
\[
F(x) \leq u - \delta \implies \hat{F}_m(x) \leq F(x) + \delta \leq u - \delta + \delta = u.
\]
Thus, \( \hat{F}_m(x) \leq u \) for all \( x \leq F^{-1}(u - \delta) \).

Similarly, for \( x \geq F^{-1}(u + \delta) \),
\[
F(x) \geq u + \delta \implies \hat{F}_m(x) \geq F(x) - \delta \geq u + \delta - \delta = u.
\]
Therefore, \( \hat{F}_m(x) \geq u \) for all \( x \geq F^{-1}(u + \delta) \).

By the definition of the quantile function,
\[
\hat{F}_m^{-1}(u) = \inf\left\{ x \in \bR : \hat{F}_m(x) \geq u \right\}.
\]
From the observations above, it follows that
\[
\hat{F}_m^{-1}(u) \in [ F^{-1}(u - \delta), \ F^{-1}(u + \delta) ].
\]
Therefore,
\[
\big| \hat{F}_m^{-1}(u) - F^{-1}(u) \big| \leq \max\left\{ \big| F^{-1}(u - \delta) - F^{-1}(u) \big|, \, \big| F^{-1}(u + \delta) - F^{-1}(u) \big| \right\} < \frac{\varepsilon}{2}.
\]
On the event \( \mathscr{A}_m \), this holds uniformly for all \( u \in [a, b] \), so
\[
\sup_{u \in [a, b]} \big| \hat{F}_m^{-1}(u) - F^{-1}(u) \big| < \frac{\varepsilon}{2}.
\]
Thus,
\[
\P\left( \sup_{u \in [a, b]} \big| \hat{F}_m^{-1}(u) - F^{-1}(u) \big| \geq \varepsilon \right) \leq \P(\mathscr{A}_m^c) \xrightarrow{m \to \infty} 0.
\]
This concludes the proof that
\[
\sup_{u \in [a, b]} \big| \hat{F}_m^{-1}(u) - F^{-1}(u) \big| \xrightarrow{\P} 0 \quad \text{as } m \to \infty.
\]

For the second part, let \( \varepsilon, \eta > 0 \) be given. Since \( F^{-1} \) is continuous on the closed interval \( [a, b] \), it is uniformly continuous. Therefore, there exists \( \delta > 0 \) such that for all \( u, v \in [a, b] \),
\begin{equation}\label{eq:uniform_continuity}
| u - v | < \delta \implies \big| F^{-1}(u) - F^{-1}(v) \big| < \frac{\varepsilon}{3}.
\end{equation}

From the uniform convergence in probability of \( \hat{F}_m^{-1} \) to \( F^{-1} \) established earlier, we have:
\[
\sup_{u \in [a, b]} \big| \hat{F}_m^{-1}(u) - F^{-1}(u) \big| \xrightarrow{\P} 0 \quad \text{as } m \to \infty.
\]
This means that for the given \( \varepsilon \) and \( \eta \), there exists \( M \in \mathbb{N} \) such that for all \( m \geq M \),
\begin{equation}\label{eq:uniform_convergence-0}
\P\left( \sup_{u \in [a, b]} \big| \hat{F}_m^{-1}(u) - F^{-1}(u) \big| \geq \frac{\varepsilon}{3} \right) < \frac{\eta}{2}.
\end{equation}
Define the event:
\[
\mathscr{B}_m = \left\{ \sup_{u \in [a, b]} \big| \hat{F}_m^{-1}(u) - F^{-1}(u) \big| < \frac{\varepsilon}{3} \right\}.
\]
Then, \( \P(\mathscr{B}_m) > 1 - \eta/2 \) for all \( m \geq M \).

On the event \( \mathscr{B}_m \), for any \( u, v \in [a, b] \) with \( | u - v | < \delta \), we have:
\begin{align*}
\big| \hat{F}_m^{-1}(u) - \hat{F}_m^{-1}(v) \big| &\leq \big| \hat{F}_m^{-1}(u) - F^{-1}(u) \big| + \big| F^{-1}(u) - F^{-1}(v) \big| + \big| F^{-1}(v) - \hat{F}_m^{-1}(v) \big| \\
&< \frac{\varepsilon}{3} + \frac{\varepsilon}{3} + \frac{\varepsilon}{3} = \varepsilon.
\end{align*}
Here, the first and third terms are bounded by \( \varepsilon/3 \) due to \eqref{eq:uniform_convergence-0}, and the middle term is bounded by \( \varepsilon/3 \) due to the uniform continuity of \( F^{-1} \) in \eqref{eq:uniform_continuity}.

Therefore, on the event \( \mathscr{B}_m \),
\[
\sup_{\substack{u, v \in [a, b] \\ |u - v| < \delta}} \big| \hat{F}_m^{-1}(u) - \hat{F}_m^{-1}(v) \big| < \varepsilon.
\]
Thus,
\[
\P\left( \sup_{\substack{u, v \in [a, b] \\ |u - v| < \delta}} \big| \hat{F}_m^{-1}(u) - \hat{F}_m^{-1}(v) \big| \geq \varepsilon \right) \leq P(\mathscr{B}_m^c) < \frac{\eta}{2} < \eta.
\]
This completes the proof that \( \{ \hat{F}_m^{-1} \} \) is uniformly equicontinuous in probability on \( [a, b] \).
\end{proof}

Next, we introduce the conditional counterpart of Theorem~\ref{thm:inverse-convergence}.

\begin{theorem}\label{thm:inverse-conditional-convergence}
Under the conditions of Lemma~\ref{lem:conditional-cdf-convergence}, assume additionally that:
\begin{itemize}
\item For each fixed \( u \in \mathcal{U} \), the conditional CDF \( F^{d \mid 1:(d-1)}(\cdot\mid u) \) is strictly increasing.
\item The conditional CDF \( F^{d\mid 1:(d-1)}(x\mid u) \) is jointly continuous in \( (x, u) \) for \( x \in \bR \) and \( u \in \mathcal{U} \).
\item Suppose \( \delta_m = o((\log m)^{-c} )\) for some constant $c>0$ and we choose $h = (\log m)^{-c}$ in the kernel.
\end{itemize}
Then for any closed interval \( [a,b] \subset (0,1) \), and any compact set \( K \subset \bR^{d-1} \), we have the following conclusions.
\begin{enumerate}
\item Uniform convergence:
\[
\sup_{u_0 \in K} \sup_{q \in [a,b]} \big|(\hat{F}_m^{d \mid 1:(d-1)})^{-1}(q\mid u_0) - (F^{d \mid 1:(d-1)})^{-1}(q\mid u_0)\big| \xrightarrow{\P} 0 \quad \text{as } m \to \infty.
\]

\item Uniform equicontinuity in conditioning variables: For any \( \varepsilon, \eta > 0 \), there exists \( \delta > 0 \) and \( M \in \mathbb{N} \) such that for all \( m \geq M \), for any \( u_0, u'_0 \in K \) with \( \|u_0 - u'_0\| < \delta \):
\[
\P\left( \sup_{q \in [a,b]} \big|(\hat{F}_m^{d \mid 1:(d-1)})^{-1}(q\mid u_0) - (\hat{F}_m^{d \mid 1:(d-1)})^{-1}(q\mid u'_0)\big| \geq \varepsilon \right) < \eta.
\]
\end{enumerate}
\end{theorem}

\begin{proof}
We first decompose the proof into several steps. For notational simplicity, we write \( F(\cdot\mid u_0) \) for \( F^{d \mid 1:(d-1)}(\cdot\mid u_0) \) and \( \hat{F}_m(\cdot\mid u_0) \) for \( \hat{F}_m^{d \mid 1:(d-1)}(\cdot\mid u_0) \). Also, let \( F^{-1}(\cdot\mid u_0) \) denote \( (F^{d \mid 1:(d-1)})^{-1}(\cdot\mid u_0) \) and \( \hat{F}_m^{-1}(\cdot\mid u_0) \) denote \( (\hat{F}_m^{d \mid 1:(d-1)})^{-1}(\cdot\mid u_0) \).

\paragraph*{Part 1: Uniform Convergence}
We aim to show that for any \( \varepsilon > 0 \),
\[
\P\left( \sup_{u_0 \in K} \sup_{q \in [a, b]} \big| \hat{F}_m^{-1}(q\mid u_0) - F^{-1}(q\mid u_0) \big| > \varepsilon \right) \xrightarrow{m \to \infty} 0.
\]
Since \( F(\cdot\mid u_0) \) is strictly increasing for each fixed \( u_0 \), its inverse \( F^{-1}(\cdot\mid u_0) \) is well-defined. By the joint continuity of \( F(x\mid u_0) \) in \( (x, u_0) \) and the compactness of \( [a,b] \times K \) (when $F^{-1}$ maps this to a compact set), the function \( F^{-1}(q\mid u_0) \) is uniformly continuous on \( [a, b] \times K \).
Therefore, for any \( \varepsilon > 0 \), there exists \( \delta_q > 0 \) such that for all \( q_1, q_2 \in [a, b] \) and $u_0 \in K$,
\[
| q_1 - q_2 | < \delta_q \implies \big| F^{-1}(q_1\mid u_0) - F^{-1}(q_2\mid u_0) \big| < \frac{\varepsilon}{2}.
\]
Define the event
\[
\mathscr{A}_m = \left\{ \sup_{u_0 \in K} \sup_{x \in \bR} \big| \hat{F}_m(x\mid u_0) - F(x\mid u_0) \big| \leq \delta_q \right\}.
\]
From Lemma~\ref{lem:conditional-cdf-convergence}, we have \( \P(\mathscr{A}_m) \xrightarrow{m \to \infty} 1 \).

On \( \mathscr{A}_m \), for all \( u_0 \in K \) and \( x \in \bR \),
\[
\big| \hat{F}_m(x\mid u_0) - F(x\mid u_0) \big| \leq \delta_q.
\]
Taking \( x = F^{-1}(q\mid u_0) \), for any \( q \in [a,b] \),
\[
\big| \hat{F}_m(F^{-1}(q\mid u_0)\mid u_0) - F(F^{-1}(q\mid u_0)\mid u_0) \big| = \big| \hat{F}_m(F^{-1}(q\mid u_0)\mid u_0) - q \big| \leq \delta_q,
\]
which implies \( \hat{F}_m(F^{-1}(q\mid u_0)\mid u_0) \in [q - \delta_q, q + \delta_q] \).

For all \( x \leq F^{-1}(q - \delta_q\mid u_0) \), since \( F(\cdot\mid u_0) \) is strictly increasing,
\[
F(x\mid u_0) \leq q - \delta_q \implies \hat{F}_m(x\mid u_0) \leq F(x\mid u_0) + \delta_q \leq q - \delta_q + \delta_q = q.
\]
Thus, $\hat{F}_m(x\mid u_0)\le q$ for all $x\le F^{-1}(q-\delta_q\mid u_0)$.

Similarly, for \( x \geq F^{-1}(q + \delta_q\mid u_0) \),
\[
F(x\mid u_0) \geq q + \delta_q \implies \hat{F}_m(x\mid u_0) \geq F(x\mid u_0) - \delta_q \geq q + \delta_q - \delta_q = q.
\]
Therefore, \( \hat{F}_m(x\mid u_0) \geq q \) for all \( x \geq F^{-1}(q + \delta_q\mid u_0) \).

Recall that
\[
\hat{F}_m^{-1}(q\mid u_0) = \inf\left\{ x \in \bR : \hat{F}_m(x\mid u_0) \geq q \right\}.
\]
We can see
\[
\hat{F}_m^{-1}(q\mid u_0) \in [ F^{-1}(q - \delta_q\mid u_0), \ F^{-1}(q + \delta_q\mid u_0) ].
\]
Therefore,
\[
\big| \hat{F}_m^{-1}(q\mid u_0) - F^{-1}(q\mid u_0) \big| \leq \max\left\{ \big| F^{-1}(q - \delta_q\mid u_0) - F^{-1}(q\mid u_0) \big|, \, \big| F^{-1}(q + \delta_q\mid u_0) - F^{-1}(q\mid u_0) \big| \right\} < \frac{\varepsilon}{2}.
\]
On the event \( \mathscr{A}_m \), this holds uniformly for all \( u_0 \in K \) and \( q \in [a, b] \), so
\[
\sup_{u_0 \in K} \sup_{q \in [a, b]} \big| \hat{F}_m^{-1}(q\mid u_0) - F^{-1}(q\mid u_0) \big| < \frac{\varepsilon}{2}.
\]

\paragraph*{Part 2: Uniform Equicontinuity}
For any given \( \varepsilon, \eta > 0 \), since \( F^{-1}(q\mid u_0) \) is jointly continuous in \( (q, u_0) \) on the compact set \( [a, b] \times K \), it is uniformly continuous on this domain.
Therefore, there exists \( \delta_u > 0 \) such that for any \( u_0, u'_0 \in K \) with \( \|u_0 - u'_0\| < \delta_u \), and any \( q \in [a,b] \):
\begin{equation}\label{eq:F_inv_uniform_continuity_u}
\big|F^{-1}(q\mid u_0) - F^{-1}(q\mid u'_0)\big| < \frac{\varepsilon}{3}.
\end{equation}

From the uniform convergence in probability of \( \hat{F}_m^{-1}(\cdot\mid \cdot) \) to \( F^{-1}(\cdot\mid \cdot) \) in the previous step, we know
\[
\sup_{u_0 \in K} \sup_{q \in [a, b]} \big| \hat{F}_m^{-1}(q\mid u_0) - F^{-1}(q\mid u_0) \big| \xrightarrow{\P} 0 \quad \text{as } m \to \infty.
\]
This means that for the given \( \varepsilon \) and \( \eta \), there exists \( M \in \mathbb{N} \) such that for all \( m \geq M \),
\begin{equation}\label{eq:F_inv_hat_uniform_convergence}
\P\left( \sup_{u_0 \in K} \sup_{q \in [a, b]} \big| \hat{F}_m^{-1}(q\mid u_0) - F^{-1}(q\mid u_0) \big| \geq \frac{\varepsilon}{3} \right) < \frac{\eta}{2}.
\end{equation}
Define the event:
\[
\mathscr{B}_m = \left\{ \sup_{u_0 \in K} \sup_{q \in [a, b]} \big| \hat{F}_m^{-1}(q\mid u_0) - F^{-1}(q\mid u_0) \big| < \frac{\varepsilon}{3} \right\}.
\]
Then, \( P(\mathscr{B}_m) > 1 - \eta/2 \) for all \( m \geq M \).

On the event \( \mathscr{B}_m \), for any \( u_0, u'_0 \in K \) with \( \|u_0 - u'_0\| < \delta_u \), and for any \( q \in [a,b] \), we have:
\begin{align*}
\big| \hat{F}_m^{-1}(q\mid u_0) - \hat{F}_m^{-1}(q\mid u'_0) \big| &\leq \big| \hat{F}_m^{-1}(q\mid u_0) - F^{-1}(q\mid u_0) \big| \\
&\quad + \big| F^{-1}(q\mid u_0) - F^{-1}(q\mid u'_0) \big| \\
&\quad + \big| F^{-1}(q\mid u'_0) - \hat{F}_m^{-1}(q\mid u'_0) \big| \\
&< \frac{\varepsilon}{3} + \frac{\varepsilon}{3} + \frac{\varepsilon}{3} = \varepsilon.
\end{align*}
Here, the first and third terms are bounded by \( \varepsilon/3 \) due to \eqref{eq:F_inv_hat_uniform_convergence}, and the middle term is bounded by \( \varepsilon/3 \) due to the uniform continuity of \( F^{-1}(\cdot\mid \cdot) \) in \eqref{eq:F_inv_uniform_continuity_u}.

Therefore, on the event \( \mathscr{B}_m \),
\[
\sup_{\substack{u_0, u'_0 \in K \\ \|u_0 - u'_0\| < \delta_u}} \sup_{q \in [a,b]} \big| \hat{F}_m^{-1}(q\mid u_0) - \hat{F}_m^{-1}(q\mid u'_0) \big| < \varepsilon.
\]
Thus,
\[
\P\left( \sup_{\substack{u_0, u'_0 \in K \\ \|u_0 - u'_0\| < \delta_u}} \sup_{q \in [a,b]} \big| \hat{F}_m^{-1}(q\mid u_0) - \hat{F}_m^{-1}(q\mid u'_0) \big| \geq \varepsilon \right) \leq \P(\mathscr{B}_m^c) < \frac{\eta}{2} < \eta.
\]
This completes the proof.
\end{proof}

\subsection{Latent distribution consistency: one dimensional case}

We first consider the one dimensional GRAND when all latent vectors are univariate ($d=1$). We assume that $\hat{F}_m$ which satisfies Corollary~\ref{coro:convergence} and Theorem~\ref{thm:inverse-convergence} is available, and we focus on the $n$ i.i.d. latent variables $Z_i \sim F$, $i = 1, \ldots, n$, for the released network, and $\hat{Z}_i$'s are the approximations of $Z_i$'s satisfying $\max_{i\in[n]}\norm{\hat{Z}_i - Z_i}\le \delta_m$ for $\delta_m = o\left(m^{-\tfrac{2\alpha}{2\alpha + (d-1)}}\right)$. Let \( \hat{F}_m \) be a kernel-smoothed CDF estimator following Lemma~\ref{lem:conditional-cdf-convergence} for $d=1$. 

Recall that the privatized latent vector in this case is
\[
\tilde{Z}_i = \hat{F}_m^{-1}( G( \hat{F}_m(\hat{Z}_i) + e_i ) )
\]
where $e_i$ is a Laplace random variable for the privacy budget, independent of everything else, and $G$ is the CDF of the distribution of $\text{Uniform}(0,1)+\text{Laplace}(0,1/\varepsilon)$.
We want to show that $\tilde{Z}_i$ follows $F$ asymptotically.

Our proving strategy takes an intermediate random variable
\[
\bar{Z}_i = \hat{F}_m^{-1}( G( \hat{F}_m(Z_i) + e_i ) ),
\]
involving the same $e_i$. In our proof, we will first show that $\bar{Z}_i$ weakly converges to the correct distribution first, and then prove that $\tilde{Z}_i-\bar{Z}_i$ converges to zero in probability.

\begin{lemma}[Marginal Convergence of \(\bar{Z}_i\) to \(F\)]\label{theorem:Si_convergence}
Under Assumptions~\ref{ass:continuous}--\ref{ass:embedding-concentration}, then for any given $i$ such that $1\le i\le n$, we have
\[
\bar{Z}_i = \hat{F}_m^{-1}( G( \hat{F}_m(Z_i) + e_i ) ) \xrightarrow{d} Y \sim F \quad \text{as } m \to \infty.
\]
\end{lemma}

\begin{proof}
First, by Corollary~\ref{coro:convergence}, \(\hat{F}_m\) converges uniformly to \(F\) in probability:
\[
\sup_{x \in \bR} \big| \hat{F}_m(x) - F(x) \big| \xrightarrow{\P} 0 \quad \text{as } m \to \infty.
\]

Next, for \(Z_i \sim F\) that is independent of $\hat{F}_m$, we have \(F(Z_i) \sim \text{Uniform}(0,1)\). Because \(\hat{F}_m\) converges uniformly to \(F\) in probability, we know that, as $m\to \infty$,
\[
\hat{F}_m(Z_i) \xrightarrow{\P} F(Z_i) = U_i \sim \text{Uniform}(0,1). 
\]
Since \(e_i\) is independent of \(\hat{F}_m(Z_i)\), we have
\[
\hat{F}_m(Z_i) + e_i \xrightarrow{d} U_i + e_i \sim \text{Uniform}(0,1) + \text{Laplace}(0,1/\varepsilon).
\]
Therefore, the Continuous Mapping Theorem indicates
\[
G(\hat{F}_m(Z_i) + e_i) \xrightarrow{d} G(U_i + e_i) \sim \text{Uniform}(0,1)\quad \text{as }m \to \infty.
\]

Finally, we want to show that
\[
\bar{Z}_i = \hat{F}_m^{-1}( G( \hat{F}_m(Z_i) + e_i ) ) \xrightarrow{d} Y \sim F \quad \text{as } m \to \infty,
\]
where \(Y\) has cumulative distribution function \(F\).

For this part, we need the result of Theorem~\ref{thm:inverse-convergence}, 
the uniform convergence of \(\hat{F}_m^{-1}\) to \(F^{-1}\) in probability. Note the function \(\hat{F}_m^{-1}\) and the variable \(G(\hat{F}_m(Z_i) + e_i)\) both depend on \(\hat{F}_m\), introducing dependence between them. To deal with it, we will show that \(\bar{Z}_i\) converges in distribution to \(Y \sim F\) by introducing an intermediate term \(T_i = F^{-1}(G(F(Z_i) + e_i))\), which has a distribution \(F\). Our strategy is to show that \(\bar{Z}_i\) is close to \(T_i\) in probability. We can write:
\begin{align*}
\bar{Z}_i - T_i &= \left[ \hat{F}_m^{-1}( G( \hat{F}_m(Z_i) + e_i ) ) - \hat{F}_m^{-1}( G( F(Z_i) + e_i ) ) \right] \\
&\quad + \left[ \hat{F}_m^{-1}( G( F(Z_i) + e_i ) ) - F^{-1}( G( F(Z_i) + e_i ) ) \right].
\end{align*}
Denote:
\[
A_i = \hat{F}_m^{-1}( G( \hat{F}_m(Z_i) + e_i ) ) - \hat{F}_m^{-1}( G( F(Z_i) + e_i ) ),
\]
\[
B_i = \hat{F}_m^{-1}( G( F(Z_i) + e_i ) ) - F^{-1}( G( F(Z_i) + e_i ) ).
\]
Then,
\[
\bar{Z}_i - T_i = A_i + B_i.
\]

We aim to prove that
\[
|A_i| = \big| \hat{F}_m^{-1}( G( \hat{F}_m(Z_i) + e_i ) ) - \hat{F}_m^{-1}( G( F(Z_i) + e_i ) ) \big| \xrightarrow{\P} 0.
\]

Since \( \hat{F}_m \) converges to \( F \) uniformly in probability, and \( Z_i \) is independent of \( \hat{F}_m \), it follows that
\[
\hat{F}_m(Z_i) \xrightarrow{\P} F(Z_i).
\]
That is, for any \( \delta > 0 \),
\[
\P\left( \big| \hat{F}_m(Z_i) - F(Z_i) \big| \geq \delta \right) \xrightarrow{m \to \infty} 0.
\]

As \( G \) is uniformly continuous, for any \( \varepsilon' > 0 \), there exists \( \delta > 0 \) such that
\[
| u - v | < \delta \implies | G(u) - G(v) | < \varepsilon'.
\]
Thus, when \( \big| \hat{F}_m(Z_i) - F(Z_i) \big| < \delta \), we have:
\[
\big| G(\hat{F}_m(Z_i) + e_i) - G(F(Z_i) + e_i) \big| < \varepsilon'.
\]

From Theorem~\ref{thm:inverse-convergence}, the sequence \( \{ \hat{F}_m^{-1} \} \) is uniformly equicontinuous in probability on any closed interval \( [a, b] \subset (0,1) \). Specifically, for any \( \varepsilon > 0 \) and \( \eta > 0 \), there exists \( \delta' > 0 \) and \( M \in \mathbb{N} \) such that for all \( m \geq M \):
\[
\P\left( \sup_{\substack{u, v \in [a, b] \\ | u - v | < \delta'}} \big| \hat{F}_m^{-1}(u) - \hat{F}_m^{-1}(v) \big| \geq \varepsilon \right) < \eta.
\]

Given an arbitrary \( \varepsilon > 0 \), choose \( \varepsilon' = \delta' \) corresponding to \( \varepsilon \) in the uniform equicontinuity condition of \( \hat{F}_m^{-1} \), and select \( \delta \) corresponding to \( \varepsilon' \) in the uniform continuity of \( G \).

Define the event:
\[
\mathscr{E}_m = \left\{ \big| \hat{F}_m(Z_i) - F(Z_i) \big| < \delta \right\} \cap \left\{ \sup_{\substack{u, v \in [a, b] \\ | u - v | < \delta'}} \big| \hat{F}_m^{-1}(u) - \hat{F}_m^{-1}(v) \big| < \varepsilon \right\}.
\]
Since \( \hat{F}_m(Z_i) \xrightarrow{\P} F(Z_i) \) and the second event occurs with high probability, \( \P(\mathscr{E}_m) \xrightarrow{m \to \infty} 1 \). Note that we can make this event even stronger as 
\[
\mathscr{E}_m = \left\{ \sup_{1\le i\le n}\big| \hat{F}_m(Z_i) - F(Z_i) \big| < \delta \right\} \cap \left\{ \sup_{\substack{u, v \in [a, b] \\ | u - v | < \delta'}} \big| \hat{F}_m^{-1}(u) - \hat{F}_m^{-1}(v) \big| < \varepsilon \right\}
\]
because the convergence of $\hat{F}_m$ is uniform.

On the event \( \mathscr{E}_m \):
\[
 \big| \hat{F}_m(Z_i) - F(Z_i) \big| < \delta \implies | G(u_1) - G(u_2) | < \varepsilon'.
\]
Then, since \( | G(u_1) - G(u_2) | < \varepsilon' = \delta' \), the uniform equicontinuity of \( \hat{F}_m^{-1} \) implies:
\[
\big| \hat{F}_m^{-1}(G(u_1)) - \hat{F}_m^{-1}(G(u_2)) \big| < \varepsilon.
\]
Therefore, on \( \mathscr{E}_m \), we have, for all $i$ simultaneously,
\[
| A_i | = \big| \hat{F}_m^{-1}(G(u_1)) - \hat{F}_m^{-1}(G(u_2)) \big| < \varepsilon.
\]

Since \( \P(\mathscr{E}_m^c) \xrightarrow{m \to \infty} 0 \), we have \(| A_i | \xrightarrow{\P} 0.\)

\medskip
Next, we prove that
\[
|B_i| = \big| \hat{F}_m^{-1}( G( F(Z_i) + e_i ) ) - F^{-1}( G( F(Z_i) + e_i ) ) \big| \xrightarrow{\P} 0.
\]

From our earlier result, for any \( \varepsilon > 0 \) and \( \eta' > 0 \), there exists \( M \in \mathbb{N} \) such that for all \( m \geq M \):
\begin{equation}\label{eq:uniform_convergence}
\P\left( \sup_{u \in [a, b]} \big| \hat{F}_m^{-1}(u) - F^{-1}(u) \big| \geq \varepsilon \right) < \eta'.
\end{equation}

We know that \( U_i = G\left( F(Z_i) + e_i \right) \in [0,1] \) and it is independent of $\hat{F}_m$.

Select a closed interval \( [a, b] \subset (0,1) \) such that
\begin{equation}\label{eq:Ui_in_ab}
\P( U_i \in [a, b] ) > 1 - \eta'.
\end{equation}

On the event \( \mathscr{E}_m' = \left\{ U_i \in [a, b] \right\} \cap \left\{ \sup_{u \in [a, b]} \big| \hat{F}_m^{-1}(u) - F^{-1}(u) \big| < \varepsilon \right\} \), we have
\[
|B_i| = \left| \hat{F}_m^{-1}(U_i) - F^{-1}(U_i) \right| < \varepsilon.
\]
And therefore, 
\[
\P( |B_i| \geq \varepsilon ) \leq \P(\mathscr{E}_m^{'c})\leq \P( U_i \notin [a, b] ) + \P\left( \sup_{u \in [a, b]} \big| \hat{F}_m^{-1}(u) - F^{-1}(u) \big| \geq \varepsilon \right).
\]

Using \eqref{eq:uniform_convergence} and \eqref{eq:Ui_in_ab}, we obtain
\[
\P( |B_i| \geq \varepsilon ) < \eta' + \eta' = 2\eta'.
\]
Since \( \eta' > 0 \) is arbitrary, we can make \( \P( |B_i| \geq \varepsilon ) \) as small as desired by choosing \( \eta' \) appropriately and ensuring \( m \geq M \). Therefore,
\[
|B_i| \xrightarrow{\P} 0.
\]

Combining the bounds on \(A_i\) and \(B_i\), we have:
\[
\bar{Z}_i - T_i \xrightarrow{\P} 0 \quad \text{as } m \to \infty.
\]

Lastly, it is easy to see that 
\[
T_i = F^{-1}( G( F(Z_i) + e_i ) ) = F^{-1}(U_i) \sim F,
\]
where \(U_i \sim \text{Uniform}(0,1)\). Since \(\bar{Z}_i - T_i \xrightarrow{\P} 0\) and \(T_i \xrightarrow{d} Y \sim F\), by Slutsky's theorem,
\[
\bar{Z}_i = T_i + (\bar{Z}_i - T_i) \xrightarrow{d} Y \sim F.
\]
\end{proof}

Next, recall again that
\[
\tilde{Z}_i = \hat{F}_m^{-1}( G( \hat{F}_m(\hat{Z}_i) + e_i ) ).
\]
Our goal is to show that $\tilde{Z}_i \to F$ in distribution.

\begin{lemma}[Weak Convergence of \(\tilde{Z}_i\) to \(F\)]\label{theorem:main-1d}
Suppose the conditions of Lemma~\ref{theorem:Si_convergence} hold. In addition, assume that the perturbation size \( \delta_m \) satisfies \( \delta_m \to 0 \) as \( m \to \infty \), and for the current set of observations \( \{Z_i\}_{i=1}^n \) and their perturbed versions \( \{\hat{Z}_i\}_{i=1}^n \), we have \( \sup_{1\le i\le n}|\hat{Z}_i - Z_i| \le \delta_m \) with probability tending to 1 as \( m \to \infty \).
Then for any given \( i \in [n] \), \( \tilde{Z}_i \) weakly converges to \( F \) as \( m \to \infty \).
\end{lemma}

\begin{proof}
Our goal is to show that \( \tilde{Z}_i \xrightarrow{d} Y \sim F \). We achieve this by leveraging Lemma~\ref{theorem:Si_convergence}, which states that \( \bar{Z}_i \xrightarrow{d} Y \sim F \). The remaining task is to prove that \( \tilde{Z}_i - \bar{Z}_i \xrightarrow{\P} 0 \). Specifically, we want to show that for any \( \varepsilon > 0 \), \( \P( | \tilde{Z}_i - \bar{Z}_i | \geq \varepsilon ) \xrightarrow{m \to \infty} 0 \). 

Let's define the arguments for \( \hat{F}_m^{-1} \):
\[
u_i = G( \hat{F}_m(\hat{Z}_i) + e_i ), \quad v_i = G( \hat{F}_m(Z_i) + e_i ).
\]
Then, \( | \tilde{Z}_i - \bar{Z}_i | = \big| \hat{F}_m^{-1}(u_i) - \hat{F}_m^{-1}(v_i) \big| \).
Our strategy is to show that \( |u_i - v_i| \) is small in probability, and then use the uniform equicontinuity of \( \hat{F}_m^{-1} \) from Theorem~\ref{thm:inverse-convergence}.

We first analyze \( | u_i - v_i | = \big| G( \hat{F}_m(\hat{Z}_i) + e_i ) - G( \hat{F}_m(Z_i) + e_i ) \big| \).
Since \( G \) is a CDF of a continuous distribution (specifically, $\text{Uniform}(0,1)+\text{Laplace}(0,1/\varepsilon)$), it is uniformly continuous on \( \bR \). For any \( \gamma > 0 \), there exists \( \delta_{G,\gamma} > 0 \) such that for any \( x, y \in \bR \):
\begin{equation}\label{eq:G_uniform_continuity}
| x - y | < \delta_{G,\gamma} \implies | G(x) - G(y) | < \gamma.
\end{equation}

Now, let's bound the difference in the arguments of $G$: \( \big| (\hat{F}_m(\hat{Z}_i) + e_i) - (\hat{F}_m(Z_i) + e_i) \big| = \big| \hat{F}_m(\hat{Z}_i) - \hat{F}_m(Z_i) \big| \).
Since the kernel \( K \) is bounded and Lipschitz continuous, \( \hat{F}_m \) itself is Lipschitz continuous. Let \( L_{\hat{F}_m} \) be its Lipschitz constant (proportional to \( K_{\max}/h \)).
Then, on the event where \( \sup_{1\le i\le n}|\hat{Z}_i - Z_i| \le \delta_m \) (which holds with probability tending to 1 by assumption), we have:
\[
\big| \hat{F}_m(\hat{Z}_i) - \hat{F}_m(Z_i) \big| \leq L_{\hat{F}_m} | \hat{Z}_i - Z_i | \leq L_{\hat{F}_m} \delta_m.
\]
As \( \delta_m \to 0 \), it follows that \( \big| \hat{F}_m(\hat{Z}_i) - \hat{F}_m(Z_i) \big| \xrightarrow{\P} 0 \).
This means that for any \( \delta_{G,\gamma} > 0 \), 
\[ \P\left( \big| \hat{F}_m(\hat{Z}_i) - \hat{F}_m(Z_i) \big| \geq \delta_{G,\gamma} \right) \xrightarrow{m \to \infty} 0. \]
By \eqref{eq:G_uniform_continuity}, this implies \( |u_i - v_i| \xrightarrow{\P} 0 \).

Next, we handle the mapping of \( \hat{F}_m^{-1} \).  From Theorem~\ref{thm:inverse-convergence}, the sequence \( \{ \hat{F}_m^{-1} \} \) is uniformly equicontinuous in probability on any closed interval \( [a, b] \subset (0,1) \). Specifically, for any \( \varepsilon > 0 \) and \( \eta > 0 \), there exist \( \delta' > 0 \) and \( M_1 \in \mathbb{N} \) such that for all \( m \geq M_1 \):
\begin{equation}\label{eq:Fm_inverse_uniform_equicontinuity}
\P\left( \sup_{\substack{q_1, q_2 \in [a, b] \\ | q_1 - q_2 | < \delta'}} \big| \hat{F}_m^{-1}(q_1) - \hat{F}_m^{-1}(q_2) \big| \geq \varepsilon \right) < \frac{\eta}{3}.
\end{equation}

The arguments to \( \hat{F}_m^{-1} \) are \( u_i = G(\hat{F}_m(\hat{Z}_i) + e_i) \) and \( v_i = G(\hat{F}_m(Z_i) + e_i) \). Both \( u_i \) and \( v_i \) are random variables in \( [0,1] \).
We know \( v_i \xrightarrow{d} \text{Uniform}(0,1) \). Therefore, for any \( \eta_0 > 0 \), we can choose a closed interval \( [a,b] \subset (0,1) \) such that \( \P(v_i \in [a,b]) \geq 1 - \eta_0/2 \).
Since \( |u_i - v_i| \xrightarrow{\P} 0 \) (from Step 1), it also follows that \( u_i \xrightarrow{\P} v_i \). Consequently, \( P(u_i \in [a,b]) \geq 1 - \eta_0/2 \) for sufficiently large \( m \).
Thus, for any \( \eta_0 > 0 \), there exists an interval \( [a,b] \) such that for sufficiently large \( m \),
\begin{equation}\label{eq:uivi_in_ab}
\P( u_i \in [a, b] \text{ and } v_i \in [a, b] ) \geq 1 - \eta_0.
\end{equation}

Now, let \( \varepsilon > 0 \) and \( \eta > 0 \) be arbitrary. We conduct the following steps.
\begin{enumerate}
    \item Choose \( \delta' \) and \( M_1 \) from \eqref{eq:Fm_inverse_uniform_equicontinuity} for this \( \varepsilon \) and \( \eta/3 \).
    \item Choose \( \gamma = \delta' \) for the uniform continuity of \( G \) in \eqref{eq:G_uniform_continuity}.
    \item Choose \( \delta_{G,\gamma} \) corresponding to this \( \gamma \).
    \item Since \( \big| \hat{F}_m(\hat{Z}_i) - \hat{F}_m(Z_i) \big| \xrightarrow{\P} 0 \), there exists \( M_2 \in \mathbb{N} \) such that for \( m \geq M_2 \), 
    $$ \P\left( \big| \hat{F}_m(\hat{Z}_i) - \hat{F}_m(Z_i) \big| \geq \delta_{G,\gamma} \right) < \frac{\eta}{3}.$$
    \item Choose \( \eta_0 = \eta/3 \) for \eqref{eq:uivi_in_ab}, which defines the interval \( [a,b] \) and ensures \( \P( u_i \in [a, b] \text{ and } v_i \in [a, b] ) \geq 1 - \eta/3 \) for \( m \geq M_3 \).
\end{enumerate}

Define the event \( \mathscr{E}_{m} \) as the intersection of three high-probability events:
\[
\mathscr{E}_{m} = \left\{ \big| \hat{F}_m(\hat{Z}_i) - \hat{F}_m(Z_i) \big| < \delta_{G,\gamma} \right\} \cap \left\{ u_i \in [a, b] \text{ and } v_i \in [a, b] \right\} \cap \left\{ \sup_{\substack{q_1, q_2 \in [a, b] \\ | q_1 - q_2 | < \delta'}} \big| \hat{F}_m^{-1}(q_1) - \hat{F}_m^{-1}(q_2) \big| < \varepsilon \right\}.
\]
For \( m \geq \max\{M_1, M_2, M_3\} \), we have \( \P(\mathscr{E}_{m}^c) \leq \eta/3 + \eta/3 + \eta/3 = \eta \).
On the event \( \mathscr{E}_{m} \):
\begin{itemize}
    \item We have \( |u_i - v_i| < \gamma = \delta' \).
    \item Both \( u_i \) and \( v_i \) are in \( [a,b] \).
    \item The uniform equicontinuity of \( \hat{F}_m^{-1} \) in \eqref{eq:Fm_inverse_uniform_equicontinuity} applies.
\end{itemize}
Therefore, on \( \mathscr{E}_{m} \), we have:
\[
| \tilde{Z}_i - \bar{Z}_i | = \big| \hat{F}_m^{-1}(u_i) - \hat{F}_m^{-1}(v_i) \big| < \varepsilon.
\]
Since \( \P(\mathscr{E}_{m}^c) \xrightarrow{m \to \infty} 0 \), we conclude that \( | \tilde{Z}_i - \bar{Z}_i | \xrightarrow{\P} 0 \).

This completes the proof.

\end{proof}

\subsection{Latent distribution consistency: multidimensional case}
Having established the asymptotic distribution of $\tilde{Z}_i$ in the case of $d=1$, now we proceed to prove the weak convergence in the multidimensional case. The proof for the multidimensional case is essentially applying the univariate proofs sequentially across variables, following the same DIP procedures using conditional distributions. Therefore, we will explain the details using the two dimensional case. 

\begin{theorem}\label{thm:main-multid}
Let \( Z_i = (Z_{i1}, Z_{i2}) \) be i.i.d. $2$-dimensional random vectors from a distribution \( F \) on \( \bR^2 \).
Assume that the true distribution \( F \) has compact support \( S \subset \bR^2 \). Assume the following conditions hold:
\begin{itemize}
    \item The conditions of Lemma~\ref{lem:conditional-cdf-convergence} apply to both the marginal CDF \( F_1 \) (for \( \hat{F}_{1,m} \)) and the conditional CDF \( F_{2\mid 1}(\cdot\mid z_1) \) (for \( \hat{F}_{2\mid 1,m} \)). This includes properties of relevant densities, kernels (bounded and Lipschitz).
    \item The conditions of Theorem~\ref{thm:inverse-conditional-convergence} apply to the inverse functions \( F_1^{-1} \) and \( (F_{2\mid 1})^{-1} \), and their estimators. This implies \( F_1 \) and \( F_{2\mid 1}(\cdot\mid z_1) \) are strictly increasing, and \( F_{2\mid 1}(x\mid z_1) \) is jointly continuous in \( (x, z_1) \).
    \item The perturbation size \( \delta_m \) satisfies \( \delta_m = o((\log m)^{-c}) \), and \( \sup_{1\le i\le n}\|\hat{Z}_i - Z_i\| \le \delta_m \) with probability tending to 1 as \( m \to \infty \), for some constant \( c > 0 \).
    \item For each \( k=1,2 \), the bandwidth \( h_m^{(k)} \) for \( \hat{F}_{k\mid 1:(k-1),m} \) is chosen as \( h_m^{(k)} = h_m = (\log m)^{-c} \) for the same constant \( c > 0 \). 
\end{itemize}
Then for any given \( i \in [n] \), the privatized latent vector \( \tilde{Z}_i = (\tilde{Z}_{i1}, \tilde{Z}_{i2}) \) weakly converges to \( F \) as \( m \to \infty \); that is, \( \tilde{Z}_i \xrightarrow{d} Y \sim F \).
\end{theorem}

\begin{proof}
Our goal is to show that the privatized latent vector $\tilde{Z}_i = (\tilde{Z}_{i1}, \tilde{Z}_{i2})$ weakly converges to the true latent distribution $F$. That is, $\tilde{Z}_i \xrightarrow{d} Y \sim F$, where $Y = (Y_1, Y_2)$ is a random vector with CDF $F$. We'll achieve this by demonstrating the convergence of each component in sequence and then combining them for joint convergence.

For notational simplicity within this proof, let $F_k(\cdot)$ denote the marginal CDF of $Z_k$, and $F_{k\mid 1:(k-1)}(\cdot\mid u)$ denote the true conditional CDF for the $k$-th dimension given $u = (z_1, \dots, z_{k-1})$. Similarly, $\hat{F}_{k,m}$ and $\hat{F}_{k\mid 1:(k-1),m}(\cdot\mid u)$ are their estimators. Inverse functions are denoted with $-1$. Let $S_k$ represent the compact support of $Z_k$ (derived from the compact support $S$ of $F$), and $S_{1:(k-1)}$ be the compact support for the conditioning variables $(Z_1, \dots, Z_{k-1})$. Because $F$ has compact support $S \subset \bR^d$, all $Z_{ik}$ and $(Z_{i1}, \dots, Z_{i,k-1})$ almost surely lie within their compact projected supports. This is crucial for applying uniform convergence results from previous lemmas, which hold over compact domains.

\paragraph{Step 1 -- convergence of the first coordinate $\tilde{Z}_{i1}$:}

The first component, $\tilde{Z}_{i1}$, is defined as:
$$\tilde{Z}_{i1} = \hat{F}_{1,m}^{-1}( G( \hat{F}_{1,m}(\hat{Z}_{i1}) + e_i^{(1)} ) ).$$
Directly from Lemma~\ref{theorem:main-1d}, we have:
$$\tilde{Z}_{i1} \xrightarrow{d} Y_1 \sim F_1 \quad \text{as } m \to \infty.$$

\paragraph{Step 2 -- convergence of the second coordinate $\tilde{Z}_{i2}$ (conditioning on $\tilde{Z}_{i1}$):}

We now analyze the conditional distribution of $\tilde{Z}_{i2}$ given $\tilde{Z}_{i1}=u_1$, for any \emph{fixed} $u_1 \in S_1$ (the compact support of $Y_1$). The quantity of interest is $\P(\tilde{Z}_{i2} \le x_2 \mid  \tilde{Z}_{i1}=u_1)$.
For this analysis, we treat $u_1$ as a deterministic conditioning variable. The definition of $\tilde{Z}_{i2}$ involves conditioning on $\tilde{Z}_{i1}$, so we consider the variable:
$$\tilde{Z}_{i2}(u_1) = \hat{F}_{2\mid 1,m}^{-1}( G( \hat{F}_{2\mid 1,m}(\hat{Z}_{i2}\mid \hat{Z}_{i1}) + e_i^{(2)} ) \mid  u_1 ).$$
Note that $\hat{Z}_{i1}$ is still a random variable in the argument of $\hat{F}_{2\mid 1,m}$. To properly apply the logic of Lemma~\ref{theorem:main-1d} for fixed conditioning, we will use the user-proposed intermediate variables where the outer conditioning is fixed to $u_1$, but the inner conditioning remains consistent with the true variables $Z_{i1}$.
Let's define two intermediate variables for comparison, where $u_1$ is the fixed outer conditioning value:
$$\bar{Z}_{i2}(u_1) = \hat{F}_{2\mid 1,m}^{-1}( G( \hat{F}_{2\mid 1,m}(Z_{i2}\mid Z_{i1}) + e_i^{(2)} ) \mid  u_1 ),$$
$$T_{i2}(u_1) = F_{2\mid 1}^{-1}( G( F_{2\mid 1}(Z_{i2}\mid Z_{i1}) + e_i^{(2)} ) \mid  u_1 ).$$
Our goal is to show $\tilde{Z}_{i2}(u_1) - \bar{Z}_{i2}(u_1) \xrightarrow{\P} 0$ and $\bar{Z}_{i2}(u_1) - T_{i2}(u_1) \xrightarrow{\P} 0$.

To show $\bar{Z}_{i2}(u_1) - T_{i2}(u_1) \xrightarrow{\P} 0$, we directly apply the logic from Lemma~\ref{theorem:main-1d} to the conditional setting. For a \emph{fixed} $u_1$, the functions $\hat{F}_{2\mid 1,m}(\cdot\mid u_1)$ and $F_{2\mid 1}(\cdot\mid u_1)$ are well-defined.
Let $q_C = G(\hat{F}_{2\mid 1,m}(Z_{i2}\mid Z_{i1}) + e_i^{(2)})$ and $q_D = G(F_{2\mid 1}(Z_{i2}\mid Z_{i1}) + e_i^{(2)})$.
We want to show $\big| \hat{F}_{2\mid 1,m}^{-1}(q_C\mid u_1) - F_{2\mid 1}^{-1}(q_D\mid u_1) \big| \xrightarrow{\P} 0$.
This splits into:
$$\big| \hat{F}_{2\mid 1,m}^{-1}(q_C\mid u_1) - F_{2\mid 1}^{-1}(q_C\mid u_1) \big|  + \big| F_{2\mid 1}^{-1}(q_C\mid u_1) - F_{2\mid 1}^{-1}(q_D\mid u_1) \big|.$$
\begin{itemize}
    \item Term 1: By Theorem~\ref{thm:inverse-conditional-convergence}, $\sup_{u \in S_1, q \in [a,b]} \big|\hat{F}_{2\mid 1,m}^{-1}(q\mid u) - F_{2\mid 1}^{-1}(q\mid u)\big| \xrightarrow{\P} 0$. Since $u_1$ is fixed in $S_1$ and $q_C$ is a random variable in $[0,1]$, this term goes to $0$ in probability, using the same type of proof as in Theorem~\ref{thm:inverse-conditional-convergence}.
    \item Term 2: By uniform continuity of $F_{2\mid 1}^{-1}(\cdot\mid u_1)$ w.r.t. its first argument, this term goes to $0$ in probability if $|q_C - q_D| \xrightarrow{\P} 0$.
        $|q_C - q_D| = \big|G(\hat{F}_{2\mid 1,m}(Z_{i2}\mid Z_{i1}) + e_i^{(2)}) - G(F_{2\mid 1}(Z_{i2}\mid Z_{i1}) + e_i^{(2)})\big|$.
        By uniform continuity of $G$, this goes to $0$ if $\big|\hat{F}_{2\mid 1,m}(Z_{i2}\mid Z_{i1}) - F_{2\mid 1}(Z_{i2}\mid Z_{i1})\big| \xrightarrow{\P} 0$.
        This is true by Lemma~\ref{lem:conditional-cdf-convergence}.
\end{itemize}
Combining the previous two results, we have $\bar{Z}_{i2}(u_1) - T_{i2}(u_1) \xrightarrow{\P} 0$.

\bigskip

On the other hand, the term $\tilde{Z}_{i2}(u_1) - \bar{Z}_{i2}(u_1) \xrightarrow{\P} 0$ involves changes in the inner argument's conditioning variable ($\hat{Z}_{i1}$ to $Z_{i1}$) and the value being fed into $G$ ($\hat{Z}_{i2}$ to $Z_{i2}$).
$$| \tilde{Z}_{i2}(u_1) - \bar{Z}_{i2}(u_1) | = \big| \hat{F}_{2\mid 1,m}^{-1}( G( \hat{F}_{2\mid 1,m}(\hat{Z}_{i2}\mid \hat{Z}_{i1}) + e_i^{(2)} ) \mid  u_1 ) - \hat{F}_{2\mid 1,m}^{-1}( G( \hat{F}_{2\mid 1,m}(Z_{i2}\mid Z_{i1}) + e_i^{(2)} ) \mid  u_1 ) \big|.$$
Let $q_A = G(\hat{F}_{2\mid 1,m}(\hat{Z}_{i2}\mid \hat{Z}_{i1}) + e_i^{(2)})$ and $q_B = G(\hat{F}_{2\mid 1,m}(Z_{i2}\mid Z_{i1}) + e_i^{(2)})$. So we want to control
$$\big| \hat{F}_{2\mid 1,m}^{-1}(q_A\mid u_1) - \hat{F}_{2\mid 1,m}^{-1}(q_B\mid u_1) \big|.$$

By the uniform equicontinuity of $\hat{F}_{2\mid 1,m}^{-1}(\cdot\mid u_1)$ with respect to its first argument from Theorem~\ref{thm:inverse-conditional-convergence}, this term will converge to $0$ in probability if $|q_A - q_B| \xrightarrow{\P} 0$.

To see this, note that
$$|q_A - q_B| = \big| G( \hat{F}_{2\mid 1,m}(\hat{Z}_{i2}\mid \hat{Z}_{i1}) + e_i^{(2)} ) - G( \hat{F}_{2\mid 1,m}(Z_{i2}\mid Z_{i1}) + e_i^{(2)} ) \big|$$
By uniform continuity of $G$, it is sufficient to show $\big|\hat{F}_{2\mid 1,m}(\hat{Z}_{i2}\mid \hat{Z}_{i1}) - \hat{F}_{2\mid 1,m}(Z_{i2}\mid Z_{i1})\big| \xrightarrow{\P} 0$.
By plugging in an intermediate term and calling the triangle inequality, we have
\begin{align*}
\big|\hat{F}_{2\mid 1,m}(\hat{Z}_{i2}\mid \hat{Z}_{i1}) - \hat{F}_{2\mid 1,m}(Z_{i2}\mid Z_{i1})\big| &\leq \big|\hat{F}_{2\mid 1,m}(\hat{Z}_{i2}\mid \hat{Z}_{i1}) - \hat{F}_{2\mid 1,m}(Z_{i2}\mid \hat{Z}_{i1})\big| \\
&\quad + \big|\hat{F}_{2\mid 1,m}(Z_{i2}\mid \hat{Z}_{i1}) - \hat{F}_{2\mid 1,m}(Z_{i2}\mid Z_{i1})\big|.
\end{align*}
\begin{itemize}
    \item Term 1: This term relies on $\hat{F}_{2\mid 1,m}(\cdot\mid u)$ being Lipschitz continuous with respect to its first argument. So, 
    $$\big| \hat{F}_{2\mid 1,m}(\hat{Z}_{i2}\mid \hat{Z}_{i1}) - \hat{F}_{2\mid 1,m}(Z_{i2}\mid \hat{Z}_{i1}) \big| \le L_{\hat{F}_{2\mid 1,m}} |\hat{Z}_{i2} - Z_{i2}| = o_\P(1).$$
    \item Term 2: This term relies on $\hat{F}_{2\mid 1,m}(x\mid u)$ being uniformly equicontinuous with respect to its conditioning argument $u$. So, $\big|\hat{F}_{2\mid 1,m}(Z_{i2}\mid \hat{Z}_{i1}) - \hat{F}_{2\mid 1,m}(Z_{i2}\mid Z_{i1})\big| \xrightarrow{\P} 0$ because $|\hat{Z}_{i1} - Z_{i1}| \le \delta_m \xrightarrow{\P} 0$.
\end{itemize}
These indicate $|q_A - q_B| \xrightarrow{\P} 0$.

As discussed, this proves the convergence of  $\tilde{Z}_{i2}$ (conditioning on $\tilde{Z}_{i1}$).

Having obtained the convergence in both coordinates, by noticing that the joint PDF is the product of the two CDFs, we get the claimed consistency.

\end{proof}

\newpage

\section{Proof of Theorem~\ref{thm:main-CDF} (Emprical CDF Convergence)}
Next, we will show that the empirical CDF of $\tilde{Z}_i$'s uniformly converges to $F$ as the proof of Theorem~\ref{thm:main-CDF}.

Again, we first work on the univariate version, $d=1$. And we start with the intermediate random vector $\bar{Z}_i$ again. Suppose $d=1$. Define the empirical CDFs of $\bar{Z}_i$'s and $\tilde{Z}_i$'s as
\[
\bar{F}_n(x) = \frac{1}{n}\sum_{i=1}^n\mathbf{1}_{\{ \bar{Z}_i \leq x \}},\quad
\tilde{F}_n(x) = \frac{1}{n}\sum_{i=1}^n\mathbf{1}_{\{ \tilde{Z}_i \leq x \}}.
\]

\begin{lemma}[One-dimensional Intermediate Convergence]\label{lem:eCDF-zbar-1d}
 Under the assumptions of Lemma~\ref{theorem:Si_convergence}, we have
\[
\sup_{x\in\bR} \big|\bar{F}_n(x) - F(x)\big| \xrightarrow{\P} 0\quad\text{as }m,n \to \infty.
\]
\end{lemma}

\begin{proof}
  The proof will be based on expanding the proof of Lemma~\ref{lemma:ecdf_convergence} and the proof of Lemma~\ref{theorem:Si_convergence}. Use $T_i$ as in the proof of Lemma~\ref{theorem:Si_convergence} and define the empirical CDF of $T_i$'s as 
  \[
F_n(x) = \frac{1}{n}\sum_{i=1}^n\mathbf{1}_{\{ T_i \leq x \}}.
\]

The uniform convergence of $F_n$ to $F$ is already known. Hence we only focus on controlling the difference between $F_n$ and $\bar{F}_n$. 
\begin{align*}
  \big|\bar{F}_n(x)-F_n(x)\big| & = \left|\frac{1}{n}\sum_{i=1}^n\mathbf{1}_{\{ \bar{Z}_i \leq x \}}-\frac{1}{n}\sum_{i=1}^n\mathbf{1}_{\{ T_i \leq x \}}\right|\\
  & \le \frac{1}{n}\sum_{i=1}^n\big|\mathbf{1}_{\{ \bar{Z}_i \leq x \}}-\mathbf{1}_{\{ T_i \leq x \}}\big|.
\end{align*}

Similar to the proof outlined for Lemma~\ref{lemma:ecdf_convergence}, note that $\mathbf{1}_{\{ \bar{Z}_i \leq x \}}$ and $\mathbf{1}_{\{ T_i \leq x \}}$ are different only if $x$ lies between $\bar{Z}_i$ and $T_i$. So for any $\varepsilon>0$, if $|\bar{Z}_i-T_i|<\varepsilon$, that indicates $x-\varepsilon <T_i < x+\varepsilon$, whose probability can be control by $F$ since $T_i \sim F$. Specifically, we have
\begin{align}\label{eq:decomp-temp-0}
  \sup_{x\in\bR}\big|\bar{F}_n(x)-F_n(x)\big| & \le \sup_{x\in\bR}\frac{1}{n}\sum_{i=1}^n\big|\mathbf{1}_{\{ \bar{Z}_i \leq x \}}-\mathbf{1}_{\{ T_i \leq x \}}\big|\notag\\
  & \le \sup_{x\in\bR}\left\{ \frac{1}{n}\sum_{i:|\bar{Z}_i-T_i|<\varepsilon} |\mathbf{1}_{\{ \bar{Z}_i \leq x \}}-\mathbf{1}_{\{ T_i \leq x \}}| +\frac{1}{n}\sum_{i:|\bar{Z}_i-T_i| \ge \varepsilon} |\mathbf{1}_{\{ \bar{Z}_i \leq x \}}-\mathbf{1}_{\{ T_i \leq x \}}|\right\}\notag\\
  & \le \sup_{x\in\bR}\left\{ \frac{1}{n}\sum_{i:|\bar{Z}_i-T_i|<\varepsilon} \mathbf{1}_{\{ T_i \in B(x,\varepsilon) \}} +\frac{1}{n}\#\{i:|\bar{Z}_i-T_i| \ge \varepsilon\}\right\}\notag\\
  & \le \sup_{x\in\bR}\left\{ \frac{1}{n}\sum_{i=1}^n \mathbf{1}_{\{ T_i \in B(x,\varepsilon) \}} +\frac{1}{n}\#\{i:|\bar{Z}_i-T_i| \ge \varepsilon\}\right\}\notag\\
  & \le \ocal_{\P,n}(\varepsilon)+\frac{1}{n}\sum_i\mathbf{1}_{\{|\bar{Z}_i-T_i| \ge \varepsilon\}}.
\end{align}
where $\ocal_{\P,n}$ denotes that quantity order with high probability in $n$.

To control the second term, we go back to the proof of Lemma~\ref{theorem:Si_convergence} again in which we have already defined 
\[
A_i = \hat{F}_m^{-1}( G( \hat{F}_m(Z_i) + e_i ) ) - \hat{F}_m^{-1}( G( F(Z_i) + e_i ) ),
\]
\[
B_i = \hat{F}_m^{-1}( G( F(Z_i) + e_i ) ) - F^{-1}( G( F(Z_i) + e_i ) ).
\]
and thus
\[
\bar{Z}_i - T_i = A_i + B_i.
\]
Recall that in that context, we take an interval $[a, b]$, and here we specify $a, b$ to be the $\varepsilon/16$ and $1-\varepsilon/16$ quantiles of $\text{Uniform}(0,1)$. We have defined 
\[
\mathscr{E}_m = \left\{ \sup_{1\le i\le n}\big| \hat{F}_m(Z_i) - F(Z_i) \big| < \delta \right\} \cap \left\{ \sup_{\substack{u, v \in [a, b] \\ | u - v | < \delta'}} \big| \hat{F}_m^{-1}(u) - \hat{F}_m^{-1}(v) \big| < \frac{\varepsilon}{2} \right\}.
\]
We have seen that $\P(\mathscr{E}_m)\xrightarrow{m\to \infty} 1$ and on $\mathscr{E}_m$, we have $\sup_{i\in[n]}|A_i|< \varepsilon/2$.

Unfortunately, the other event in the proof, $\mathscr{E}_m'$ depends on $i$, so we could not achieve the uniform control on $B_i$'s. Instead, let us define \( \mathscr{E}_{m}' = \left\{ \sup_{u \in [a, b]} \big| \hat{F}_m^{-1}(u) - F^{-1}(u) \big| < \varepsilon/2 \right\} \). 
Again, $\P(\mathscr{E}_m')\xrightarrow{m\to \infty} 1$ as shown in Lemma~\ref{theorem:Si_convergence}.

Therefore, on the event $\mathscr{E}_{m}\cap \mathscr{E}_{m}'$, we have
$$\frac{1}{n}\sum_{i=1}^n\mathbf{1}_{\{|\bar{Z}_i-T_i| \ge \varepsilon\}} \le \frac{1}{n}\sum_{i=1}^n\mathbf{1}_{\{U_i \notin [a,b]\}}.$$
By Hoeffding's inequality, we can further define an event $\mathscr{E}_{m}^{''}$ such that 
$$\frac{1}{n}\sum_{i=1}^n\mathbf{1}_{\{U_i \notin [a,b]\}} \le \frac{\varepsilon}{2}$$
with probability tending to 1 (with respect to $n$). By union probability, we can see that the event of $\mathscr{E}_{m}\cap\mathscr{E}_{m}'\cap \mathscr{E}_{m}^{''}$ happens with probability going to 1. 

Since $\varepsilon$ is arbitrary, because of \eqref{eq:decomp-temp-0}, we have 
$$\sup_{x\in\bR}\big|\bar{F}_n(x)-F_n(x)\big|\xrightarrow{\P} 0.$$

\end{proof}

Then this convergence can be transferred to $\tilde{Z}_i$ with a similar control.
\begin{lemma}[One-dimensional Privatized Convergence]\label{lem:eCDF-ztilde-1d}
 Under the assumptions of Lemma~\ref{lem:eCDF-zbar-1d}, we further have
\[
\sup_{x\in\bR} \big|\tilde{F}_n(x) - F(x)\big| \xrightarrow{\P} 0\quad\text{as }m, n \to \infty.
\]
\end{lemma} 

\begin{proof}
  With Lemma~\ref{lem:eCDF-zbar-1d}, we only have to show that
  $$\sup_{x\in\bR}\big|\tilde{F}_n(x)-\bar{F}_n(x)\big| \xrightarrow{\P} 0.$$
Calling a similar comparison as before, for any sufficiently small $\varepsilon$,
\begin{align}\label{eq:decomp-temp}
  \sup_{x\in\bR}\big|\bar{F}_n(x)-\tilde{F}_n(x)\big| & \le \sup_{x\in\bR}\frac{1}{n}\sum_{i=1}^n\big|\mathbf{1}_{\{ \bar{Z}_i \leq x \}}-\mathbf{1}_{\{ \tilde{Z}_i \leq x \}}\big|\notag\\
  & \le \sup_{x\in\bR}\left\{ \frac{1}{n}\sum_{i:|\bar{Z}_i-\tilde{Z}_i|<\varepsilon} \big|\mathbf{1}_{\{ \bar{Z}_i \leq x \}}-\mathbf{1}_{\{ \tilde{Z}_i \leq x \}}\big| +\frac{1}{n}\sum_{i:|\bar{Z}_i-\tilde{Z}_i| \ge \varepsilon} \big|\mathbf{1}_{\{ \bar{Z}_i \leq x \}}-\mathbf{1}_{\{ \tilde{Z}_i \leq x \}}\big|\right\}\notag\\
  & \le \sup_{x\in\bR}\left\{ \frac{1}{n}\sum_{i:|\bar{Z}_i-\tilde{Z}_i|<\varepsilon} \mathbf{1}_{\{ \bar{Z}_i \in B(x,\varepsilon) \}} +\frac{1}{n}\sum_{i=1}^n\mathbf{1}_{\{|\bar{Z}_i-\tilde{Z}_i| \ge \varepsilon\}}\right\}\notag\\
& \le \sup_{x\in\bR}\left\{ \frac{1}{n}\sum_{i=1}^n \mathbf{1}_{\{ \bar{Z}_i \in B(x,\varepsilon) \}} +\frac{1}{n}\sum_{i=1}^n\mathbf{1}_{\{|\bar{Z}_i-\tilde{Z}_i| \ge \varepsilon\}}\right\}.
\end{align}

Note that $\tilde{Z}_i$, $i\in [n]$ are independent, because $\hat{F}_m$ is based only on $Z_{i}$, $i> n$. For the first term $\frac{1}{n}\sum_{i=1}^n \mathbf{1}_{\{ \bar{Z}_i \in B(x,\varepsilon) \}} $, note that $\P(Z_i \in B(x, \varepsilon)) =\ocal(\varepsilon)$ in the current assumption and because of the uniform convergence of $\bar{F}_n$ and the independence between $\tilde{Z}_i$'s, we have 
$$\sup_{x\in\bR} \frac{1}{n}\sum_{i=1}^n \mathbf{1}_{\{ \bar{Z}_i \in B(x,\varepsilon) \}} = \ocal_{\P,n}(\varepsilon).$$

To control the second term, we will reuse the quantities in the proof of Lemma~\ref{theorem:main-1d}. Recall that
\[
u_i = G( \hat{F}_m(\hat{Z}_i) + e_i ), \quad v_i = G( \hat{F}_m(Z_i) + e_i ).
\]

For the current $\varepsilon$, pick $a, b$ such that $\P(U\in [a, b]) \ge 1-\varepsilon/4$. We have
\begin{align*}
  &\left\{i: |\bar{Z}_i-\tilde{Z}_i| \ge \varepsilon\right\}\\
  \subset &\left\{i: |\bar{Z}_i-\tilde{Z}_i| \ge \varepsilon, v_i, u_i \in [a, b]\right\}\cup \left\{i: v_i \notin [a, b]\right\} \cup \left\{i: u_i \notin [a, b]\right\}\\
 =& \left\{i: \big|\hat{F}_m^{-1}(u_i)-\hat{F}_m^{-1}(v_i)\big| \ge \varepsilon, v_i, u_i \in [a, b]\right\}\cup \left\{i: v_i \notin [a, b]\right\} \cup \left\{i: u_i \notin [a, b]\right\}.
\end{align*}

Therefore, we have 
$$\frac{1}{n}\sum_{i=1}^n\mathbf{1}_{\{|\bar{Z}_i-T_i| \ge \varepsilon\}} \le \frac{1}{n}\sum_{i=1}^n\mathbf{1}_{\{|\hat{F}_m^{-1}(u_i)-\hat{F}_m^{-1}(v_i)| \ge \varepsilon, v_i, u_i \in [a, b]\}} +\frac{1}{n}\sum_{i=1}^n\mathbf{1}_{\{v_i \notin [a, b]\}}+\frac{1}{n}\sum_{i=1}^n\mathbf{1}_{\{u_i \notin [a, b]\}}.$$

As the proof of Lemma~\ref{theorem:main-1d}, Theorem~\ref{thm:inverse-convergence} indicates that there exists $\delta>0$ and $M>0$, such that for any $m\ge M$, 
$$\sup_{\substack{u,v\in [a, b]\\ |u-v|<\delta}}\big|\hat{F}^{-1}_m(u)-\hat{F}^{-1}_m(v)\big| < \varepsilon$$
with probability larger than $1-\varepsilon/4$. Define the event $\mathscr{E}^{'''}$ to be the intersection of this event and the event $\sup_{i\in[n]}|u_i-v_i|< \delta$. Under the current assumptions of $\hat{Z}_i$'s concentration and Lemma~\ref{lemma:ecdf_convergence}, we can see that 
$\P(\mathscr{E}^{'''}) \xrightarrow{m\to \infty} 1.$

The last two terms in the bound would be trivial, as they do not involve the inverse CDF estimate. Specifically, we already know the uniform convergence of $\hat{F}_m$ to $F$ under the current assumptions by Lemma~\ref{theorem:main-1d}. We know that uniformly, we have
$$u_i - G(F(Z_i)+e_i)\xrightarrow{\P} 0,\quad v_i - G(F(Z_i)+e_i)\xrightarrow{\P} 0.$$
Additionally, since $G(F(Z_i)+e_i) \sim \text{Uniform}(0,1)$, we have
$$\frac{1}{n}\sum_{i=1}^n\mathbf{1}_{\{v_i \notin [a, b]\}} = \ocal_{\P,m,n}(\varepsilon), \quad \frac{1}{n}\sum_{i=1}^n\mathbf{1}_{\{u_i \notin [a, b]\}} = \ocal_{\P,m,n}(\varepsilon).$$

Combining these, we have $\frac{1}{n}\sum_{i=1}^n\mathbf{1}_{\{|\bar{Z}_i-T_i| \ge \varepsilon\}} = \ocal_{\P,m,n}(\varepsilon)$.

Because $\varepsilon$ is arbitrary, we complete the proof.

\end{proof}

\begin{lemma}[Two-dimensional Intermediate Convergence]\label{lem:eCDF-zbar-2d}
Under the conditions of Lemma~\ref{lem:eCDF-zbar-1d} and Lemma~\ref{lem:conditional-cdf-convergence}, 
\[
\sup_{(x_1,x_2)\in\bR^2} \big|\bar{F}_n(x_1,x_2) - F(x_1,x_2)\big| \xrightarrow{\P} 0\quad\text{as }m, n \to \infty,
\]
where 
\[
\bar{F}_n(x_1,x_2) = \frac{1}{n}\sum_{i=1}^n \mathbf{1}_{\{\bar{Z}_{i1} \leq x_1, \bar{Z}_{i2} \leq x_2\}}.
\]
\end{lemma}

\begin{proof}
Define reference variables $T_i = (T_{i1}, T_{i2})$ where
\[
T_{i1} = F_1^{-1}(G(F_1(Z_{i1}) + e_{i1}))
\]
and
\[
T_{i2} = F_{2\mid 1}^{-1}(G(F_{2\mid 1}(Z_{i2}\mid Z_{i1}) + e_{i2})\mid T_{i1})
\]

Let $F_n$ be their empirical CDF:
\[
F_n(x_1,x_2) = \frac{1}{n}\sum_{i=1}^n \mathbf{1}_{\{T_{i1} \leq x_1, T_{i2} \leq x_2\}}.
\]

By construction, $T_i \stackrel{\mathrm{i.i.d.}}{\sim} F$ , so by the multivariate Glivenko-Cantelli theorem:
\[
\sup_{(x_1,x_2)\in\bR^2} \big|F_n(x_1,x_2) - F(x_1,x_2)\big| \xrightarrow{\P} 0.
\]

Thus we only need to show:
\[
\sup_{(x_1,x_2)\in\bR^2} \big|\bar{F}_n(x_1,x_2) - F_n(x_1,x_2)\big| \xrightarrow{\P} 0.
\]

For any $(x_1,x_2)\in\bR^2$:
\begin{align*}
\big|\bar{F}_n(x_1,x_2) - F_n(x_1,x_2)\big| &= \left|\frac{1}{n}\sum_{i=1}^n\left[\mathbf{1}_{\{\bar{Z}_{i1} \leq x_1, \bar{Z}_{i2} \leq x_2\}} - \mathbf{1}_{\{T_{i1} \leq x_1, T_{i2} \leq x_2\}}\right]\right|.
\end{align*}

For any $\varepsilon > 0$, we can bound this by:
\begin{align*}
&\frac{1}{n}\sum_{i:\|\bar{Z}_i-T_i\| < \varepsilon} \big|\mathbf{1}_{\{\bar{Z}_{i1} \leq x_1, \bar{Z}_{i2} \leq x_2\}} - \mathbf{1}_{\{T_{i1} \leq x_1, T_{i2} \leq x_2\}}\big| + \frac{1}{n}\sum_{i:\|\bar{Z}_i-T_i\| \geq \varepsilon} 1.
\end{align*}

For any $(x_1,x_2)\in\bR^2$, when $\|\bar{Z}_i-T_i\| < \varepsilon$, we can bound:
Note that $\big|\mathbf{1}_{\{\bar{Z}_{i1} \leq x_1, \bar{Z}_{i2} \leq x_2\}} - \mathbf{1}_{\{T_{i1} \leq x_1, T_{i2} \leq x_2\}}\big| \le \big|\mathbf{1}_{\{\bar{Z}_{i1} \leq x_1\}} - \mathbf{1}_{\{T_{i1} \leq x_1\}}\big| + \big|\mathbf{1}_{\{\bar{Z}_{i2} \leq x_2\}} - \mathbf{1}_{\{T_{i2} \leq x_2\}}\big|$. Moreover, by moving the sup operator from outside the sum to inside the sum, we have:
\begin{align*}
&\sup_{(x_1,x_2)\in\bR^2} \frac{1}{n}\sum_{i:\|\bar{Z}_i-T_i\| < \varepsilon} \big|\mathbf{1}_{\{\bar{Z}_{i1} \leq x_1, \bar{Z}_{i2} \leq x_2\}} - \mathbf{1}_{\{T_{i1} \leq x_1, T_{i2} \leq x_2\}}\big| \\
& \leq \frac{1}{n}\sum_{i:\|\bar{Z}_i-T_i\| < \varepsilon}\sup_{(x_1,x_2)\in\bR^2}\big|\mathbf{1}_{\{\bar{Z}_{i1} \leq x_1, \bar{Z}_{i2} \leq x_2\}} - \mathbf{1}_{\{T_{i1} \leq x_1, T_{i2} \leq x_2\}}\big| \\
&\leq  \frac{1}{n}\sum_{i:\|\bar{Z}_i-T_i\| < \varepsilon}\sup_{(x_1,x_2)\in\bR^2}\left[ \big|\mathbf{1}_{\{\bar{Z}_{i1} \leq x_1\}} - \mathbf{1}_{\{T_{i1} \leq x_1\}}\big| + \big|\mathbf{1}_{\{\bar{Z}_{i2} \leq x_2\}} - \mathbf{1}_{\{T_{i2} \leq x_2\}}\big|\right]\\
&\leq  \frac{1}{n}\sum_{i:\|\bar{Z}_i-T_i\| < \varepsilon}\left[ \sup_{(x_1,x_2)\in\bR^2}\big|\mathbf{1}_{\{\bar{Z}_{i1} \leq x_1\}} - \mathbf{1}_{\{T_{i1} \leq x_1\}}\big| + \sup_{(x_1,x_2)\in\bR^2}\big|\mathbf{1}_{\{\bar{Z}_{i2} \leq x_2\}} - \mathbf{1}_{\{T_{i2} \leq x_2\}}\big|\right]\\
& =  \frac{1}{n}\sum_{i:\|\bar{Z}_i-T_i\| < \varepsilon}\left[ \sup_{x_1\in\bR}\big|\mathbf{1}_{\{\bar{Z}_{i1} \leq x_1\}} - \mathbf{1}_{\{T_{i1} \leq x_1\}}\big| + \sup_{x_2\in\bR}\big|\mathbf{1}_{\{\bar{Z}_{i2} \leq x_2\}} - \mathbf{1}_{\{T_{i2} \leq x_2\}}\big|\right]\\
&\leq \frac{1}{n}\sum_{i:\|\bar{Z}_i-T_i\| < \varepsilon} \sup_{x_1\in\bR}\big|\mathbf{1}_{\{\bar{Z}_{i1} \leq x_1\}} - \mathbf{1}_{\{T_{i1} \leq x_1\}}\big| + \frac{1}{n}\sum_{i:\|\bar{Z}_i-T_i\| < \varepsilon}\sup_{x_2\in\bR}\big|\mathbf{1}_{\{\bar{Z}_{i2} \leq x_2\}} - \mathbf{1}_{\{T_{i2} \leq x_2\}}\big|\\
&\leq \frac{1}{n}\sum_{i:|\bar{Z}_{i1}-T_{i1}| < \varepsilon} \sup_{x_1\in\bR}\big|\mathbf{1}_{\{\bar{Z}_{i1} \leq x_1\}} - \mathbf{1}_{\{T_{i1} \leq x_1\}}\big| + \frac{1}{n}\sum_{i:|\bar{Z}_{i2}-T_{i2}| < \varepsilon}\sup_{x_2\in\bR}\big|\mathbf{1}_{\{\bar{Z}_{i2} \leq x_2\}} - \mathbf{1}_{\{T_{i2} \leq x_2\}}\big|.
\end{align*}

For each coordinate $j=1,2$, the term $\sup_{x_j\in\bR}\big|\mathbf{1}_{\{\bar{Z}_{ij} \leq x_j\}} - \mathbf{1}_{\{T_{ij} \leq x_j\}}\big|$ equals 1 only if there exists some $x_j$ between $\bar{Z}_{ij}$ and $T_{ij}$. When $\|\bar{Z}_i-T_i\| < \varepsilon$, we have $|\bar{Z}_{ij} - T_{ij}| < \varepsilon$ for both $j=1,2$.

Using our assumptions about the true $F$, with the same type of probability concentration we have used in previous lemmas, we can see the above bound leads to
\begin{align*}
&\sup_{(x_1,x_2)\in\bR^2} \frac{1}{n}\sum_{i:\|\bar{Z}_i-T_i\| < \varepsilon} \big|\mathbf{1}_{\{\bar{Z}_{i1} \leq x_1, \bar{Z}_{i2} \leq x_2\}} - \mathbf{1}_{\{T_{i1} \leq x_1, T_{i2} \leq x_2\}}\big| =\ocal_{\P,n}(\varepsilon).
\end{align*}

For the second term, we need to control $\|\bar{Z}_i-T_i\|$. Note that
\[
\|\bar{Z}_i-T_i\| \geq \varepsilon \implies |\bar{Z}_{i1}-T_{i1}| \geq \varepsilon/\sqrt{2} \text{ or } |\bar{Z}_{i2}-T_{i2}| \geq \varepsilon/\sqrt{2}.
\]

For the first coordinate, we can directly use the events in Lemma~\ref{lem:eCDF-zbar-1d}. Recall
\[
\mathscr{E}_m = \left\{ \sup_{1\le i\le n}\big| \hat{F}_m(Z_{i1}) - F_1(Z_{i1}) \big| < \delta \right\} \cap \left\{ \sup_{\substack{u, v \in [a, b] \\ | u - v | < \delta'}} \big| \hat{F}_m^{-1}(u) - \hat{F}_m^{-1}(v) \big| < \frac{\varepsilon}{2\sqrt{2}} \right\}
\]
and
\[
\mathscr{E}_{m}' = \left\{ \sup_{u \in [a, b]} \big| \hat{F}_m^{-1}(u) - F_1^{-1}(u) \big| < \frac{\varepsilon}{2\sqrt{2}} \right\}.
\]
where $\delta, \delta'$ are chosen as in Lemma~\ref{lem:eCDF-zbar-1d}, and $[a,b]$ contains the $\varepsilon/16$ and $1-\varepsilon/16$ quantiles of $\text{Uniform}(0,1)$.

For the second coordinate, conditioning on $\bar{Z}_{i1}$ and $T_{i1}$, define analogous events. Specifically, let $a,b$ be the $\varepsilon/16$ and $1-\varepsilon/16$ quantiles of $\text{Uniform}(0,1)$. Define
\begin{align*}
\mathscr{E}_{m,2} &= \left\{ \sup_{1\le i\le n}\big| \hat{F}_{2|1,m}(Z_{i2}|Z_{i1}) - F_{2|1}(Z_{i2}|Z_{i1}) \big| < \delta \right\} \\
&\quad\cap \left\{ \sup_{x_1\in\bR} \sup_{\substack{u, v \in [a, b] \\ | u - v | < \delta'}} \big| \hat{F}_{2|1,m}^{-1}(u|x_1) - \hat{F}_{2|1,m}^{-1}(v|x_1) \big| < \frac{\varepsilon}{2\sqrt{2}} \right\}.
\end{align*}
From Lemma~\ref{lem:conditional-cdf-convergence} we know that $\P(\mathscr{E}_{m,2}) \xrightarrow{m\to \infty} 1$. Similarly, define
\[
\mathscr{E}_{m,2}' = \left\{ \sup_{x_1\in\bR} \sup_{u \in [a, b]} \big| \hat{F}_{2\mid 1,m}^{-1}(u\mid x_1) - F_{2\mid 1}^{-1}(u\mid x_1) \big| < \frac{\varepsilon}{2\sqrt{2}} \right\}.
\]

On the intersection of events $\mathscr{E}_m \cap \mathscr{E}_m' \cap \mathscr{E}_{m,2} \cap \mathscr{E}_{m,2}'$, we have for all $i$:
\begin{align*}
\left\{\|\bar{Z}_i-T_i\| \geq \varepsilon\right\} 
&\subseteq \left\{|\bar{Z}_{i1}-T_{i1}| \geq \varepsilon/\sqrt{2} \text{ or } |\bar{Z}_{i2}-T_{i2}| \geq \varepsilon/\sqrt{2}\right\} \\
&\subseteq \left\{|\bar{Z}_{i1}-T_{i1}| \geq \varepsilon/\sqrt{2}\right\} \cup \left\{|\bar{Z}_{i2}-T_{i2}| \geq \varepsilon/\sqrt{2}\right\}.
\end{align*}

For the first coordinate, on event $\mathscr{E}_m \cap \mathscr{E}_m'$, we have shown in Lemma~\ref{lem:eCDF-zbar-1d} that $|\bar{Z}_{i1}-T_{i1}| \geq \varepsilon/\sqrt{2}$ can happen only if $U_{i1} \notin [a,b]$, where $U_{i1} = G(F_1(Z_{i1}) + e_{i1})$.

For the second coordinate, on event $\mathscr{E}_{m,2} \cap \mathscr{E}_{m,2}'$, similarly $|\bar{Z}_{i2}-T_{i2}| \geq \varepsilon/\sqrt{2}$ can happen only if $U_{i2} \notin [a,b]$, where $U_{i2} = G(F_{2\mid 1}(Z_{i2}\mid Z_{i1}) + e_{i2})$.
By Hoeffding's inequality:
\[
\P\left(\frac{1}{n}\sum_{i=1}^n\mathbf{1}_{\{U_{i1} \notin [a,b]\}} > \frac{\varepsilon}{4}\right) \xrightarrow{n\to \infty} 0
\]
\[
\P\left(\frac{1}{n}\sum_{i=1}^n\mathbf{1}_{\{U_{i2} \notin [a,b]\}} > \frac{\varepsilon}{4}\right) \xrightarrow{n\to \infty} 0
\]

Therefore, with additional intersection of the two events from the Hoeffding's inequality,
\begin{align*}
\frac{1}{n}\sum_{i:\|\bar{Z}_i-T_i\| \geq \varepsilon} 1 
&\leq \frac{1}{n}\sum_{i=1}^n \left[\mathbf{1}_{\{U_{i1} \notin [a,b]\}} + \mathbf{1}_{\{U_{i2} \notin [a,b]\}}\right] \\
&\leq \frac{\varepsilon}{4} + \frac{\varepsilon}{4} = \frac{\varepsilon}{2}.
\end{align*}

By union bound, the intersection of all these events has probability approaching 1 as $m,n \to \infty$.
Since $\varepsilon$ is arbitrary, we conclude:
\[
\sup_{(x_1,x_2)\in\bR^2} \big|\bar{F}_n(x_1,x_2) - F_n(x_1,x_2)\big| \xrightarrow{\P} 0.
\]
\end{proof}

\begin{lemma}[Two-dimensional Privatized Convergence]\label{lem:eCDF-ztilde-2d}
Under the conditions of Lemma~\ref{lem:eCDF-zbar-2d}, we have
\[
\sup_{(x_1,x_2)\in\bR^2} \big|\tilde{F}_n(x_1,x_2) - F(x_1,x_2)\big| \xrightarrow{\P} 0\quad\text{as }m, n \to \infty,
\]
where 
\[
\tilde{F}_n(x_1,x_2) = \frac{1}{n}\sum_{i=1}^n \mathbf{1}_{\{\tilde{Z}_{i1} \leq x_1, \tilde{Z}_{i2} \leq x_2\}}.
\]
\end{lemma}

\begin{proof}
With Lemma~\ref{lem:eCDF-zbar-2d}, we only need to show that
\[
\sup_{(x_1,x_2)\in\bR^2} \big|\tilde{F}_n(x_1,x_2) - \bar{F}_n(x_1,x_2)\big| \xrightarrow{\P} 0.
\]

For any $(x_1,x_2)\in\bR^2$:
\begin{align*}
\big|\tilde{F}_n(x_1,x_2) - \bar{F}_n(x_1,x_2)\big| &= \left|\frac{1}{n}\sum_{i=1}^n\left[\mathbf{1}_{\{\tilde{Z}_{i1} \leq x_1, \tilde{Z}_{i2} \leq x_2\}} - \mathbf{1}_{\{\bar{Z}_{i1} \leq x_1, \bar{Z}_{i2} \leq x_2\}}\right]\right|.
\end{align*}

For any $\varepsilon > 0$, we can bound the difference by:
\[
\frac{1}{n}\sum_{i:\|\tilde{Z}_i-\bar{Z}_i\| < \varepsilon} \big|\mathbf{1}_{\{\tilde{Z}_{i1} \leq x_1, \tilde{Z}_{i2} \leq x_2\}} - \mathbf{1}_{\{\bar{Z}_{i1} \leq x_1, \bar{Z}_{i2} \leq x_2\}}\big| + \frac{1}{n}\sum_{i:\|\tilde{Z}_i-\bar{Z}_i\| \geq \varepsilon} 1.
\]

For the first term, when $\|\tilde{Z}_i-\bar{Z}_i\| < \varepsilon$, similar to Lemma~\ref{lem:eCDF-zbar-2d}, we use the following decomposition
\begin{align*}
&\sup_{(x_1,x_2)\in\bR^2} \frac{1}{n}\sum_{i:\|\tilde{Z}_i-\bar{Z}_i\| < \varepsilon} \big|\mathbf{1}_{\{\tilde{Z}_{i1} \leq x_1, \tilde{Z}_{i2} \leq x_2\}} - \mathbf{1}_{\{\bar{Z}_{i1} \leq x_1, \bar{Z}_{i2} \leq x_2\}}\big| \\
&\leq \frac{1}{n}\sum_{i:|\tilde{Z}_{i1}-\bar{Z}_{i1}| < \varepsilon} \sup_{x_1\in\bR}\big|\mathbf{1}_{\{\tilde{Z}_{i1} \leq x_1\}} - \mathbf{1}_{\{\bar{Z}_{i1} \leq x_1\}}\big| + \frac{1}{n}\sum_{i:|\tilde{Z}_{i2}-\bar{Z}_{i2}| < \varepsilon}\sup_{x_2\in\bR}\big|\mathbf{1}_{\{\tilde{Z}_{i2} \leq x_2\}} - \mathbf{1}_{\{\bar{Z}_{i2} \leq x_2\}}\big|.
\end{align*}
Using our assumptions about the true $F$ and because of the convergence indicated by Lemma~\ref{lem:eCDF-zbar-2d}, this term is $\ocal_{\P,n}(\varepsilon)$.

For the second term, note that
\[
\|\tilde{Z}_i-\bar{Z}_i\| \geq \varepsilon \implies |\tilde{Z}_{i1}-\bar{Z}_{i1}| \geq \varepsilon/\sqrt{2} \text{ or } |\tilde{Z}_{i2}-\bar{Z}_{i2}| \geq \varepsilon/\sqrt{2}.
\]

For the first coordinate, from Lemma~\ref{lem:eCDF-ztilde-1d}, we know that for any $\gamma > 0$, there exist events with probability approaching 1 such that
\[
\frac{1}{n}\sum_{i=1}^n \mathbf{1}_{\{|\tilde{Z}_{i1}-\bar{Z}_{i1}| \geq \varepsilon/\sqrt{2}\}} \leq \gamma.
\]
For the second coordinate, recall the construction:
\begin{align*}
\tilde{Z}_{i2} &= \hat{F}_{2\mid 1,m}^{-1}(G(\hat{F}_{2\mid 1,m}(\hat{Z}_{i2}\mid \hat{Z}_{i1}) + e_{i2})\mid \tilde{Z}_{i1}), \\
\bar{Z}_{i2} &= \hat{F}_{2\mid 1,m}^{-1}(G(\hat{F}_{2\mid 1,m}(Z_{i2}\mid Z_{i1}) + e_{i2})\mid \bar{Z}_{i1}).
\end{align*}

Define
\[
u_i = G(\hat{F}_{2\mid 1,m}(\hat{Z}_{i2}\mid \hat{Z}_{i1}) + e_{i2}), \quad v_i = G(\hat{F}_{2\mid 1,m}(Z_{i2}\mid Z_{i1}) + e_{i2}).
\]

For controlling $|\tilde{Z}_{i2}-\bar{Z}_{i2}|$, by triangle inequality,
\begin{align*}
|\tilde{Z}_{i2}-\bar{Z}_{i2}| &= \big|\hat{F}_{2\mid 1,m}^{-1}(u_i\mid \tilde{Z}_{i1}) - \hat{F}_{2\mid 1,m}^{-1}(v_i\mid \bar{Z}_{i1})\big| \\
&\leq \big|\hat{F}_{2\mid 1,m}^{-1}(u_i\mid \tilde{Z}_{i1}) - \hat{F}_{2\mid 1,m}^{-1}(u_i\mid \bar{Z}_{i1})\big| + \big|\hat{F}_{2\mid 1,m}^{-1}(u_i\mid \bar{Z}_{i1}) - \hat{F}_{2\mid 1,m}^{-1}(v_i\mid \bar{Z}_{i1})\big|.
\end{align*}

For any $\gamma$, since we have already seen the convergence in the first dimension as well as Lemma~\ref{theorem:Si_convergence}) and Lemma~\ref{theorem:main-1d}, we know there exists a compact set $K$ such that:
\[
\P\left(\frac{1}{n}\sum_{i=1}^n \mathbf{1}_{\{\tilde{Z}_{i1} \notin K\}} > \frac{\gamma}{4}\right) \xrightarrow{n\to \infty} 0, \quad \P\left(\frac{1}{n}\sum_{i=1}^n \mathbf{1}_{\{\bar{Z}_{i1} \notin K\}} > \frac{\gamma}{4}\right) \xrightarrow{n\to \infty} 0.
\]

For $\tilde{Z}_{i1}, \bar{Z}_{i1} \in K$, by Theorem~\ref{thm:inverse-conditional-convergence}, for our fixed $\varepsilon$, there exists $\delta > 0$ such that with probability approaching 1:
\[
\sup_{\substack{x_1,y_1 \in K \\ |x_1-y_1| < \delta}} \sup_{u \in [a,b]} \big|\hat{F}_{2\mid 1,m}^{-1}(u\mid x_1) - \hat{F}_{2\mid 1,m}^{-1}(u\mid y_1)\big| < \frac{\varepsilon}{2\sqrt{2}}.
\]

For the difference in probability levels $(u_i,v_i)$, by the same theorem on compact set $K$:
\[
\sup_{x_1 \in K} \sup_{\substack{u,v \in [a,b]\\ |u-v| < \delta_1}} \big|\hat{F}_{2\mid 1,m}^{-1}(u\mid x_1) - \hat{F}_{2\mid 1,m}^{-1}(v\mid x_1)\big| < \frac{\varepsilon}{2\sqrt{2}}.
\]

Therefore, when $\tilde{Z}_{i1}, \bar{Z}_{i1} \in K$:
\[
\left\{|\tilde{Z}_{i2}-\bar{Z}_{i2}| \geq \varepsilon/\sqrt{2}\right\} \subseteq \left\{|\tilde{Z}_{i1} - \bar{Z}_{i1}| \geq \delta\right\} \cup \left\{u_i \notin [a,b]\right\} \cup \left\{v_i \notin [a,b]\right\} \cup \left\{|u_i - v_i| \geq \delta_1\right\}.
\]

Combining all terms and using Hoeffding's inequality, we have
\begin{align*}
\frac{1}{n}\sum_{i=1}^n \mathbf{1}_{\{|\tilde{Z}_{i2}-\bar{Z}_{i2}| \geq \varepsilon/\sqrt{2}\}} &\leq \frac{1}{n}\sum_{i=1}^n \mathbf{1}_{\{\tilde{Z}_{i1} \notin K\}} + \frac{1}{n}\sum_{i=1}^n \mathbf{1}_{\{\bar{Z}_{i1} \notin K\}} + \frac{1}{n}\sum_{i=1}^n \mathbf{1}_{\{|\tilde{Z}_{i1} - \bar{Z}_{i1}| \geq \delta\}} \\
&\quad + \frac{1}{n}\sum_{i=1}^n \mathbf{1}_{\{u_i \notin [a,b]\}} + \frac{1}{n}\sum_{i=1}^n \mathbf{1}_{\{v_i \notin [a,b]\}} + \frac{1}{n}\sum_{i=1}^n \mathbf{1}_{\{|u_i - v_i| \geq \delta_1\}} \\
&\leq \gamma.
\end{align*}
with probability approaching 1 as $m,n \to \infty$.

Combining the above results and picking $\gamma = \ocal(\varepsilon)$, we have
\begin{align*}
\frac{1}{n}\sum_{i:\|\tilde{Z}_i-\bar{Z}_i\| \geq \varepsilon} 1 
&\leq \frac{1}{n}\sum_{i=1}^n \left[\mathbf{1}_{\{|\tilde{Z}_{i1}-\bar{Z}_{i1}| \geq \varepsilon/\sqrt{2}\}} + \mathbf{1}_{\{|\tilde{Z}_{i2}-\bar{Z}_{i2}| \geq \varepsilon/\sqrt{2}\}}\right] = \ocal_{\P,m,n}(\varepsilon).
\end{align*}

Because $\varepsilon$ is arbitrary, we complete the proof.
\end{proof}

Combining Lemmas~\ref{lem:eCDF-zbar-2d} and \ref{lem:eCDF-ztilde-2d} gives the claimed result of Theorem~\ref{thm:main-CDF} in the two dimensional case.  Note that the arguments of Lemmas~\ref{lem:eCDF-zbar-2d} and \ref{lem:eCDF-ztilde-2d} can carried over to any fixed number of dimension $d$, which gives the final theorem.

\newpage
\section{Proof of Theorem~\ref{thm:model-consistency}}

\subsection{Consistency under inner product latent space models}

In this section, we give the proof of the needed concentration bound of $\hat{Z}_i$'s. We will prove a high probability error bound for each individual $\hat{Z}_i$, $i = 1, \ldots, n$. Recall that $\hat{Z}_i$'s are identifiable only up to an orthogonal transformation. This indeterminacy does not affect the results of our study, so we ignore it for notational simplicity.

The estimation of $\hat{Z}_i$, $i = 1, \ldots, n$ is done by solving the problem
$$\mathop{\arg\min}_{(X_i, \alpha_i)\in\bR^d}\sum_{j=n+1}^{n+m} -\{A_{ij}\log(\sigma(X_i^\top \hat{X}_j+\alpha_i+\hat{\alpha}_j)) + (1-A_{ij})\log(1-\sigma(X_i^\top \hat{X}_j+\alpha_i+\hat{\alpha}_j))\}.$$

As discussed, this problem is indeed a logistic regression estimation problem with measurement errors and one offset variable. Therefore, the proof is based on the theory we can get from studying logistic regression.

We first introduce the result of logistic regression as the tool. For clarity, below we first set up a logistic regression problem for our discussion.
We consider i.i.d.\ data $\{(x_i,y_i)\}_{i=1}^n$ where $x_i \in \mathbb{R}^d$ and $y_i \in \{0,1\}$, generated according to the logistic model:
\[
  \P(y_i = 1 \mid x_i) = \sigma(\beta^{*\top} x_i), \quad \text{where } \sigma(z) = \frac{1}{1+e^{-z}}.
\]
Here, $\beta^* \in \mathbb{R}^d$ is the true parameter vector. It uniquely minimizes the population (expected) negative log-likelihood:
\[
  \mathcal{L}(\beta) = \E_{(x,y)}[\ell(\beta; x, y)],
\]
where $\ell(\beta; x, y)$ is the negative log-likelihood for a single data point (using the identity $\log(\sigma(z)) = z - \log(1+e^{z})$ and $\log(1-\sigma(z)) = - \log(1+e^{z})$):
\begin{align} \label{eq:logistic_neg_lik}
  \ell(\beta; x, y) &= -[y\log(\sigma(\beta^\top x)) + (1-y)\log(1-\sigma(\beta^\top x))] \notag\\
  &= -[y(\beta^\top x) - \log(1+e^{\beta^\top x})] \notag\\
  &= \log(1+e^{\beta^\top x}) - y(\beta^\top x).
\end{align}

For the proposed method, we observe perturbed predictors $\{\tilde{x}_i\}_{i=1}^n$ satisfying $\norm{\tilde{x}_i - x_i} \le \delta$.
Following equation \eqref{eq:logistic_neg_lik}, we define two empirical average negative log-likelihood functions based on the sample $\{ (x_i, \tilde{x}_i, y_i) \}_{i=1}^n$:
\begin{itemize}
  \item The loss based on the true (unobserved) predictors $x_i$:
  \[ \mathcal{L}(\beta) = \frac{1}{n}\sum_{i=1}^n \ell(\beta;x_i,y_i) = \frac{1}{n}\sum_{i=1}^n [\log(1+e^{\beta^\top x_i}) - y_i(\beta^\top x_i)]. \]
  \item The loss based on the observed perturbed predictors $\tilde{x}_i$:
  \[ \tilde{\mathcal{L}}(\beta) = \frac{1}{n}\sum_{i=1}^n \ell(\beta;\tilde{x}_i,y_i) = \frac{1}{n}\sum_{i=1}^n [\log(1+e^{\beta^\top \tilde{x}_i}) - y_i(\beta^\top \tilde{x}_i)]. \]
\end{itemize}
The corresponding estimators (Maximum Likelihood Estimates based on these empirical losses) are:
\begin{itemize}
  \item The ideal estimator (uncomputable in practice as $x_i$ are unknown):
  \[ \hat{\beta}_{\rm true} = \mathop{\arg\min}_{\beta\in\mathbb{R}^d} \mathcal{L}(\beta). \]
  \item The actual estimator computed from observed data $(y_i, \tilde{x}_i)$:
  \[ \hat{\beta} = \mathop{\arg\min}_{\beta\in\mathbb{R}^d} \tilde{\mathcal L}(\beta). \]
\end{itemize}
Our goal is to bound the error $\|\hat{\beta} - \beta^*\|$ of the actual estimator relative to the true parameter $\beta^*$. Our theoretical analysis follows the strategy of \cite{wang2024perturbation}, with necessary modifications.

First, we need the following regularity conditions.

\begin{enumerate}
  \item[(C1)] \label{cond:C1}$x_i$ are i.i.d. sub-exponential vectors with bounded Orlicz norm and $\E\norm{x_i}^4 < \infty$. 
  \item[(C2)] \label{cond:C2}$\|\beta^*\| \le R_\beta(n) = C_{\beta}\log m$ (constant $R_\beta > 0$) for another (potentially large integer $m$).
  \item[(C3)] \label{cond:C3}$\mathcal{L}(\beta)$ minimized uniquely at $\beta^*$. Hessian $H^* := \nabla^2 \mathcal{L}(\beta^*) \succeq \mu I_d$ ($\mu > 0$). Exists fixed $r_1 > 0$ s.t. $\nabla^2 \mathcal{L}(\beta) \succeq (\mu/2) I_d$ for $\beta \in B_{r_1}(\beta^*)$.
  \item[(C4)] \label{cond:C4}$\|\tilde{x}_i - x_i\| \le \delta$.
  \item[(C5)] \label{cond:C5}$1/\sqrt{n} \ll \delta \ll 1$; $\delta = o(1/(\log n\log m))$.
\end{enumerate}

\begin{lemma}\label{lem:bound-assist-1}
Under the previous setup and Assumptions (C1)--(C5), for any desired polynomial decay rate $k>2$, there exists a sufficiently large constant $C_1$ and the corresponding $c_1$, such that the event $\mathcal{E}_1$: $\|\nabla \mathcal{L}(\beta^*)\| \le C_1 \sqrt{(\log n)/n}$ holds with probability at least 
$1 - c_1n^{-(k+1)}$
for sufficiently large $n$.
\end{lemma}
\begin{proof}
Let $g = \nabla \mathcal L(\beta^*) \in \mathbb{R}^d$. The $j$-th component is $g_j = \frac{1}{n}\sum_{i=1}^n Z_{ij}$, where $Z_{ij} = \varepsilon_i x_{ij} = (\sigma(\beta^{*\top} x_i) - y_i)x_{ij}$.
Since $x_i$ is sub-exponential and $\varepsilon_i$ is bounded ($|\varepsilon_i|\le 1$), each $Z_{ij}$ is a mean-zero sub-exponential random variable. Let $K = \sup_{i,j} \|Z_{ij}\|_{\psi_1}$ be the sub-exponential parameter (Orlicz $\psi_1$-norm), which is $\ocal(1)$. Let $\sigma^2 = \sup_{i,j} \E[Z_{ij}^2] \le \sup_{i,j} \frac{1}{4}\E[x_{ij}^2] = \ocal(1)$.

We apply the scalar Bernstein inequality to each coordinate sum $g_j$. There exists a universal constant $C_B$ such that for $t>0$:
\[ \P( |g_j| \ge t ) \le 2 \exp\left( -C_B n \min\left\{\frac{t^2}{\sigma^2}, \frac{t}{K}\right\} \right) \le 2 \exp( -c' n \min\{t^2, t\} ) \]
for some constant $c'$ depending on $\sigma^2$ and $K$.
We want a bound $t_n$ such that $|g_j| \le t_n$ holds for all $j=1,\dots,d$ simultaneously with probability $p_n' \le \ocal(n^{-(k+1)})$ for the target $k>2$. Using a union bound:
\[ \P\left( \max_{1\le j\le d} |g_j| \ge t_n \right) \le \sum_{j=1}^d \P( |g_j| \ge t_n ) \le d \cdot 2 \exp( -c' n \min\{t_n^2, t_n\} ). \]
Let $C_1' = \sqrt{(k+1)/c'}$. Choosing $t_n = C_1' \sqrt{(\log n)/n}$ ensures the failure probability for a single coordinate is $\ocal(n^{-(k+1)})$. This deviation $t_n \to 0$, justifying the use of $\min\{t_n^2, t_n\} = t_n^2$.
By the union bound, the event $\mathcal{E}'_1 := \left\{\max_{j\in[d]} |g_j| \le t_n\right\}$ holds with probability $\P(\mathcal{E}'_1) \ge 1 - \ocal(n^{-(k+1)})$.

Now we bound the L2 norm on this event $\mathcal{E}'_1$:
\[ \|\nabla \mathcal{L}(\beta^*)\|^2 = \sum_{j=1}^d g_j^2 \le \sum_{j=1}^d t_n^2 = d \cdot t_n^2 = d (C_1')^2 \frac{\log n}{n}. \]
Taking the square root:
\[ \|\nabla \mathcal L(\beta^*)\| \le \sqrt{d} C_1' \sqrt{\frac{\log n}{n}}. \]
Let $C_1 = \sqrt{d} C_1'$. Thus, event $\mathcal{E}_1$ holds with probability $\P(\mathcal{E}_1) \ge 1 - \ocal(n^{-(k+1)})$. 

\end{proof}

\begin{lemma}\label{lem:bound-assist-2}
Under Assumptions (C1)--(C5), define event $\mathcal{E}_2$: $\nabla^2 \mathcal L(\beta) \succeq (\mu/4) I_d$ uniformly for $\beta \in B_{r_1}(\beta^*)$. For some constant $c_3>0$, we have
$$\P(\mathcal{E}_2) \ge 1-e^{-c_3n}$$
for sufficiently large $n$.
\end{lemma}
\begin{proof}
 Fix any \(\beta\) with \(\|\beta-\beta^*\|\le r_1\). Write
\[
X_i(\beta)
=
\nabla^2\ell(\beta;x_i,y_i)
-
\E[\nabla^2\ell(\beta;x_i,y_i)].
\]
Then \(\{X_i(\beta)\}_{i=1}^n\) are independent, mean‐zero, symmetric \(d\times d\) matrices. Moreover
\[
\nabla^2\ell(\beta;x,y)
=\sigma'(x^\top\beta)xx^\top,
\quad
\sigma'(z)\le\frac14,
\]
so
\[
\|X_i(\beta)\|_{\rm op}
\le
\|\nabla^2\ell(\beta;x_i,y_i)\|_{\rm op}
+\|\E[\nabla^2\ell(\beta;x_i,y_i)]\|_{\rm op}
=\ocal(\|x_i\|^2).
\]
Under the sub‐exponential assumption on each coordinate of \(x_i\), one shows
\(\|X_i(\beta)\|_{\psi_1}\le K\) for some constant \(K\). Likewise, defining
\[
\sigma^2(\beta)
=\left\|\sum_{i=1}^n \E[X_i(\beta)^2]\right\|_{\rm op} =\ocal(n).
\]
Hence by the matrix‐Bernstein inequality (e.g. Vershynin 2018, Thm 6.2.1) there exist constants \(c_1,c_2>0\) such that for any \(t>0\),
\[
\P\left(\left\|\frac1n\sum_{i=1}^n X_i(\beta)\right\|_{\rm op}>t\right)
\le
2d\exp\left(-c_1n\min\Bigl\{\frac{t^2}{K^2},\frac{t}{K}\Bigr\}\right).
\]
In particular, taking \(t=\mu/4\) gives
\[
\P\left(\|\nabla^2\mathcal L(\beta)-\E[\nabla^2\ell(\beta)]\|_{\rm op}>\frac\mu4\right)
\le
2de^{-c_2n}.
\]

Now, we proceed to extend the result from a fixed \(\beta\) to uniform over the ball using the $\varepsilon$-net construction. Let 
\[
B=\left\{\beta\in\mathbb R^d:\|\beta-\beta^*\|\le r_1\right\}.
\]
For $\varepsilon>0$ (to be specified), choose an \(\varepsilon\)-net \(\mathcal N\subset B\) of size
\[
|\mathcal N|\le\Bigl(\frac{3r_1}{\varepsilon}\Bigr)^d.
\]

\medskip
 For each \(\beta_0\in\mathcal N\), we already have
\[
\P\left(\|\nabla^2\mathcal L(\beta_0)-\E[\nabla^2\ell(\beta_0)]\|_{\rm op}>\frac{\mu}{4}\right)
\le2de^{-c_2n}.
\]
A union bound over \(\mathcal N\) yields
\[
\P\left(\exists\beta_0\in\mathcal N: \|\nabla^2\mathcal L(\beta_0)-\E[\nabla^2\ell(\beta_0)]\|_{\rm op}>\frac\mu4\right)
\le|\mathcal N|\cdot 2de^{-c_2n}
=2d\Bigl(\frac{3r_1}{\varepsilon}\Bigr)^de^{-c_2n}.
\]

\medskip

We have
\[
\nabla^2\ell(\beta;x,y)
=\sigma'(x^\top\beta)xx^\top,
\]
so for any two parameters \(\beta,\beta'\),
\[
\nabla^2\ell(\beta;x,y)
-\nabla^2\ell(\beta';x,y)
=[\sigma'(x^\top\beta)-\sigma'(x^\top\beta')]xx^\top.
\]
By the mean‐value theorem, there exists some scalar \(\xi\) on the line segment between \(x^\top\beta\) and \(x^\top\beta'\), such that
\[
\nabla^2\ell(\beta;x,y)
-\nabla^2\ell(\beta';x,y)
=\sigma''(\xi)[x^\top(\beta-\beta')]xx^\top.
\]
Taking operator norms and using \(\|xx^\top\|_{\rm op}=\|x\|^2\), we get
\[
\|\nabla^2\ell(\beta;x,y)
-\nabla^2\ell(\beta';x,y)\|_{\rm op}
\le
|\sigma''(\xi)||x^\top(\beta-\beta')|\|x\|^2
\le
\sup_{u}\bigl|\sigma''(u)\bigr|\|x\|\|\beta-\beta'\|\|x\|^2.
\]
Since on \(\|\beta-\beta^*\|\le r_1\) one has \(\xi\) ranging over a compact interval, \(\sup_u|\sigma''(u)|\) is finite. Therefore we can set
\[
L=\Bigl(\sup_{u}|\sigma''(u)|\Bigr)\|x\|^3
=\ocal(\|x\|^3),
\]
and conclude
\[
\|\nabla^2\ell(\beta;x,y)-\nabla^2\ell(\beta';x,y)\|_{\rm op}
\le L\|\beta-\beta'\|.
\]

With probability at least $1-e^{-c_3'n}$, we have \(\max_{i\in [n]}\|x_i\|\le C\log n\). We can take
\[
L = C'(\log n)^3
\]
for some constant \(C'\). Set
\[
\varepsilon = \min\left\{\frac{\mu}{16L},r_1\right\}.
\]
Then for any \(\beta\in B\), there exists \(\beta_0\in\mathcal N\) with \(\|\beta-\beta_0\|\le\varepsilon\), and
\[
\|\nabla^2\mathcal L(\beta)-\nabla^2\mathcal L(\beta_0)\|_{\rm op}
\le L\|\beta-\beta_0\|\le\frac\mu{16}.
\]
Also note that with this choice of $\varepsilon$, the probability in (i) holds with \(\ocal(e^{-c_3n})\).

On the intersection of the high‐probability events from (i) and the bound \(\max_{i\in[n]}\|x_i\|\le C\log n\), we have for every \(\beta\in B\):
\begin{align*}
\|\nabla^2\mathcal L(\beta)-\E[\nabla^2\ell(\beta)]\|_{\rm op}
&\le\|\nabla^2\mathcal L(\beta_0)-\E[\nabla^2\ell(\beta_0)]\|_{\rm op}
   +\|\nabla^2\mathcal L(\beta)-\nabla^2\mathcal L(\beta_0)\|_{\rm op}\\
&\le \frac\mu4 + \frac\mu{16}
<\frac\mu2.
\end{align*}

Since \(\E[\nabla^2\ell(\beta)]=\nabla^2\mathcal L(\beta)\succeq\mu I_d\), it follows that
\(\nabla^2L(\beta)\succeq\mu I_d - (\mu/2) I_d=(\mu/2) I_d\succ(\mu/4) I_d\)
uniformly over \(B\). Hence
\[
\P(\mathcal E_2^c)=\ocal(e^{-c_3n}).
\]
 
\end{proof}

\begin{lemma}
\label{thm:logistic-subexponential}
Under Assumptions (C1)--(C5), for any desired polynomial decay rate $k>2$, there exist constants $C, C_p > 0$ (depending on $\sigma_x, R_\beta, \mu, d, k$), such that for sufficiently large $n$:
\[
  \P\left( \|\hat{\beta} - \beta^*\| \le C \max\left\{\delta \log n\log m, \frac{\log n}{\sqrt{n}}\right\} \right) \ge 1 - C_p n^{-(k+1)}.
\]
\end{lemma}

\begin{proof}
Let $k>2$ be the target polynomial decay exponent for the probability. Define high-probability events:
\begin{itemize}
  \item $\mathcal{E}_0$: Event where $\|x_i\| \le C_R \log n =: R'(n)$ for all $i=1,\ldots,n$. $\P(\mathcal{E}_0) \ge 1-n^{-(k+1)}$.
  \item $\mathcal{E}_1$: Event where $\|\nabla \mathcal L(\beta^*)\| \le C_1 \sqrt{(\log n)/n}$. From Lemma~\ref{lem:bound-assist-1}, $\P(\mathcal{E}_1) \ge 1-c_1n^{-(k+1)}$.  
  \item $\mathcal{E}_2$: $\nabla^2 \mathcal L(\beta) \succeq (\mu/4) I_d$ uniformly for $\beta \in B_{r_1}(\beta^*)$. From Lemma~\ref{lem:bound-assist-2}, $\P(\mathcal{E}_2) \ge 1-e^{-c_3n}$.
  \item $\mathcal{E}_3$: $\nabla^2 \tilde{\mathcal L}(\beta) \succeq (\mu/8) I_d =: \tilde{\mu} I_d$ uniformly for $\beta \in B_{r_1}(\beta^*)$. 
\end{itemize}

Independent of any randomness, a simple perturbation check shows
\[
\|\nabla^2\tilde{\mathcal{L}}(\beta)-\nabla^2 \mathcal{L}(\beta)\|_{\rm op}
\le K\delta,
\]
where $K$ depends only on the $\|x_i\|$‐bounds (and thus is $\ocal((\log n)^2)$ on $\mathcal E_0$). Hence
\[
\nabla^2\tilde{\mathcal L}(\beta)
\succeq
\nabla^2\mathcal L(\beta)-K\delta I_d
\succeq
\frac{\mu}{4}I_d -K\delta I_d.
\]
By Assumption (C5) we have $K\delta\le\mu/8$ for large $n$, so
\[
\nabla^2\tilde{\mathcal L}(\beta)\succeq\frac\mu8I_d,
\quad \forall\|\beta-\beta^*\|\le r_1,
\]
which is exactly the event $\mathcal E_3$. Thus 
\[
\mathcal E_2\subset\mathcal E_3
\implies
\P(\mathcal E_3)\ge\P(\mathcal E_2).
\]

Let $\mathcal{E} = \mathcal{E}_0 \cap \mathcal{E}_1 \cap \mathcal{E}_2 \cap \mathcal{E}_3$, and we have $\P(\mathcal{E}) \ge 1 - c_5 n^{-(k+1)}$. 

Next, we work under the conditions from $\mathcal{E}$.
We want to use Lemma~\ref{lem:yin} to show the unique solution. Set $\phi(\beta) = \nabla \tilde{\mathcal L}(\beta)$, $x^* = \beta^*$, $y^* = \nabla \tilde{\mathcal L}(\beta^*)$, $y = 0$, $\delta_1 = r_1$.

\medskip
\noindent
We first verify the conditions of Lemma~\ref{lem:yin}.

\begin{itemize}
  \item \textbf{Lower bound:} For $\norm{\beta - \beta^*} = r_1$. Let $H_{\mathrm{avg}}(\beta) = \int_0^1 \nabla^2 \tilde{\mathcal L}(\beta^* + t(\beta-\beta^*))\,\mathrm dt$.
  The segment $[\beta^*, \beta] \subset B_{r_1}(\beta^*)$. On $\mathcal{E}_3$, $\nabla^2 \tilde{\mathcal L}(\cdot) \succeq \tilde{\mu} I_d$ in $B_{r_1}(\beta^*)$, thus $H_{\mathrm{avg}}(\beta) \succeq \tilde{\mu} I_d$.
  Then, $\norm{\phi(\beta) - y^*} = \norm{H_{\mathrm{avg}}(\beta) (\beta - \beta^*)} \ge \lambdaMin(H_{\mathrm{avg}}(\beta)) \norm{\beta - \beta^*} \ge \tilde{\mu} r_1$.
  Set $\rho = \tilde{\mu} r_1$. The condition $\min_{\norm{\beta - \beta^*} = r_1} \norm{\phi(\beta) - y^*} \ge \rho$ holds on $\mathcal{E}_3$.

  \item \textbf{Upper bound:} We need to show $\norm{y - y^*} = \norm{0 - \nabla \tilde{\mathcal L}(\beta^*)} = \norm{\nabla \tilde{\mathcal L}(\beta^*)} \le \rho = \tilde{\mu} r_1$.
  We decompose the gradient using the triangle inequality:
  \begin{align*}
    \norm{\nabla \tilde{\mathcal L}(\beta^*)} &= \norm{\nabla \mathcal L(\beta^*) + \nabla(\tilde{\mathcal L}-\mathcal L)(\beta^*)} \\
    &\le \norm{\nabla \mathcal L(\beta^*)} + \norm{\nabla(\tilde{\mathcal L}-\mathcal L)(\beta^*)}.
  \end{align*}
  On event $\mathcal{E}_1$, the first term is bounded:
  \[ \norm{\nabla \mathcal L(\beta^*)} \le C_1 \sqrt{\frac{\log n}{n}}. \]
  For the second term, we analyze $\nabla(\tilde{\mathcal L}-\mathcal L)(\beta^*) = \frac{1}{n} \sum_{i=1}^n T_i(\beta^*)$, where
  \[
    T_i(\beta^*) = -y_i(\tilde{x}_i - x_i) + \sigma(\beta^{*\top} \tilde{x}_i) (\tilde{x}_i - x_i) + [\sigma(\beta^{*\top} \tilde{x}_i) - \sigma(\beta^{*\top} x_i)] x_i.
  \]
  We bound the norm of $T_i(\beta^*)$ using the triangle inequality:
  \[
    \|T_i(\beta^*)\| \le |y_i| \norm{\tilde{x}_i - x_i} + |\sigma(\beta^{*\top} \tilde{x}_i)| \norm{\tilde{x}_i - x_i} + |\sigma(\beta^{*\top} \tilde{x}_i) - \sigma(\beta^{*\top} x_i)| \norm{x_i}.
  \]
  Using bounds $|y_i| \le 1$, $|\sigma(\cdot)| \le 1$, $\norm{\tilde{x}_i - x_i} \le \delta$ (which is Assumption C4), and the $\frac{1}{4}$-Lipschitz property of $\sigma$:
  \[
    \|T_i(\beta^*)\| \le 1 \cdot \delta + 1 \cdot \delta + \frac{1}{4} |\beta^{*\top}(\tilde{x}_i - x_i)| \norm{x_i}.
  \]
  Applying Cauchy-Schwarz inequality to $|\beta^{*\top}(\tilde{x}_i - x_i)| \le \norm{\beta^*} \norm{\tilde{x}_i - x_i} \le \norm{\beta^*} \delta$, under $\mathcal{E}_0$ and Assumption C2, we have
  \begin{align*}
    \|T_i(\beta^*)\| &\le 2\delta + \frac{1}{4} \norm{\beta^*} \delta \norm{x_i} \\
    &= \left( 2 + \frac{1}{4} \norm{\beta^*} \norm{x_i} \right)\delta \\ &\le \left( 2 + \frac{1}{4} R_\beta(n) R'(n) \right) \delta \\
    &\le \left( 2 + \frac{1}{4} (C_{\beta} \log m) (C_R \log n) \right) \delta \\
    &= \left( 2 + \frac{C_{\beta} C_R}{4} \log n\log m \right) \delta.
  \end{align*}
  Therefore, this leads to
  \begin{align*}
    \norm{\nabla(\tilde{\mathcal L}-\mathcal L)(\beta^*)} &= \left\| \frac{1}{n} \sum_{i=1}^n T_i(\beta^*) \right\|\le \frac{1}{n} \sum_{i=1}^n \|T_i(\beta^*)\| = \left( 2 + \frac{C_{\beta} C_R}{4} \log n\log m \right) \delta.
  \end{align*}
  Let $C_{G}^*(n) = 2 + \frac{C_{\beta} C_R}{4} (\log n)(\log m) = \ocal(\log n\log m)$. We have:
  \[ \norm{\nabla(\tilde{\mathcal L}-\mathcal L)(\beta^*)}_2 \le C_{G}^*(n) \delta. \]
  Combining with the bound for $\norm{\nabla \mathcal L(\beta^*)}$ on event $\mathcal{E}_1$:
  \[
    \rho_n := \norm{\nabla \tilde{\mathcal L}(\beta^*)} \le \frac{C_1 \log n}{\sqrt{n}} + C_{G}^*(n) \delta \quad (\text{holds on } \mathcal{E}_0 \cap \mathcal{E}_1).
  \]
  The condition required for Lemma~\ref{lem:bound-assist-2} is $\rho_n \le \rho = \tilde{\mu} r_1$. Substituting the orders:
  \[ \ocal\left(\frac{\log n}{\sqrt{n}}\right) + \ocal(\delta \log n\log m) \le \frac{\mu r_1}{8}. \]
  Since $\delta = o(1/(\log n\log m))$ by Assumption C5, the inequality holds for sufficiently large $n$.

  \item \textbf{Injectivity:} On $\mathcal{E}_3$, Jacobian $J_\phi(\beta) = \nabla^2 \tilde{\mathcal L}(\beta) \succeq \tilde{\mu} I_d \succ 0$ in $B_{r_1}(\beta^*)$, implying $\phi$ is injective on $B_{r_1}(\beta^*)$.
\end{itemize}
All conditions hold on $\mathcal{E}$ for sufficiently large $n$. Therefore, we can apply Lemma~\ref{lem:yin} for Localization. On $\mathcal{E}$, Lemma~\ref{lem:yin} guarantees the existence of $\hat{\beta}_{\mathrm{soln}} \in B_{r_1}(\beta^*)$, such that $\nabla\tilde{\mathcal L}(\hat{\beta}_{\mathrm{soln}}) = 0$.
Since $\tilde{\mathcal L}(\beta)$ is strictly convex on $B_{r_1}(\beta^*)$ (as $\nabla^2\tilde{\mathcal L} \succeq \tilde{\mu}I_d > 0$ on $\mathcal{E}_3$), $\hat{\beta}_{\mathrm{soln}}$ is the unique minimizer in $B_{r_1}(\beta^*)$. Global convexity implies that $\hat{\beta}_{\mathrm{soln}}$ is the unique global minimum $\hat{\beta}$.
Therefore, on $\mathcal{E}$, we have that $\hat{\beta}$ exists, is unique, and satisfies
\[ \|\hat{\beta} - \beta^*\| \le r_1. \]

\medskip
\noindent

Next, we derive a bound on $\|\hat{\beta}_{\rm true} - \beta^*\|$. We will first show that $\hat{\beta}_{\rm true}$ must also be in $B_{r_1}(\beta^*)$, under $\mathcal{E}$. Recall that that
\[
\nabla^2\mathcal L(\beta)\succeq\frac\mu4I_d,
\quad \forall\|\beta-\beta^*\|\le r_1,
\quad\text{and}\quad
\|\nabla \mathcal L(\beta^*)\|\le C_1\sqrt{\frac{\log n}{n}}.
\]

On the event \(\mathcal E_1\cap\mathcal E_2\), for sufficiently large $n$, we have
\[
\nabla^2\mathcal L(\beta)\succeq\frac\mu4I_d,
\quad\forall\|\beta-\beta^*\|\le r_1,
\quad\text{and}\quad
\|\nabla \mathcal L(\beta^*)\|\le\frac\mu8r_1.
\]

Let \(\beta\) satisfy \(\|\beta-\beta^*\|\ge r_1\), and define
\[
\beta_0 =\beta^* 
 +\frac{r_1}{\|\beta-\beta^*\|}(\beta-\beta^*),
\]
so that \(\|\beta_0-\beta^*\|=r_1\). By Taylor’s theorem along the segment from \(\beta^*\) to \(\beta_0\),
\[
\mathcal L(\beta_0)
= \mathcal L(\beta^*)
 + \nabla \mathcal L(\beta^*)^\top(\beta_0-\beta^*)
 + \frac12(\beta_0-\beta^*)^\top
  \left[\int_0^1\nabla^2\mathcal L(\beta^*+s(\beta_0-\beta^*))\,\mathrm ds\right]
  (\beta_0-\beta^*).
\]
On the event \(\mathcal E_2\), for every point \(\xi\) on that segment we have
\(\nabla^2\mathcal L(\xi)\succeq(\mu/4)I_d\), so the averaged Hessian
\( H_{\rm avg}
=\int_0^1\nabla^2\mathcal L(\beta^*+s(\beta_0-\beta^*))\,\mathrm ds\)
also satisfies \(H_{\rm avg}\succeq(\mu/4)I_d\). Hence
\[
\mathcal L(\beta_0)-\mathcal L(\beta^*)
\ge
-\|\nabla \mathcal L(\beta^*)\|r_1
+\frac12\cdot\frac\mu4r_1^2
=\frac\mu8r_1^2
-\|\nabla \mathcal L(\beta^*)\|r_1.
\]
Since on \(\mathcal E_1\) we have \(\|\nabla \mathcal L(\beta^*)\|\le\mu r_1/8\), it follows that
\[
\mathcal L(\beta_0)\ge\mathcal L(\beta^*).
\]

Define the “outward” vector for \(\beta\) (recall that \(\|\beta-\beta^*\|\ge r_1\))
\[
w =\beta-\beta_0
=\beta^* + \frac{\|\beta-\beta^*\|}{r_1}(\beta_0-\beta^*) - \beta_0
=\frac{\|\beta-\beta^*\|-r_1}{r_1}(\beta_0-\beta^*).
\]
Notice \(w\) and \(\beta_0-\beta^*\) point in exactly the same direction.

\medskip

Since \(\mathcal L\) is convex and differentiable everywhere, we have
\[
\mathcal L(\beta)
=\mathcal L(\beta_0 + w)
\ge
\mathcal L(\beta_0)
+\nabla \mathcal L(\beta_0)^\top w.
\]
using the supporting hyperplane property of convex functions.

\medskip
Note that for \(v=(\beta_0-\beta^*)/r_1\), using the support hyperplane property again, we have
\[
v^\top\nabla \mathcal L(\beta_0) \;>\;0.
\]
But 
\[
w =\frac{\|\beta-\beta^*\|-r_1}{r_1}(\beta_0-\beta^*)
=(\|\beta-\beta^*\|-r_1)v,
\]
so
\[
\nabla \mathcal L(\beta_0)^\top w
=(\|\beta-\beta^*\|-r_1)v^\top\nabla \mathcal L(\beta_0)
>0.
\]
Hence we have
\[
\mathcal L(\beta)\ge \mathcal L(\beta_0)+\nabla \mathcal L(\beta_0)^\top w
>\mathcal L(\beta_0).
\]
Thus no minimizer can lie on or outside the sphere of radius \(r_1\).

\medskip

Now, since \(\hat\beta_{\rm true}\in B_{r_1}(\beta^*)\) and \(\nabla \mathcal L(\hat\beta_{\rm true})=0\), using Taylor expansion of the gradient around \(\beta^*\), we have:
\[
0
=\nabla \mathcal L\bigl(\hat\beta_{\rm true}\bigr)
=\nabla \mathcal L(\beta^*)
+\underbrace{\int_0^1\nabla^2\mathcal L(\beta^*+t(\hat\beta_{\rm true}-\beta^*))\,\mathrm dt}_{=:H_{\rm avg}}\;
(\hat\beta_{\rm true}-\beta^*).
\]
By construction \(H_{\rm avg}\succeq(\mu/4)I_d\). Therefore
\[
\frac\mu4\|\hat\beta_{\rm true}-\beta^*\|
\le
\|H_{\rm avg}(\hat\beta_{\rm true}-\beta^*)\|
=\|\nabla \mathcal L(\beta^*)\|
\le C_0\sqrt{\frac{\log n}{n}}.
\]

\medskip
\noindent
Next, we get the bound for $\|\hat{\beta} - \hat{\beta}_{\rm true}\|$. Let $\Delta = \hat{\beta} - \hat{\beta}_{\rm true}$. On $\mathcal{E}$, $\hat{\beta}, \hat{\beta}_{\rm true} \in B_{r_1}(\beta^*)$, so the segment $[\hat{\beta}_{\rm true}, \hat{\beta}] \subset B_{r_1}(\beta^*)$.
On $\mathcal{E}_3 \subset \mathcal{E}$, $\nabla^2 \tilde{\mathcal L}(\beta) \succeq \tilde{\mu} I_d$ on this segment.
Let $f(\Delta') = \tilde{\mathcal L}(\hat{\beta}_{\rm true} + \Delta')$. $f$ is $\tilde{\mu}$-strongly convex along $[0, \Delta]$.
Gradient monotonicity gives $\langle \nabla f(\Delta) - \nabla f(0), \Delta - 0 \rangle \ge \tilde{\mu} \|\Delta\|^2$.
With $\nabla f(\Delta)=0$ and $\nabla f(0) = \nabla(\tilde{\mathcal L}-\mathcal L)(\hat{\beta}_{\rm true})$:
\[ \langle -\nabla(\tilde{\mathcal L}-\mathcal L)(\hat{\beta}_{\rm true}), \Delta \rangle \ge \tilde{\mu} \|\Delta\|^2. \]
By Cauchy-Schwarz inequality:
\[ \|\nabla(\tilde{\mathcal{L}}-\mathcal{L})(\hat{\beta}_{\rm true})\| \|\Delta\| \ge \tilde{\mu} \|\Delta\|^2. \]
If $\|\Delta\| \ne 0$:
\[ \|\Delta\| \le \frac{1}{\tilde{\mu}} \|\nabla(\tilde{\mathcal{L}}-\mathcal{L})(\hat{\beta}_{\rm true})\|. \]
We use the similar derivations we had before to bound the gradient differences. Let $B(n) = \|\hat{\beta}_{\rm true}\| = \ocal(\log n)$ on $\mathcal{E}$. Then
\[ \|\nabla(\tilde{\mathcal{L}}-\mathcal{L})(\hat{\beta}_{\rm true})\| \le (2 + \ocal(\log n\log m))\delta. \]
Let $C_G(n) = 2 + \ocal(\log n\log m)$. Then $\|\nabla(\tilde{\mathcal{L}}-\mathcal{L})(\hat{\beta}_{\rm true})\| \le C_G(n)\delta$.
Substituting:
\[ \|\hat{\beta} - \hat{\beta}_{\rm true}\| \le \frac{C_G(n)}{\tilde{\mu}} \delta = \frac{8 C_G(n)}{\mu} \delta. \]
Let $C_\Delta(n) = 8 C_G(n)/\mu = \ocal(\log n\log m)$. Then on $\mathcal{E}$:
\[ \|\hat{\beta} - \hat{\beta}_{\rm true}\| \le C_\Delta(n) \delta = \ocal(\delta\log n\log m). \]

\medskip

\noindent In summary, under event $\mathcal{E}$:
\begin{align*}
  \|\hat{\beta} - \beta^*\| &\le \|\hat{\beta} - \hat{\beta}_{\rm true}\| + \|\hat{\beta}_{\rm true} - \beta^*\| \\
  &\le \ocal\left(\max\left\{\delta \log n\log m, \frac{\log n}{\sqrt{n}}\right\}\right). \label{eq:combined_bound_log}
\end{align*}

This holds with probability at least $1 - c_5 n^{-(k+1)}$. This completes the proof.

\medskip

Before concluding the proof, we list a few side notes here.
\begin{itemize}
\item If we know that all $x_i$'s are bounded and $\beta$ is also bounded, following the same derivations above would drop the $\log n$ terms involved because of the increasing norm. Also, we no longer need the polynomial probability from $\mathcal{E}_0$ and $\mathcal{E}_1$. And the result error bound would be 
\[ \|\hat{\beta} - \beta^*\| \le C \max\left\{\delta, \sqrt{\frac{\log n}{n}}\right\}\]
with probability at least $1 - e^{-c_5 n}$ as long as $\delta \gg 1/\sqrt{n}$.
\item If we have one offset covariate, the proof and the result still hold. The only difference is that we exclude one coordinate from $\beta$ in the estimation. 
\end{itemize}

\end{proof}

With the preparation of Lemma~\ref{thm:logistic-subexponential}, we are ready to prove the result for $\hat{Z}_i$'s.

\begin{thm}\label{thm:LSM-bound}
  Under the inner product latent space model model. Let the distribution $F$ satisfy the following two properties:
\begin{enumerate}
    \item The support of $F$ is contained within a ball of radius $R$ for some $R>0$
    \item Let $\Sigma = \e[Z Z^\top]$. There exists a constant $\mu > 0$ such that:
    $$\Sigma \succeq \mu I_d$$
    where $I_d$ is the $d \times d$ identity matrix.
\end{enumerate}
  Moreover, on the hold-out data, suppose that with probability tending to 1, we can achieve
  $$\max_{n+1 \le i\le n+m}\norm{\hat{Z}_i - Z_i} \le \delta_{m} = o(1).$$
  Let $\hat{Z_i}$, $1\le i \le n$, be the estimated latent space vectors from our node-wide maximum likelihood estimation procedure (Algorithm~\ref{algo:main}). There exists positive constants $c, C$ and $C'$ such that for sufficiently large $n$, we have
  $$\max_{1\le i\le n}\norm{\hat{Z}_i - Z_i} \le C \max\left\{\delta_m, \frac{\sqrt{\log m+\log n}}{\sqrt{m}}\right\}$$
  with probability approaching 1, for $m \gg \log n$.
\end{thm}

\begin{proof}[Proof of Theorem~\ref{thm:LSM-bound}]

In the node-wise estimation procedure, for any fixed node $i$, we can treat $(\hat{X}_i, \hat{\alpha}_i)$ as the $\hat{\beta}$ in a logistic regression problem with true covariates $Z_j$, $n+1 \le j\le n+m$, and perturbed covariates $\hat{Z}_j$, $n+1 \le j\le n+m$. We apply Lemma~\ref{thm:logistic-subexponential}  to all $1\le i\le n$ to get the result. Notice that in handling the success probability, the event involving $\hat{Z}_i$, $n+1 \le i\le n+m$, is indeed shared across all $1\le i\le n$, for which the probability control can be improved.

To apply Lemma~\ref{thm:logistic-subexponential} to all $\hat{Z}_i$, $1\le i\le n$, the only requirement is that the positive definite assumption on the population (expected) Hessian in Lemma~\ref{thm:logistic-subexponential} holds uniformly for all $1\le i\le n$. To see this, let $\beta$ be an arbitrary latent vector from the distribution $F$. The population Hessian for the logistic loss is defined as the expectation of the single-point Hessian over the distribution of the covariates $x \sim F$:
\begin{equation}
\label{eq:hess_def}
H(Z_i) = \e_{Z_j, A_{ij}} [ \nabla^2 \ell(Z_i; Z_j, A_{ij}) ] = \e [ \sigma'( Z_i^\top Z_j ) Z_j Z_j^\top ]
\end{equation}
where $\sigma'(z) = e^z/(1+e^z)^2$ is the derivative of the standard sigmoid function.

To find a uniform lower bound for the smallest eigenvalue of $H(\beta)$, it is sufficient to find a $\mu > 0$ such that for any unit vector $v \in \mathbb{R}^d$ (i.e., $\|v\|=1$), we have $v^\top H(Z_i) v \ge \mu$, for any $Z_i$ in the domain.

By the Cauchy-Schwarz inequality, we have
$$|Z_i^\top Z_j| \le \|Z_i\| \|Z_j\| \le R^2$$
The minimum value of $\sigma'$ on the interval $[-R^2, R^2]$ occurs at the endpoints $\sigma'(R^2)$.
Let us define this positive constant as $w_{\min} := \sigma'(R^2) > 0$. Since $|Z_i^\top Z_j| \le R^2$ for any draws, we have
$$\sigma'(\beta^\top x) \ge w_{\min}.$$

From \eqref{eq:hess_def}, we can write the quadratic form as:
\begin{align}
\label{eq:quad_form}
v^\top H(Z_i) v &= \e [ \sigma'(Z_i^\top Z_j) (v^\top Z_j)^2 ] \ge w_{\min} \e [ (v^\top Z_j)^2 ] = w_{\min} v^\top ( \e [Z_j Z_j^\top] ) v.
\end{align}
That means, 
$$v^\top H(Z_i) v \ge w_{\min} (v^\top \Sigma v)$$
for any $Z_i$ in the domain. 

Therefore, as long as we have $\Sigma \succeq \mu I_d$ for some $\mu>0$, the requirement hold for all $Z_i$, $i\in [n]$.

\end{proof}

\subsection{Consistency under the RDPG}

Similar to the inner product model scenario, for the RDPG case, the etimation is essentially an ordinary least square (OLS) problem with measurement errors. Thus we first set up a generic lienar regression problem and study the crucial properties.

We consider i.i.d.\ data $\{(x_i,y_i)\}_{i=1}^n$ where $x_i \in \mathbb{R}^d$ and $y_i \in \{0,1\}$, generated conditional on true predictors $x_i$ with probability $p_i = \E[y_i \mid x_i] = x_i^\top \beta^*$, where $\beta^* \in \mathbb{R}^d$ is the true parameter vector.
The squared error loss for a single data point is $\ell_{\rm sq}(\beta; x, y) = (y - x^\top \beta)^2$.

For the proposed method, we observe perturbed predictors $\{\tilde{x}_i\}_{i=1}^n$ satisfying $\norm{\tilde{x}_i - x_i} \le \delta_n$.
The average empirical losses based on true and perturbed data can hence be represented as:
$$ \mathcal L(\beta) = \frac{1}{n}\sum_{i=1}^n (y_i - x_i^\top \beta)^2. $$
$$ \tilde{\mathcal L}(\beta) = \frac{1}{n}\sum_{i=1}^n (y_i - \tilde{x}_i^\top \beta)^2. $$
The population loss is $\mathcal{L}(\beta) = \E[(y - x^\top \beta)^2]$. The true parameter $\beta^*$ minimizes $\mathcal{L}(\beta)$.

The corresponding OLS estimators are:
$$ \hat{\beta}_{\rm true} = \mathop{\arg\min}_{\beta\in\mathbb{R}^d} \mathcal L(\beta). $$
$$ \hat{\beta} = \mathop{\arg\min}_{\beta\in\mathbb{R}^d} \tilde{\mathcal L}(\beta). $$
Our goal is to bound the error $\|\hat{\beta} - \beta^*\|$. We impose the following regularity conditions.
\begin{enumerate}
  \item[(B1)] \textbf{(Model):} $y_i \sim \text{Bernoulli}(x_i^\top \beta^*)$ independently.
  \item[(B2)] \textbf{(Boundedness):} $\|x_i\| \le R_x$ a.s. and $\|\beta^*\| \le R_\beta$ for constants $R_x, R_\beta > 0$. We also assume $0 \le x_i^\top \beta^* \le 1$ for all $i$.
  \item[(B3)] \textbf{(Expected Design Matrix):} $\Sigma_x = \E[x_i x_i^\top] \succeq \mu I_d$ for some fixed $\mu > 0$. Let $\hat{\Sigma}_x = \frac{1}{n} \sum x_i x_i^\top$.
  \item[(B4)] \textbf{(Perturbation):} $\|\tilde{x}_i - x_i\| \le \delta_n$, with $\delta_n \to 0$ as $n \to \infty$.
\end{enumerate}

\begin{lemma}
\label{lem:ols-perturbed}
Under Assumptions (B1)-(B4), there exist constants $C, c_1, c_2 > 0$ (depending on $R_x, R_\beta, \mu, d$) such that for sufficiently large $n$:
$$ \P\left( \|\hat{\beta} - \beta^*\| \le C \max\left\{ \sqrt{\frac{d}{n}}, \delta_n \right\} \right) \ge 1 - c_1 e^{-c_2 n}. $$
\end{lemma}

\begin{proof}
Let $\Delta_{\mathrm{total}} = \hat{\beta} - \beta^*$. We use the triangle inequality:
$$ \|\hat{\beta} - \beta^*\| \le \|\hat{\beta} - \hat{\beta}_{\rm true}\| + \|\hat{\beta}_{\rm true} - \beta^*\|. $$
Let $\mathcal{E}$ denote a high-probability event (specified later) where concentration bounds hold.

Let $\varepsilon_i = y_i - \E[y_i\mid x_i] = y_i - x_i^\top \beta^*$. $\E[\varepsilon_i \mid x_i] = 0$.
The error of the ideal OLS estimator is given by:
\[ \hat{\beta}_{\rm true} - \beta^* = \hat{\Sigma}_x^{-1} \left(\frac{1}{n} \sum_{i=1}^n x_i \varepsilon_i\right), \]
where $\hat{\Sigma}_x = \frac{1}{n} \sum x_i x_i^\top$. Taking norms:
\[ \|\hat{\beta}_{\rm true} - \beta^*\| \le \|\hat{\Sigma}_x^{-1}\|_{\rm op} \left\|\frac{1}{n} \sum_{i=1}^n x_i \varepsilon_i\right\|. \]
We bound the two terms on the right using concentration inequalities.

First, the term $\|\hat{\Sigma}_x^{-1}\|_{\rm op}$ involves the sample covariance concentration for i.i.d. random vectors. By standard concentration result (e.g., \cite{vershynin_2018}), we have
  \[ \p\left( \|\hat{\Sigma}_x - \Sigma_x\|_{\rm op} \ge \frac{\mu}{2} \right) \le 2d \exp\left( \frac{- n (\mu/2)^2}{8 R_x^4} \right) = 2d \exp\left( \frac{- n \mu^2}{32 R_x^4} \right). \]
  Define $\mathcal{E}_A$ to be the event that $\|\hat{\Sigma}_x - \Sigma_x\|_{\rm op} \le \mu/2$. Then $\P(\mathcal{E}_A) \ge 1 - 2d e^{-c' n}$ for $c' = \mu^2/(32 R_x^4)$. On event $\mathcal{E}_A$, using Weyl's inequality:
  \[ \lambdaMin(\hat{\Sigma}_x) \ge \lambdaMin(\Sigma_x) - \opnorm{\hat{\Sigma}_x - \Sigma_x} \ge \mu - \frac{\mu}{2} = \frac{\mu}{2}, \]
which leads to
  \[ \|\hat{\Sigma}_x^{-1}\|_{\rm op} \le \frac{1}{\lambdaMin(\hat{\Sigma}_x)} \le \frac{2}{\mu}. \]

For the second term, let $S = \frac{1}{n} \sum_{i=1}^n Z_i$ where $Z_i = x_i \varepsilon_i$. As noted, $\E[Z_i]=0$ and $|\varepsilon_i| \le 1$. Thus, $\|Z_i\| = \|x_i \varepsilon_i\| \le \|x_i\| |\varepsilon_i| \le R_x$. The vectors $Z_i$ are bounded, independent, mean-zero random vectors.
  By Hoeffding's inequality on bounded random vectors, we have
  \[\p\left(\left\|\frac{1}{n} \sum_{i=1}^n x_i \varepsilon_i\right\| \le C_S \sqrt{\frac{d}{n}}\right) \ge 1- c_1'e^{-c_2'n}.\]
   Define such an event as $\mathcal{E}_B$, we have
  \begin{equation}
  \P(\mathcal{E}_B) \ge 1 - c'_1 e^{-c'_2 n}.
  \end{equation}

Let $\mathcal{E}_{\rm true} = \mathcal{E}_A \cap \mathcal{E}_B$. Then $\P(\mathcal{E}_{\rm true}) \ge 1 - c''_1 e^{-c''_2 n}$. On event $\mathcal{E}_{\rm true}$:
$\|\hat{\beta}_{\rm true} - \beta^*\| \le C_{\rm true} \sqrt{d/n}$
for constant $C_{\rm true}>0$.

\medskip
Next, we bound $\|\hat{\beta} - \hat{\beta}_{\rm true}\|$.

\noindent

Let $\Delta = \hat{\beta} - \hat{\beta}_{\rm true}$. The first-order condition for $\hat{\beta}$ is $\nabla \tilde{\mathcal L}(\hat{\beta}) = 0$.
The Hessian $\nabla^2 \tilde{\mathcal L}(\beta) = 2 \hat{\Sigma}_{\tilde{x}} = \frac{2}{n} \sum \tilde{x}_i \tilde{x}_i^\top$ is constant w.r.t. $\beta$.
Since $L$ is quadratic, we have
$$ 0 = \nabla \tilde{\mathcal L}(\hat{\beta}) = \nabla \tilde{\mathcal L}(\hat{\beta}_{\rm true}) + 2 \hat{\Sigma}_{\tilde{x}} \Delta $$
Substituting $\nabla \mathcal L(\hat{\beta}_{\rm true}) = 0$ into the equation above, we have 
$$\nabla \tilde{\mathcal L}(\hat{\beta}_{\rm true}) = \nabla(\tilde{\mathcal L}-\mathcal L)(\hat{\beta}_{\rm true}).$$
Assuming $\hat{\Sigma}_{\tilde{x}}$ is invertible (justified below), we get
$$ \Delta = - \frac{1}{2} \hat{\Sigma}_{\tilde{x}}^{-1} \nabla(\tilde{\mathcal L}-\mathcal L)(\hat{\beta}_{\rm true})$$
and
$$ \|\Delta\| \le \frac{1}{2} \|\hat{\Sigma}_{\tilde{x}}^{-1}\|_{\rm op} \|\nabla(\tilde{\mathcal L}-\mathcal L)(\hat{\beta}_{\rm true})\|. $$

Again, we bound the two terms on the right. 

\medskip

For the first,   
$$ \|\hat{\Sigma}_{\tilde{x}} - \hat{\Sigma}_x\|_{\rm op} \le \frac{1}{n} \sum_{i=1}^n \|\tilde{x}_i \tilde{x}_i^\top - x_i x_i^\top\|_{\rm op} \le (2R_x+\delta_n)\delta_n. $$
  Note that $\delta_n \to 0$. On event $\mathcal{E}_A$ ($\lambdaMin(\hat{\Sigma}_x) \ge \mu/2$) and for sufficiently large $n$, we get $$\lambdaMin(\hat{\Sigma}_{\tilde{x}}) \ge \frac{\mu}{4}.$$ 
  Let $\mathcal{E}_C$ be this event ($\P(\mathcal{E}_C)\ge \P(\mathcal{E}_A)$).

For the second term of gradient difference: $\nabla(\tilde{\mathcal L}-\mathcal L)(\beta) = \frac{1}{n} \sum_{i=1}^n T_i(\beta)$, where
  $$ T_i(\beta) = (\tilde{x}_i^\top \beta - y_i)\tilde{x}_i - (x_i^\top \beta - y_i)x_i = \tilde{x}_i^\top \beta (\tilde{x}_i-x_i) + (\tilde{x}_i-x_i)^\top\beta x_i - y_i(\tilde{x}_i-x_i). $$
  Thus, on event $\mathcal{E}_{\rm true}$, we have
  $$\|T_i(\hat{\beta}_{\rm true})\| \le (|\tilde{x}_i^\top \hat{\beta}_{\rm true}| + |y_i|) \delta_n + |(\tilde{x}_i - x_i)^\top \hat{\beta}_{\rm true}| \|x_i\| = \ocal(\delta_n).$$
 Therefore, we have
  $$ \|\nabla(\tilde{\mathcal L}-\mathcal L)(\hat{\beta}_{\rm true})\| \le \frac{1}{n}\sum_{i=1}^n \|T_i(\hat{\beta}_{\rm true})\| \le \ocal(\delta_n). $$

Substituting these bounds into the inequality for $\|\Delta\|$, we get on event $\mathcal{E} = \mathcal{E}_{\rm true} \cap \mathcal{E}_C$:
$$ \|\hat{\beta} - \hat{\beta}_{\rm true}\| = \|\Delta\| \le \frac{1}{2} \left(\frac{4}{\mu}\right) (C_G \delta_n) = \frac{2 C_G}{\mu} \delta_n. $$

\medskip
\noindent
Finally, let $\mathcal{E} = \mathcal{E}_{\rm true} \cap \mathcal{E}_C$. Taking $d$ as a constant, using the two bounds, we get
$$ \|\hat{\beta} - \beta^*\| \le C \max\left\{ \sqrt{\frac{1}{n}}, \delta_n \right\} $$
with probability at least $1 - c_1 e^{-c_2 n}$.

\end{proof}

Similarly as in the inner product latent space model, for RDPG, by Lemma~\ref{lem:ols-perturbed}, we get the following result.

\begin{thm}\label{thm:RDPG-bound}
  Under the random dot product graph model, assume that the latent random vector $Z_i$'s are bounded. Suppose $\e[Z_iZ_i^\top]\succeq \mu I_d$ for some $\mu>0$. Moreover, on the hold-out data, suppose that with probability tending to 1, we can achieve
  $$\max_{n+1 \le i\le n+m}\norm{\hat{Z}_i - Z_i} \le \delta_{m} = o(1).$$
  Let $\hat{Z_i}$, $1\le i \le n$ be the estimated latent space vectors from our node-wide maximum likelihood estimation procedure (Algorithm~\ref{algo:main}). If  $m \gg \log{n}$, there exists positive constants $c, C$ and $C'$ such that for sufficiently large $n$, we have
  $$\max_{1\le i\le n}\norm{\hat{Z}_i - Z_i} \le C \max\left\{\delta_m, \frac{1}{\sqrt{m}}\right\}$$
  with probability approaching 1 as $n, m\to \infty$.
\end{thm}

\begin{proof}
  The proof is simply applying Lemma~\ref{lem:ols-perturbed} for each estimator $\hat{Z}_i$, $1\le i\le n$, then taking the union bound. 
\end{proof}

\subsection{Corollary~\ref{coro:main-specific-models}}

\begin{proof}[Proof of Corollary~\ref{coro:main-specific-models}]
  The only thing we need to check is Assumption~\ref{ass:embedding-concentration}. Note that the latent vectors in both models are not identifiable up to an orthogonal transformation. Since we have already match the orientations in our algorithm, we do not have to worry about orthogonal transformation anymore.

    For the inner product latent space model, Theorem~2.1 of \cite{li2023statistical} shows that 
  $$\max_{n+1\le i\le n+m } \norm{\hat{Z}_i - Z_i} = \ocal_\P\left(\frac{1}{m^{1/2-c'}}\right)$$
  for some constant $c'<1/2$. 

  For the RDPG model or gRDPG, Theorem~1 in \cite{rubin2022statistical} already gives that 
  $$\max_{n+1\le i\le n+m } \norm{\hat{Z}_i - Z_i} = \ocal_\P\left(\frac{(\log m)^{c'}}{\sqrt{m}}\right)$$
  for some constant $c'>0$. 

By applying Theorem~\ref{thm:model-consistency}, with $\delta_m = 1/m^{1/2-c'}$ and $(\log m)^{c'}/\sqrt{m}$, respectively. We can see that Assumption~\ref{ass:embedding-concentration} holds.

\end{proof}

\newpage

\section{Proof of Theorem~\ref{thm:main-moments} (Convergence of Network Moments)}

The conditional expectation of network moments, given the latent vectors, are essentially U-statistics. Hence, our proof uses the conditional expectation as an intermediate quantity. To structure the proof, we begin by analyzing the distribution of a U-statistic with respect to the approximation of an empirical CDF.

Let $X_1, X_2, \dotsc, X_n$ be independent and identically distributed (i.i.d.) random variables from a continuous distribution $F$ on $\bR^d$. Suppose we have approximate observations $\hat{X}_1, \hat{X}_2, \ldots, \hat{X}_n$, which may be dependent, but whose empirical distribution function $\hat{F}_n$ converges uniformly in probability to $F$. We are interested in understanding the relationship between U-statistics computed from the exact data and those computed from the approximate data.

Let $h: (\bR^d)^r \to \bR$ be a bounded, Lipchitz continuous, and symmetric kernel function. We define the U-statistics:
\[
U_n = \frac{1}{\binom{n}{r}} \sum_{1 \leq i_1 < \dotsb < i_r \leq n} h(X_{i_1}, \dotsc, X_{i_r}),
\]
\[
\hat{U}_n = \frac{1}{\binom{n}{r}} \sum_{1 \leq i_1 < \dotsb < i_r \leq n} h(\hat{X}_{i_1}, \dotsc, \hat{X}_{i_r}).
\]

Let us first introduce some notation and preliminary results that will be used in the proof.

\textit{Empirical Measures on \( r \)-Tuples}:

- The empirical measure over combinations without replacement from \( X_i \) is:
\[
F_n^{[r]} = \frac{1}{\binom{n}{r}} \sum_{(i_1, \dots, i_m) \in I_n} \delta_{(X_{i_1}, \dots, X_{i_r})},
\]
where \( \delta_{(X_{i_1}, \dots, X_{i_r})} \) is the Dirac measure at the point \( (X_{i_1}, \dots, X_{i_r}) \) in \( (\bR^d)^r \).

- Similarly, for \( \hat{X}_i \):
\[
\hat{F}_n^{[r]} = \frac{1}{\binom{n}{r}} \sum_{(i_1, \dots, i_m) \in I_n} \delta_{(\hat{X}_{i_1}, \dots, \hat{X}_{i_r})}.
\]

\textit{Product Measures}:

- The product measure of \( F \) is \( F^{\otimes r} \), defined on \( (\bR^d)^r \) as:
\[
F^{\otimes r}(A) = \int_{\mathcal A} \,\mathrm dF(x_1) \cdots \mathrm dF(x_r),
\]
for measurable sets \( \mathcal A \subset (\bR^d)^r \).

\begin{lemma}\label{lemma:weak_convergence}
Let $F$ be a CDF on $\bR^d$ and $\hat{F}_n$ be empirical distribution functions on $\bR^d$ defined as:
\[
 \hat{F}_n(x) = \frac{1}{n} \sum_{i=1}^n \delta_{\hat{X}_i}(x),
\]
where:
\begin{itemize}
  \item $\hat{X}_1, \hat{X}_2, \ldots, \hat{X}_n$ are random variables (not necessarily independent) such that:
  \[
  \sup_{x \in \bR^d} \big| \hat{F}_n(x) - F(x) \big| \xrightarrow{\P} 0 \quad \text{as } n \to \infty.
  \]
\end{itemize}
Then, for each $m \geq 1$, the product measures $F_n^{\otimes r}$ and $\hat{F}_n^{\otimes r}$ converge weakly to $F^{\otimes r}$ on $(\bR^d)^r$, that is:
\begin{align*}
\hat{F}_n^{\otimes r} &\xrightarrow{\mathrm{w.}\mathbb P} F^{\otimes r},
\end{align*}
where $\xrightarrow{\mathrm{w.}\mathbb P}$ denotes weak convergence in probability.
\end{lemma}

\begin{lemma}\label{lemma:weak_convergence}
Let $F_n$ and $\hat{F}_n$ be empirical distribution functions on $\bR^d$ defined as:
\[
F_n(x) = \frac{1}{n} \sum_{i=1}^n \delta_{X_i}(x), \quad \hat{F}_n(x) = \frac{1}{n} \sum_{i=1}^n \delta_{\hat{X}_i}(x),
\]
where:
\begin{itemize}
  \item $X_1, X_2, \dotsc, X_n$ are independent and identically distributed (i.i.d.) random variables with distribution $F$ on $\bR^d$.
  \item $\hat{X}_1, \hat{X}_2, \dotsc, \hat{X}_n$ are random variables (not necessarily independent) such that:
  \[
  \sup_{x \in \bR^d} \big| \hat{F}_n(x) - F(x) \big| \xrightarrow{\P} 0 \quad \text{as } n \to \infty.
  \]
\end{itemize}
Then, for each $m \geq 1$, the product measures $F_n^{\otimes r}$ and $\hat{F}_n^{\otimes r}$ converge weakly to $F^{\otimes r}$ on $(\bR^d)^r$, that is:
\[
F_n^{\otimes r} \xrightarrow{\mathrm{w.a.s.}} F^{\otimes r},\quad \hat{F}_n^{\otimes r} \xrightarrow{\mathrm{w.\P}} F^{\otimes r},
\]
where $\xrightarrow{\mathrm{w.a.s.}}$ denotes weak convergence almost surely, and $\xrightarrow{\mathrm{w.\P}}$ denotes weak convergence in probability.
\end{lemma}

\begin{proof}

Since $X_1, X_2, \dotsc, X_n$ are i.i.d. with distribution $F$, by the Glivenko--Cantelli theorem, we have:
\[
\sup_{x \in \bR^d} \big| F_n(x) - F(x) \big| \xrightarrow{\mathrm{a.s.}} 0 \quad \text{as } n \to \infty.
\]

This implies that for any bounded, continuous function $g: \bR^d \to \bR$,
\[
\int g \, \mathrm dF_n \xrightarrow{\mathrm{a.s.}} \int g \, \mathrm dF.
\]

Now, consider any bounded, continuous function $h: (\bR^d)^r \to \bR$. We need to show that:
\[
\int h \, \mathrm dF_n^{\otimes r} \xrightarrow{\mathrm{a.s.}} \int h \, \mathrm dF^{\otimes r}.
\]

Since $F_n^{\otimes r} = F_n \otimes F_n \otimes \cdots \otimes F_n$ ($r$ times), we have:
\[
\int h(x_1, x_2, \dotsc, x_r) \, \mathrm dF_n^{\otimes r}(x_1, x_2, \dotsc, x_r) = \int \cdots \int h(x_1, x_2, \dotsc, x_r) \, \mathrm dF_n(x_1) \cdots \mathrm dF_n(x_r).
\]

Similarly, for $F^{\otimes r}$:
\[
\int h(x_1, x_2, \dotsc, x_r) \, \mathrm dF^{\otimes r}(x_1, x_2, \dotsc, x_r) = \int \cdots \int h(x_1, x_2, \dotsc, x_r) \, \mathrm dF(x_1) \cdots \mathrm dF(x_r).
\]

Since $F_n \xrightarrow{\mathrm{a.s.}} F$ uniformly, for each $j = 1, 2, \dotsc, m$, we have:
\[
\int g_j \, \mathrm dF_n \xrightarrow{\mathrm{a.s.}} \int g_j \, \mathrm dF
\]
for any bounded, continuous function $g_j: \bR^d \to \bR$.

By the weak convergence of product measures \citep{billingsley2013convergence}, it follows that:
\[
\int h \, \mathrm dF_n^{\otimes r} \xrightarrow{\mathrm{a.s.}} \int h \, \mathrm dF^{\otimes r}.
\]

The convergence of $\hat{F}_n^{\otimes r}$ to $F^{\otimes r}$ can be proved with the same reasoning.

\end{proof}

\begin{lemma}\label{thm:u-stat-convergence}
  Under the above assumptions of Lemma~\ref{lemma:weak_convergence}, we have
\[
\hat{U}_n - U_n \xrightarrow{\P} 0 \quad \text{as } n \to \infty.
\]
\end{lemma}
\begin{proof}

We aim to show that \( \hat{U}_n - U_n \xrightarrow{\P} 0 \) as \( n \to \infty \). We will use
\[
\theta = \int h \, \mathrm dF^{\otimes r} = \mathbb{E}[ h(X_1, \dotsc, x_r) ]
\]
as the intermediate quantity for the convergence. To directly characterize the approximation error between the U-statistic and the integral over the product measure, we consider the difference between sampling without replacement (as in \( U_n, \hat{U}_n \)) and sampling with replacement (as in \( \int h \, \mathrm d{F}_n^{\otimes r}, \int h \, \mathrm d\hat{F}_n^{\otimes r} \)).

It is well-known that the U-statistic \( U_n \) can be related to an integral over the product measure \( F_n^{\otimes r} \), but with a subtle difference in normalization. Let
\[
S_{\mathrm{distinct}} = \sum_{1 \leq i_1 < \cdots < i_r \le n} h(\hat{X}_{i_1}, \dotsc, \hat{X}_{i_r})
\]
be the sum over all distinct \( r \)-tuples from \( \hat{X}_i \)'s. Also define
\[
S_{\mathrm{all}} = \sum_{i_1=1}^n \cdots \sum_{i_r=1}^n h(\hat{X}_{i_1}, \dotsc, \hat{X}_{i_r}),
\]
the sum over all \( r \)-tuples. With these definitions, we have
\[
\hat{U}_n = \frac{S_{\mathrm{distinct}}}{\binom{n}{r}}, \quad \int h \, \mathrm d\hat{F}_n^{\otimes r} = \frac{S_{\mathrm{all}}}{n^r}.
\]
Now, taking the difference of the two terms as
\[
\varepsilon_n = \hat{U}_n - \int h \, \mathrm d\hat{F}_n^{\otimes r} = \frac{S_{\mathrm{distinct}}}{\binom{n}{r}} - \frac{S_{\mathrm{all}}}{n^r}.
\]

\noindent If we consider only tuples with distinct indices, there are \(\binom{n}{r} r!\) such tuples. Therefore, we have
\[
S_{\mathrm{all}} = r! S_{\mathrm{distinct}} + R_n,
\]
where \( R_n \) accounts for the contributions from tuples that include repeated indices. The number of tuples with repeated indices is at most:
\(
n^r - \binom{n}{r} r!.
\)
Note that from the definitions of combinatorial terms, we have
\[
\binom{n}{r} r! = n^r \left(1 - \ocal\left(\frac{1}{n}\right)\right), \quad n^r - \binom{n}{r} r! = n^r \ocal\left(\frac{1}{n}\right) = \ocal(n^{r-1}).
\]
\noindent From this, we get
\(
\hat{U}_n = S_{\mathrm{distinct}}/\binom{n}{r} = (S_{\mathrm{all}} - R_n)/(r! \binom{n}{r}).
\)
Compare this with \(\int h \, \mathrm d\hat{F}_n^{\otimes r}\), we have
\[
\hat{U}_n - \int h \, \mathrm d\hat{F}_n^{\otimes r} = \frac{S_{\mathrm{all}} - R_n}{r! \binom{n}{r}} - \frac{S_{\mathrm{all}}}{n^r} = S_{\mathrm{all}}\left(\frac{1}{r!\binom{n}{r}} - \frac{1}{n^r}\right) - \frac{R_n}{r! \binom{n}{r}}.
\]
With the boundedness assumption of $h$, we have \( S_{\mathrm{all}} \leq M n^r \), thus:
\[
S_{\mathrm{all}}\left(\frac{1}{r! \binom{n}{r}} - \frac{1}{n^r}\right) = M n^r \ocal\left(\frac{1}{n^{r+1}}\right) = \ocal\left(\frac{1}{n}\right).
\]
Similarly, we have \( |R_n| = M \ocal(n^{r-1}) \) and 
\(
1/(r! \binom{n}{r}) = \ocal(1/n).
\)
Combining these leads to
\[
\left|\hat{U}_n - \int h \, \mathrm d\hat{F}_n^{\otimes r}\right| \leq \ocal\left(\frac{1}{n}\right) + \ocal\left(\frac{1}{n}\right) = \ocal\left(\frac{1}{n}\right).
\]

\noindent Thus \( \hat{U}_n - \int h \, \mathrm d\hat{F}_n^{\otimes r} = \ocal(1/n) \to 0 \) as \( n \to \infty \).

\medskip

\noindent Using the same reasoning, define
\(
\varepsilon_n = {U}_n - \int h \, \mathrm d{F}_n^{\otimes r}
\)
and we also have 
\(
| {\varepsilon}_n | = \ocal(1/n).
\)

Now we have the decomposition
\begin{align*}
\hat{U}_n - U_n &= \left( \hat{U}_n - \int h \, \mathrm d\hat{F}_n^{\otimes r} \right) + \left( \int h \, \mathrm d\hat{F}_n^{\otimes r} -\theta\right) + \left(\theta- \int h \, \mathrm dF_n^{\otimes r} \right) + \left( \int h \, \mathrm dF_n^{\otimes r} - U_n \right) \\
&= \hat{\varepsilon}_n + \left( \int h \, \mathrm d\hat{F}_n^{\otimes r} -\theta\right) + \left(\theta- \int h \, \mathrm dF_n^{\otimes r} \right) - \varepsilon_n.
\end{align*}
\noindent Therefore,
\[
| \hat{U}_n - U_n | \leq | \hat{\varepsilon}_n | + \left| \int h \, \mathrm d\hat{F}_n^{\otimes r} - \int h \, \mathrm dF^{\otimes r} \right|+ \left| \int h \, \mathrm dF^{\otimes r} - \int h \, \mathrm dF_n^{\otimes r} \right| + | \varepsilon_n |.
\]

Because of the boundedness and Lipschitz continuity of $h$, by the previous results of $\varepsilon$ and $\hat{\varepsilon}$, as well as Lemma~\ref{lemma:weak_convergence}, we have
\[
\hat{U}_n - U_n \xrightarrow{\P} 0 \quad \text{as } n \to \infty.
\]

\end{proof}

Next, we introduce the property that, conditioning on the latent spaces, the network moments concentrate around its conditional expectation.

\begin{lemma}[Conditional Concentration of Network Moments]\label{lemma:beronulli-concentration}
\label{lem:concentration-subgraph}
Consider a random (simple, undirected) graph on $n$ vertices with adjacency matrix $A = (A_{ij})$, where $1 \leq i < j \leq n$, and $A_{ij} \sim \text{Bernoulli}(P_{ij})$ are independent random variables. Let $m \geq 2$ be fixed and consider a bounded function $h$ on the induced subgraph of $m$ distinct vertices with $|h| \leq 1$. Define
\[
X(A) = \frac{1}{\binom{n}{r}} \sum_{1 \leq i_1 < i_2 < \cdots < i_r \leq n} h(A_{[i_1, i_2, \ldots, i_r]}),
\]
where $A_{[i_1, i_2, \ldots, i_r]}$ denotes the induced subgraph on vertices $\{i_1, i_2, \ldots, i_r\}$. For every $\varepsilon > 0$, there exist positive constants $c(r)$ (not depending on $n$) such that
\[
\mathbb{P}(|X - \mathbb{E}[X]| \geq \varepsilon ) \leq 2 \exp(-c(r)\varepsilon^2 n^2),
\]
and in particular,
\[
X \xrightarrow{\p} \mathbb{E}[X] \quad \text{as } n \to \infty.
\]
\end{lemma}

\begin{proof} Consider the set of edges $\{A_{ij}: 1 \leq i < j \leq n\}$. There are $\binom{n}{2}$ such edges, each an independent Bernoulli random variable. The random variable $X$ is a function of these $\binom{n}{2}$ independent variables. We will apply McDiarmid's inequality (Lemma~\ref{lem:mcdiarmid}), a bounded-differences inequality, to show that $X$ is sharply concentrated around its expectation.

First, we need to bound the effect of changing one edge on the value of $X$. Fix an edge $(u,v)$ with $1 \leq u < v \leq n$. Consider two graphs $A$ and $A'$ that differ only on the edge $(u,v)$: $A'=A$ on all entries, except for $(u,v)$, where $A'_{uv} = 1 - A_{uv}$. We have
\[
X(A) = \frac{1}{\binom{n}{r}} \sum_{1 \leq i_1 < \cdots < i_r \leq n} h(A[i_1,\ldots,i_r]), \quad X(A') = \frac{1}{\binom{n}{r}} \sum_{1 \leq i_1 < \cdots < i_r \leq n} h(A'[i_1,\ldots,i_r]).
\]
The only $r$-tuples of vertices that can be affected by the change in the edge $(u,v)$ are those that contain both $u$ and $v$. The number of such subsets is $\binom{n-2}{r-2}$, since we choose the remaining $m-2$ vertices from the $n-2$ other vertices.

For each affected $r$-tuple, the value of $h$ can change by at most $2$ in absolute value (since $|h| \leq 1$). Therefore, the change in the numerator of $X$ is at most $2 \binom{n-2}{r-2}.$
The change in $X$ when flipping a single edge $(u,v)$ is bounded by
\[
\frac{2 \binom{n-2}{r-2}}{\binom{n}{r}}.
\]

We use the combinatorial identity:
\[
\frac{\binom{n-2}{r-2}}{\binom{n}{r}} = \frac{r(r-1)}{n(n-1)}.
\]
Hence, the maximum change in $X$ due to flipping one edge is
\[
\Delta := \frac{2 r(r-1)}{n(n-1)} .
\]
As $n \to \infty$, $\Delta \approx \frac{2r(r-1)}{n^2}$, which vanishes.

Now we apply McDiarmid's inequality (Lemma~\ref{lem:mcdiarmid}). There are $M = \binom{n}{2}$ edges. Each edge affects $X$ by at most $\Delta$. Hence,
\[
\sum_{j=1}^r \Delta_j^2 \leq M \Delta^2 = \binom{n}{2} \left(\frac{2r(r-1)}{n(n-1)}\right)^2 \leq \frac{C'(r)}{n^2}
\]
for some constant $C'(r)$ that depends only on $r$.

By McDiarmid's inequality, for any $\varepsilon > 0$,
\[
\mathbb{P}(|X - \mathbb{E}[X]| \geq \varepsilon) \leq 2 \exp\left(- \frac{2\varepsilon^2}{\sum_{j=1}^r \Delta_j^2}\right).
\]
Since $\sum_{j=1}^r \Delta_j^2 = \ocal(1/n^2)$, we have
\[
\mathbb{P}(|X - \mathbb{E}[X]| \geq \varepsilon) \leq 2 \exp(-c(r)\varepsilon^2 n^2)
\]
for some constant $c(r) > 0$. This shows exponential concentration of $X$ around its expectation as $n$ grows. In particular, this indicates that
\[
X \xrightarrow{\p} \mathbb{E}[X] \quad \text{as } n \to \infty.
\]
\end{proof}

Now combining Lemma~\ref{thm:u-stat-convergence} and Lemma~\ref{lemma:beronulli-concentration}, we can get the major claim about the network moments.

\begin{proof}[Proof of Theorem~\ref{thm:main-moments}]

First, based on \cite{qi2024multivariate}, we know that the conditional expectations of $X_R(A)$ and $X_R(\tilde{A})$, given $\{Z_i\}$ and $\{\tilde{Z}_i\}$, respective, can be written as U-statistics with a bounded continuous function of $h$. That is,
\[
U_n = \mathbb{E}[X_R(A)|\{Z_i\}] = \frac{1}{\binom{n}{r}} \sum_{1\le i_1 < \cdots < i_r \le n}\mathbb{E}[\mbone(A_{[i_1, \ldots, i_r]} \cong R)|\{Z_i\}] = \frac{1}{\binom{n}{r}} \sum_{1\le i_1 < \cdots < i_r \le n}h(Z_{i_1}, \ldots, Z_{i_r})
\]
and

\[
\tilde{U}_n = \mathbb{E}[X_R(A)|\{\tilde{Z}_i\}] = \frac{1}{\binom{n}{r}} \sum_{1\le i_1 < \cdots < i_r \le n}h(\tilde{Z}_{i_1}, \ldots, \tilde{Z}_{i_r}).
\]

We have the following decomposition
\begin{align*}
  X_R(A)-X_R(\tilde{A}) = (X_R(A)- U_n) + (U_n-\tilde{U}_n) + (\tilde{U}_n - X_R(\tilde{A}) ).
\end{align*}

Conditioning on \(\{Z_i\}\) and \(\{\tilde{Z}_i\}\), respectively, the first and third term vanish in probability from Lemma~\ref{lemma:beronulli-concentration}. By Theorem~\ref{thm:main-CDF}, we know that the empirical CDF of $\tilde{Z}_i$'s satisfies the requirement of Lemma~\ref{lemma:weak_convergence}.
And then we call Lemma~\ref{thm:u-stat-convergence}, and the second term also vanishes in probability.
\end{proof}

\newpage

\section{Supporting lemmas}

\begin{lemma}[Telescoping Inequality]
\label{lem:product_difference_sum_bound}
Let \( m \geq 1 \) be an integer, and let \( \{a_i\}_{i=1}^r \) and \( \{b_i\}_{i=1}^r \) be two sequences of real numbers satisfying \( 0 \leq a_i, b_i \leq 1 \) for all \( i = 1, 2, \dots, m \). Then, the following inequality holds:
\[
\left| \prod_{i=1}^r a_i - \prod_{i=1}^r b_i \right| \leq \sum_{k=1}^r \left| a_k - b_k \right|.
\]
\end{lemma}

\begin{lemma}[McDiarmid's Inequality]
\label{lem:mcdiarmid}
Let $Y_1,\ldots,Y_M$ be independent random variables taking values in arbitrary sets, and let $f(y_1,\ldots,y_M)$ be a function such that for all $j$ and for all $y_1,\ldots,y_M, y_j'$, we have
\[
|f(y_1,\ldots,y_M) - f(y_1,\ldots,y_{j-1}, y_j', y_{j+1},\ldots,y_M)| \leq c_j,
\]
for some constants $c_j$. Define $X = f(Y_1,\ldots,Y_M)$ and $\mu = \mathbb{E}[X]$. Then for any $\varepsilon > 0$,
\[
\mathbb{P}(|X - \mu| \geq \varepsilon) \leq 2\exp\left(-\frac{2\varepsilon^2}{\sum_{j=1}^r c_j^2}\right).
\]
\end{lemma}

\begin{lemma}[Lemma 2 of \cite{yin2006asymptotic}]
\label{lem:yin}
Let $\phi: \mathbb{R}^d \to \mathbb{R}^d$ smooth, injective on $B_{\delta_1}(x^*)$. Let $\phi(x^*) = y^*$. If for $\rho, \delta_1 > 0$, $\min_{\norm{x - x^*} = \delta_1} \norm{\phi(x) - y^*} \ge \rho$. Then $\forall y$ with $\norm{y - y^*} \le \rho$, $\exists x \in B_{\delta_1}(x^*)$ s.t. $\phi(x) = y$.
\end{lemma}

\newpage
\section{Handling node attributed networks}\label{sec:attributes}
The proposed GRAND method and its theory can be readily extended to handle network data with node attributes. Here we outline this extension. Suppose in addition to the network $A$, we also observe node attribute vectors $U_i$ for each node $i$. Let $U \in \mathbb{R}^{(n+m)\times p}$ be the matrix stacking all node attributes. We use the same indices to denote $U^{1}$ as the first $n$ rows associated with the to-be-released nodes and $U^{2}$ as the last $m$ rows associated with the holdout nodes. Our task is to release a network with node attributes for individuals $1, \ldots, n$ with differential privacy at the node level.

To extend GRAND to this situation, we can define a generalization of the latent space model with node attributes.
\begin{defi}[General Attributed Latent Space Model]\label{defi:generic-attribute}
We say that $A$ is a network generated from the general latent space model with node attributes if there exists a distribution $F$ (unknown) on $\mathbb{R}^{d+p}$ and a known symmetric generative function $W: \mathbb{R}^d\times \mathbb{R}^d\to[0,1]$ such that $A$ can be generated through the procedure below:
$$
(Z_1, U_1),\ldots,(Z_N, U_N) \stackrel{\mathrm{i.i.d.}}{\sim} F, \quad\quad A_{ij}\stackrel{\mathrm{ind.}}{\sim}\text{Bernoulli}(W(Z_i,Z_j)), \quad i>j.\footnote{We can also include $U_i$ in the probabilities, if preferred.}
$$
\end{defi}
Note that here, the node attributes and the latent vectors for the network can be arbitrarily correlated or even overlapping.

Algorithm~\ref{alg:TNR} can be easily extended to handle this situation, with only two minor modifications:
\begin{enumerate}
    \item In Step 4, the CDF will be estimated using $\{(\hat{Z}_i, U_i)\}_{i=n+1}^{n+m}$.
    \item In Step 6, the privatization procedure \eqref{eq:TNR} will be applied to $(\hat{Z}_i, U_i)$ jointly.
\end{enumerate}

The theoretical results of GRAND still hold. The only change we need is to extend the regularity conditions on the CDF for the latent vectors to the joint CDF of the latent vectors and the node attributes (still denoted by $F$ under Definition~\ref{defi:generic-attribute}). These are outlined below.
\begin{coro}
\begin{enumerate}
    \item The released attributed network satisfies the node-DP requirement.
    \item Under the same conditions, the results of Theorems~\ref{thm:main-marginal}, \ref{thm:main-CDF}, and \ref{thm:main-moments} hold for the released attributed network.
\end{enumerate}
\end{coro}

\begin{proof}
    The result is true by noticing that for $\tilde{Z}_i$ properties, $U_i$'s can be simply treated as an additional $p$ dimensions of the latent vectors without zero estimation errors.
\end{proof}

\begin{figure}[H]
\centering
\subfigure{\includegraphics[width=0.3\textwidth]{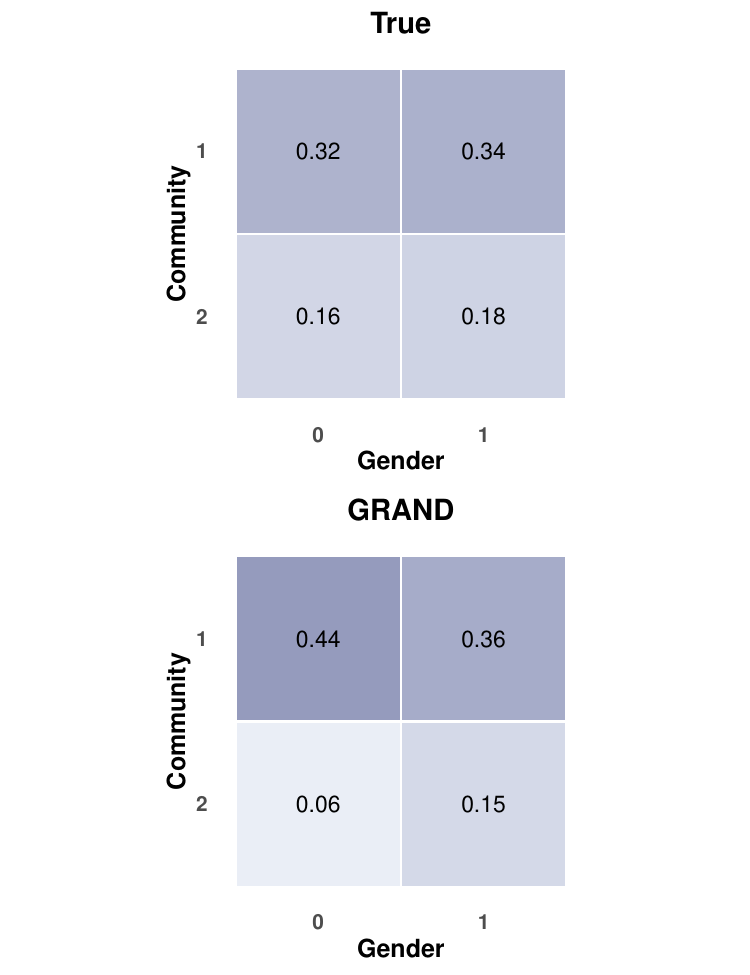}}
\hspace{-1.5cm}
\subfigure{\includegraphics[width=0.3\textwidth]{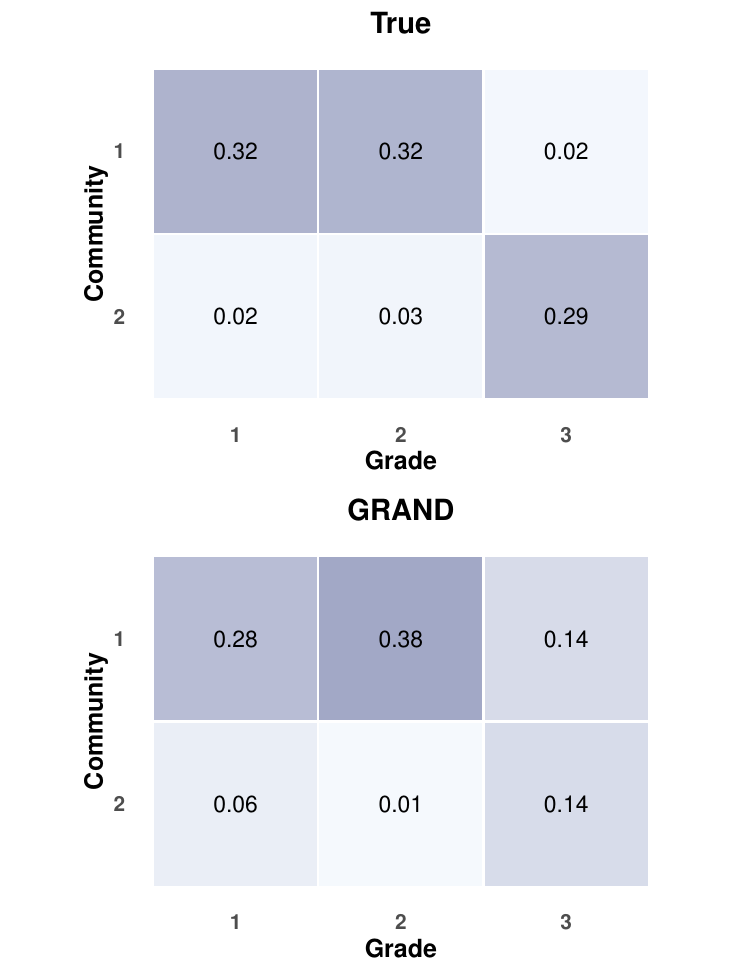}}
\hspace{-1.5cm}
\subfigure{\includegraphics[width=0.3\textwidth]{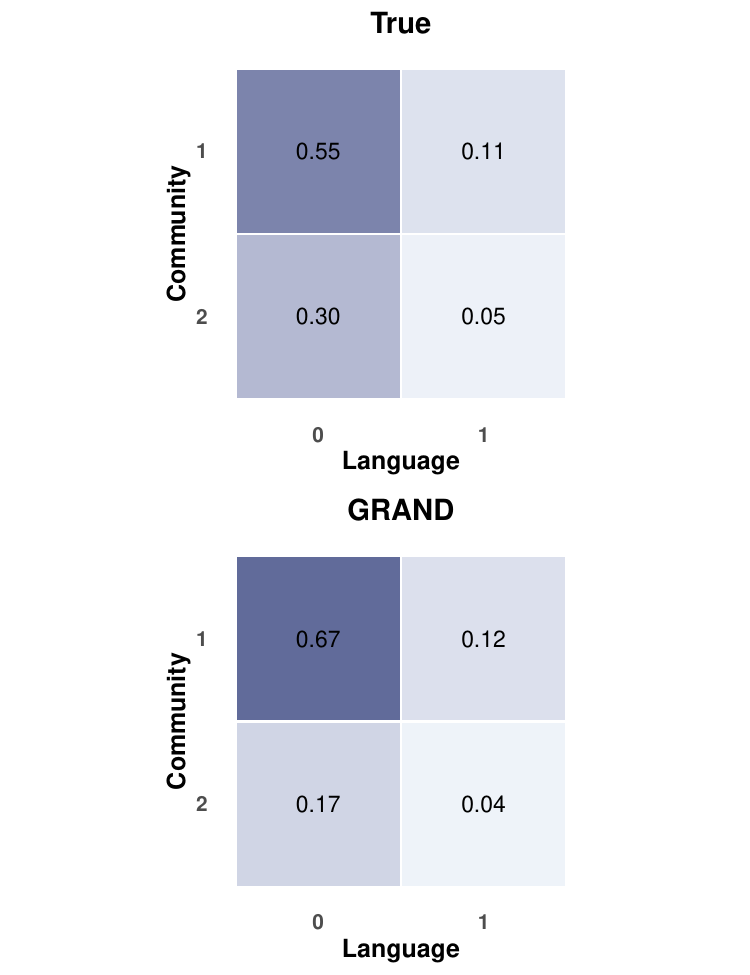}}
\hspace{-2cm}
\subfigure{\includegraphics[width=0.3\textwidth]{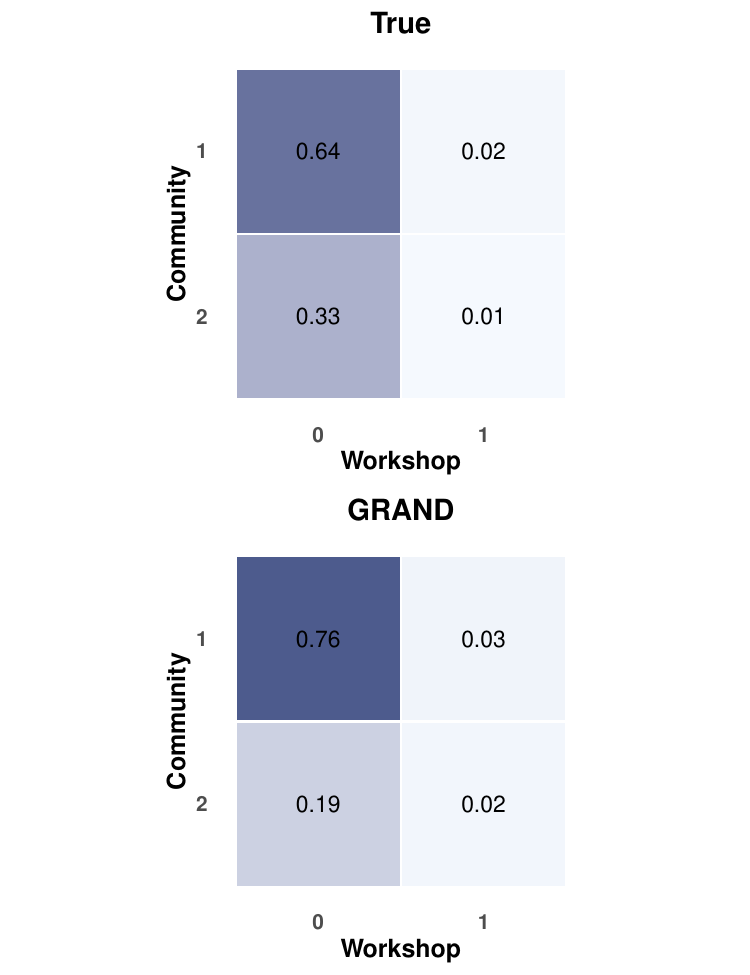}}
\caption{The distributions of four node (true/privatized) attributes across two communities in the (true/privatized) friendship network. The color of each cell indicates the proportion of the combination of values.}
\label{fig:attributes}
\end{figure}

Next, we demonstrate GRAND's performance on preserving the privacy of node-level attributes in addition to network data.
The data for this study is a friendship network from a middle school \cite{paluck2016changing,le2022linear,wang2024perturbation}. This is a social network of 653 students. In addition, the data set contains four node level attributes: Gender (Male: 1, Female: 0), Grade (1, 2, 3), (Home) Language (English: 1, Other: 0), and (If participated in educational) Workshop (Yes: 1, No: 0). We will use GRAND to privatize the network and the attributes at the same time. As before, we randomly choose half of the students to release and use $\varepsilon=1$ as the privacy budget. To demonstrate the preservation of attributes and their relations with the network, we apply spectral clustering \citep{lei2014consistency} to identify two communities from the true and privatized networks, respectively. We then look at the proportions of all the (true/privatized) attribute values across the (true/privatized) communities. The results are shown in Figure~\ref{fig:attributes}. It can be seen that GRAND reasonably preserves the overall distributions of attributes with respect to the community structure of the network.

\newpage

\section{Additional evaluation on CommunityFitNet}\label{appendix:500Network}

The CommunitFitNet is a collection of network data sets introduced in \cite{ghasemian2019evaluating}, which is also used in \cite{ghasemian2020stacking} and \cite{li2023network}. In addition to the real data analyses in the main text, we also conducted a comprehensive evaluation of our method on this data set. In particular, since the privacy problem is more meaningful in the context of social networks, we focus on the 107 social networks from this collection with size larger than 200. We follow the same procedure of releasing the network of only half of nodes while using the other half as the hold-out data set for model selection and model fitting. The results are evaluated using the same five metrics as we used in the main text.

\begin{figure}[H]
\centering
\subfigure{\includegraphics[width=0.32\textwidth]{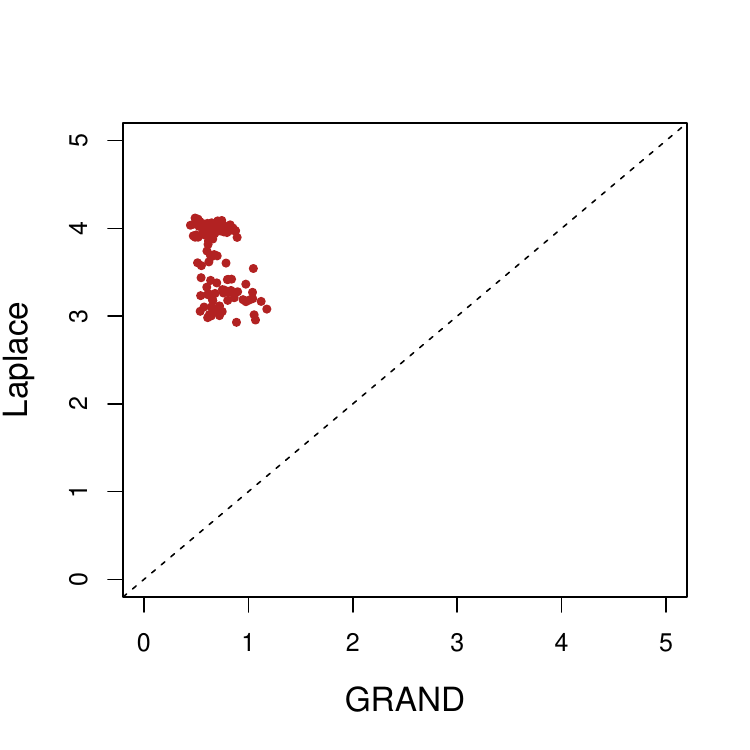}}
\subfigure{\includegraphics[width=0.32\textwidth]{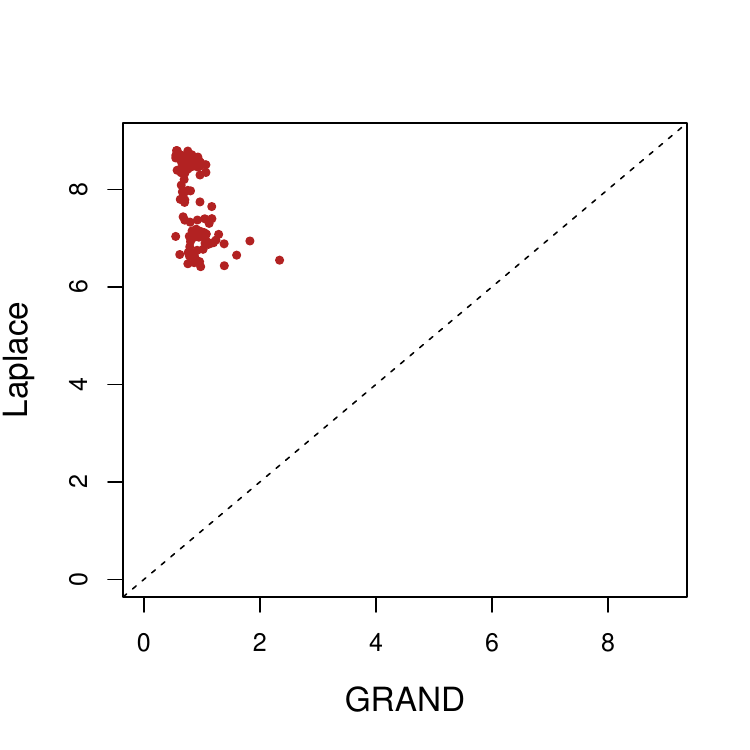}}
\subfigure{\includegraphics[width=0.32\textwidth]{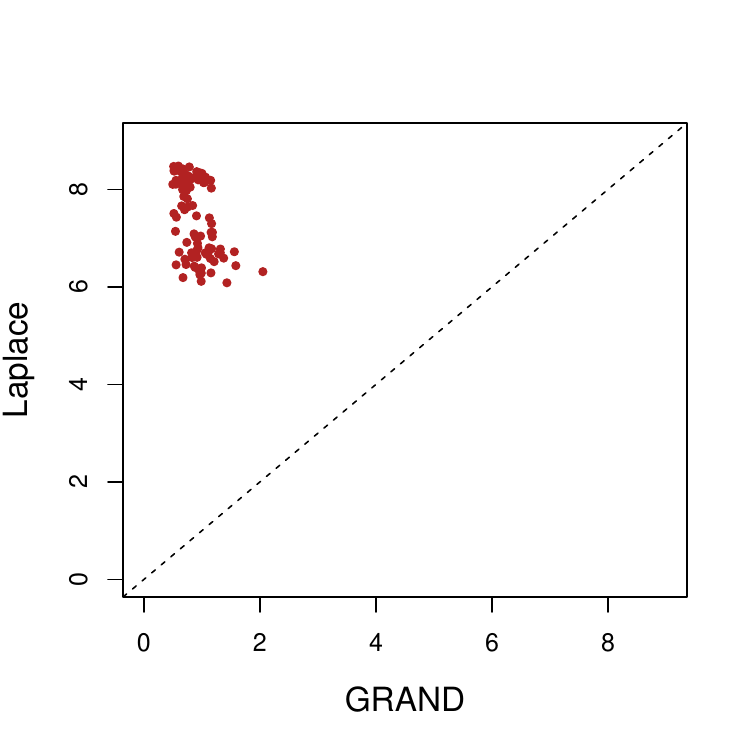}}
\\

\medskip

\subfigure{\includegraphics[width=0.32\textwidth]{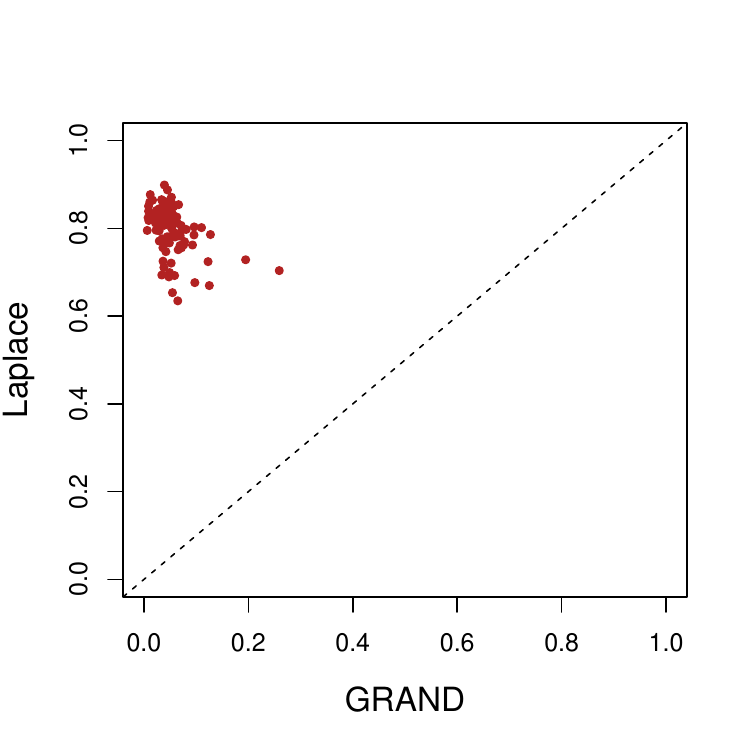}}
\subfigure{\includegraphics[width=0.32\textwidth]{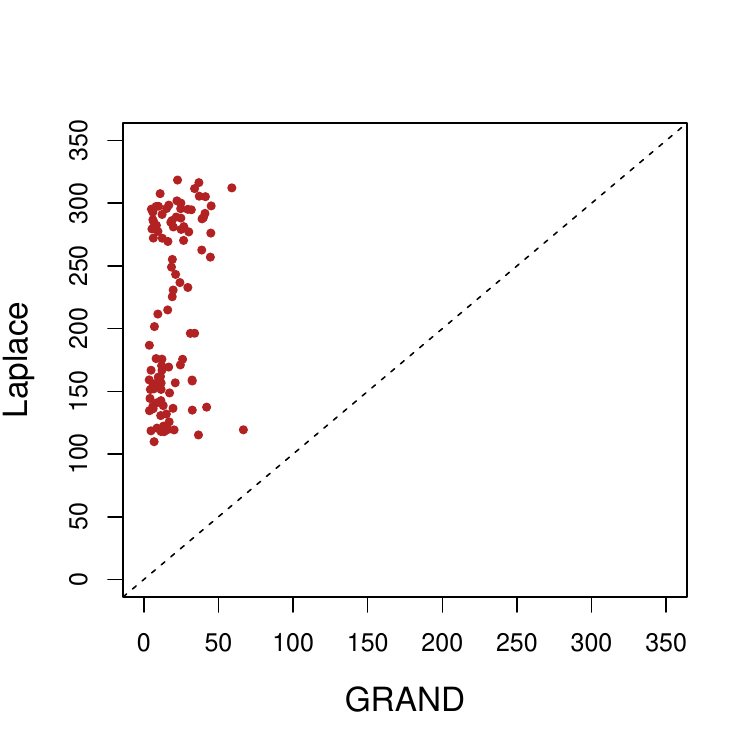}}
\caption{The Wasserstein distances between the privatized distributions of given local statistics and those in the true network on 107 social networks from CommunityFitNet.}
\label{fig:500Net}
\end{figure}
In Figure~\ref{fig:500Net}, we show the scatterplot of the Wasserstein distances between the privatized network and the true network for the distribution of local statistics. Both our method and the Laplace mechanism are evaluated for pairwise comparison. In particular, it can be seen that in all these networks, our method gives a better perservation of network properties (i.e., a much smaller Wasserstein distances across all datasets) than the naive Laplace method.

\newpage

\section{Additional Simulation Results}\label{app:additional-sim}

Table~\ref{tab:LSM-RDPG-2000-Wasserstein-rho05-noSE} presents the same set of results as in Tables~\ref{tab:LSM-RDPG-4000-d3-rho05-noSE} and \ref{tab:LSM-RDPG-4000-d6-rho05-noSE} but for $n=2000$.

Next, we include additional simulation results for $(\rho, n) =(0.025, 2000)$, $(0.1, 2000)$, $(0.025, 4000)$ and $(0.1, 4000)$, corresponding to Tables \ref{tab:LSM-RDPG-2000-Wasserstein-rho025-noSE}, \ref{tab:LSM-RDPG-2000-Wasserstein-rho1-noSE}, \ref{tab:LSM-RDPG-4000-Wasserstein-rho025-noSE}, and \ref{tab:LSM-RDPG-4000-Wasserstein-rho1-noSE}, respectively. The patterns are consistent with those in the main paper. It is worth noting that, overall, model estimation under the RDPG framework becomes noticeably noisier (than the inner product model) in very sparse networks, resulting in less discernible trends with respect to $\varepsilon$ in many cases. This is likely due to the fact that the RDPG estimation is based on first-order information rather than the likelihood \citep{athreya2017statistical}. However, in denser scenarios, we still observe the expected improvement in property preservation as the privacy budget increases.

Finally, we include simulation results (in Figure \ref{fig:SBM}) to demonstrate the advantage of GRAND, as a data-perturbation paradigm, on preserving the structural over-protection.

\begin{figure}[H]
\centering
\captionsetup{type=table}%
\caption{%
  Comparison of node-level statistic distribution preservation for $\rho=0.05$ with $n=2000$ under the Latent Space Model (LSM) and Random Dot Product Graph (RDPG) model. Values are the average Wasserstein distance over 100 replications. Standard errors are omitted for brevity.
}
\label{tab:LSM-RDPG-2000-Wasserstein-rho05-noSE}
\begin{singlespace}{%
\footnotesize
  \begin{tabular}{ccc|ccc|ccc}
    \toprule
    \multirow{2}{*}{\raisebox{-0.5ex}{Metric}} & \multirow{2}{*}{\raisebox{-0.5ex}{$d$}} & \multirow{2}{*}{\raisebox{-0.5ex}{$\varepsilon$}} & \multicolumn{3}{c|}{LSM} & \multicolumn{3}{c}{RDPG} \\
    \cmidrule(lr){4-6} \cmidrule(lr){7-9}
    & & & Hat & GRAND & Laplace & Hat & GRAND & Laplace \\
    \midrule
    \multirow{8}{*}{\makecell{Node\\Degree}} & \multirow{4}{*}{3} & 1 & 0.017 & 0.050 & 2.334 & 0.017 & 0.031 & 2.037 \\
    & & 2 & 0.017 & 0.047 & 2.264 & 0.017 & 0.030 & 1.408 \\
    & & 5 & 0.017 & 0.036 & 1.924 & 0.017 & 0.028 & 0.554 \\
    & & 10 & 0.017 & 0.030 & 1.300 & 0.017 & 0.028 & 0.508 \\
    \cmidrule(lr){2-9}
    & \multirow{4}{*}{6} & 1 & 0.019 & 0.051 & 2.347 & 0.018 & 0.025 & 2.269 \\
    & & 2 & 0.019 & 0.046 & 2.332 & 0.018 & 0.024 & 2.082 \\
    & & 5 & 0.019 & 0.038 & 2.227 & 0.018 & 0.023 & 1.179 \\
    & & 10 & 0.019 & 0.033 & 1.928 & 0.018 & 0.023 & 0.524 \\
    \midrule
    \multirow{8}{*}{\makecell{V-shape\\Count}} & \multirow{4}{*}{3} & 1 & 0.034 & 0.102 & 4.700 & 0.034 & 0.065 & 4.106 \\
    & & 2 & 0.034 & 0.097 & 4.559 & 0.034 & 0.062 & 2.841 \\
    & & 5 & 0.034 & 0.074 & 3.877 & 0.034 & 0.058 & 1.129 \\
    & & 10 & 0.034 & 0.061 & 2.624 & 0.034 & 0.057 & 1.069 \\
    \cmidrule(lr){2-9}
    & \multirow{4}{*}{6} & 1 & 0.040 & 0.103 & 4.724 & 0.037 & 0.051 & 4.567 \\
    & & 2 & 0.040 & 0.094 & 4.693 & 0.037 & 0.048 & 4.191 \\
    & & 5 & 0.040 & 0.078 & 4.484 & 0.037 & 0.047 & 2.378 \\
    & & 10 & 0.040 & 0.066 & 3.885 & 0.037 & 0.048 & 1.063 \\
    \midrule
    \multirow{8}{*}{\makecell{Triangle\\Count}} & \multirow{4}{*}{3} & 1 & 0.029 & 0.132 & 6.768 & 0.037 & 0.078 & 6.530 \\
    & & 2 & 0.029 & 0.126 & 6.622 & 0.037 & 0.076 & 5.105 \\
    & & 5 & 0.029 & 0.100 & 5.899 & 0.037 & 0.066 & 2.135 \\
    & & 10 & 0.029 & 0.090 & 4.449 & 0.037 & 0.058 & 1.373 \\
    \cmidrule(lr){2-9}
    & \multirow{4}{*}{6} & 1 & 0.034 & 0.102 & 6.781 & 0.052 & 0.067 & 7.005 \\
    & & 2 & 0.034 & 0.092 & 6.750 & 0.052 & 0.062 & 6.583 \\
    & & 5 & 0.034 & 0.085 & 6.531 & 0.052 & 0.058 & 4.442 \\
    & & 10 & 0.034 & 0.084 & 5.890 & 0.052 & 0.056 & 2.049 \\
    \midrule
    \multirow{8}{*}{\makecell{Eigen\\Centrality}} & \multirow{4}{*}{3} & 1 & 0.010 & 0.041 & 0.250 & 0.031 & 0.033 & 0.140 \\
    & & 2 & 0.010 & 0.040 & 0.199 & 0.031 & 0.033 & 0.180 \\
    & & 5 & 0.010 & 0.041 & 0.108 & 0.031 & 0.034 & 0.284 \\
    & & 10 & 0.010 & 0.039 & 0.048 & 0.031 & 0.032 & 0.286 \\
    \cmidrule(lr){2-9}
    & \multirow{4}{*}{6} & 1 & 0.012 & 0.040 & 0.408 & 0.042 & 0.036 & 0.272 \\
    & & 2 & 0.012 & 0.040 & 0.293 & 0.042 & 0.036 & 0.092 \\
    & & 5 & 0.012 & 0.042 & 0.168 & 0.042 & 0.036 & 0.233 \\
    & & 10 & 0.012 & 0.036 & 0.066 & 0.042 & 0.037 & 0.339 \\
    \midrule
    \multirow{8}{*}{\makecell{Harmonic\\Centrality}} & \multirow{4}{*}{3} & 1 & 3.616 & 6.631 & 453.333 & 1.628 & 3.133 & 332.865 \\
    & & 2 & 3.616 & 6.335 & 427.719 & 1.628 & 2.990 & 169.035 \\
    & & 5 & 3.616 & 4.681 & 318.755 & 1.628 & 2.866 & \phantom{0}76.293 \\
    & & 10 & 3.616 & 3.376 & 171.327 & 1.628 & 2.891 & \phantom{0}81.309 \\
    \cmidrule(lr){2-9}
    & \multirow{4}{*}{6} & 1 & 3.736 & 6.732 & 458.477 & 1.563 & 2.144 & 427.069 \\
    & & 2 & 3.736 & 6.134 & 452.140 & 1.563 & 2.027 & 352.837 \\
    & & 5 & 3.736 & 4.812 & 411.398 & 1.563 & 2.018 & 136.861 \\
    & & 10 & 3.736 & 3.472 & 313.355 & 1.563 & 2.025 & \phantom{0}58.349 \\
    \bottomrule
  \end{tabular}%
}\end{singlespace}
\captionsetup{type=figure}%
\end{figure}

\begin{figure}[H]
\centering
\captionsetup{type=table}%
\caption{%
  Comparison of node-level statistic distribution preservation for $\rho=0.025$ with $n=2000$ under the Latent Space Model (LSM) and Random Dot Product Graph (RDPG) model. Values are the average Wasserstein distance over 100 replications. Standard errors are omitted for brevity.
}
\label{tab:LSM-RDPG-2000-Wasserstein-rho025-noSE}
\begin{singlespace}{%
\footnotesize
  \begin{tabular}{ccc|ccc|ccc}
    \toprule
    \multirow{2}{*}{\raisebox{-0.5ex}{Metric}} & \multirow{2}{*}{\raisebox{-0.5ex}{$d$}} & \multirow{2}{*}{\raisebox{-0.5ex}{$\varepsilon$}} & \multicolumn{3}{c|}{LSM} & \multicolumn{3}{c}{RDPG} \\
    \cmidrule(lr){4-6} \cmidrule(lr){7-9}
    & & & Hat & GRAND & Laplace & Hat & GRAND & Laplace \\
    \midrule
    \multirow{8}{*}{\makecell{Node\\Degree}} & \multirow{4}{*}{3} & 1 & 0.047 & 0.068 & 3.033 & 0.032 & 0.042 & 2.559 \\
    & & 2 & 0.047 & 0.063 & 2.950 & 0.032 & 0.043 & 1.715 \\
    & & 5 & 0.047 & 0.046 & 2.554 & 0.032 & 0.043 & 0.521 \\
    & & 10 & 0.047 & 0.047 & 1.804 & 0.032 & 0.047 & 0.470 \\
    \cmidrule(lr){2-9}
    & \multirow{4}{*}{6} & 1 & 0.051 & 0.074 & 3.043 & 0.036 & 0.038 & 2.921 \\
    & & 2 & 0.051 & 0.068 & 3.022 & 0.036 & 0.039 & 2.631 \\
    & & 5 & 0.051 & 0.053 & 2.895 & 0.036 & 0.039 & 1.444 \\
    & & 10 & 0.051 & 0.055 & 2.534 & 0.036 & 0.040 & 0.516 \\
    \midrule
    \multirow{8}{*}{\makecell{V-shape\\Count}} & \multirow{4}{*}{3} & 1 & 0.099 & 0.142 & 6.134 & 0.068 & 0.088 & 5.183 \\
    & & 2 & 0.099 & 0.132 & 5.968 & 0.068 & 0.091 & 3.486 \\
    & & 5 & 0.099 & 0.097 & 5.171 & 0.068 & 0.091 & 1.069 \\
    & & 10 & 0.099 & 0.097 & 3.662 & 0.068 & 0.098 & 1.005 \\
    \cmidrule(lr){2-9}
    & \multirow{4}{*}{6} & 1 & 0.108 & 0.155 & 6.151 & 0.074 & 0.078 & 5.902 \\
    & & 2 & 0.108 & 0.142 & 6.110 & 0.074 & 0.080 & 5.321 \\
    & & 5 & 0.108 & 0.112 & 5.855 & 0.074 & 0.081 & 2.935 \\
    & & 10 & 0.108 & 0.115 & 5.131 & 0.074 & 0.084 & 1.055 \\
    \midrule
    \multirow{8}{*}{\makecell{Triangle\\Count}} & \multirow{4}{*}{3} & 1 & 0.055 & 0.127 & 8.774 & 0.075 & 0.095 & 8.209 \\
    & & 2 & 0.055 & 0.120 & 8.608 & 0.075 & 0.091 & 6.208 \\
    & & 5 & 0.055 & 0.122 & 7.789 & 0.075 & 0.085 & 2.366 \\
    & & 10 & 0.055 & 0.142 & 6.137 & 0.075 & 0.080 & 1.199 \\
    \cmidrule(lr){2-9}
    & \multirow{4}{*}{6} & 1 & 0.066 & 0.129 & 8.794 & 0.127 & 0.100 & 8.981 \\
    & & 2 & 0.066 & 0.139 & 8.754 & 0.127 & 0.102 & 8.312 \\
    & & 5 & 0.066 & 0.163 & 8.501 & 0.127 & 0.092 & 5.332 \\
    & & 10 & 0.066 & 0.224 & 7.754 & 0.127 & 0.085 & 2.314 \\
    \midrule
    \multirow{8}{*}{\makecell{Eigen\\Centrality}} & \multirow{4}{*}{3} & 1 & 0.012 & 0.042 & 0.333 & 0.051 & 0.049 & 0.115 \\
    & & 2 & 0.012 & 0.039 & 0.288 & 0.051 & 0.050 & 0.187 \\
    & & 5 & 0.012 & 0.042 & 0.198 & 0.051 & 0.047 & 0.299 \\
    & & 10 & 0.012 & 0.039 & 0.069 & 0.051 & 0.052 & 0.263 \\
    \cmidrule(lr){2-9}
    & \multirow{4}{*}{6} & 1 & 0.016 & 0.045 & 0.460 & 0.075 & 0.054 & 0.303 \\
    & & 2 & 0.016 & 0.048 & 0.361 & 0.075 & 0.053 & 0.080 \\
    & & 5 & 0.016 & 0.048 & 0.242 & 0.075 & 0.054 & 0.264 \\
    & & 10 & 0.016 & 0.053 & 0.145 & 0.075 & 0.050 & 0.337 \\
    \midrule
    \multirow{8}{*}{\makecell{Harmonic\\Centrality}} & \multirow{4}{*}{3} & 1 & 13.597 & 14.552 & 579.684 & 3.548 & 6.465 & 394.221 \\
    & & 2 & 13.597 & 13.014 & 551.887 & 3.548 & 6.846 & 207.534 \\
    & & 5 & 13.597 & \phantom{0}7.125 & 433.659 & 3.548 & 6.853 & \phantom{0}44.747 \\
    & & 10 & 13.597 & \phantom{0}6.481 & 276.721 & 3.548 & 7.361 & \phantom{0}52.710 \\
    \cmidrule(lr){2-9}
    & \multirow{4}{*}{6} & 1 & 15.681 & 15.535 & 581.975 & 4.030 & 5.486 & 526.587 \\
    & & 2 & 15.681 & 13.531 & 574.048 & 4.030 & 5.613 & 419.844 \\
    & & 5 & 15.681 & \phantom{0}8.566 & 528.079 & 4.030 & 5.680 & 182.552 \\
    & & 10 & 15.681 & \phantom{0}7.184 & 419.364 & 4.030 & 6.044 & \phantom{0}56.276 \\
    \bottomrule
  \end{tabular}%
}\end{singlespace}
\captionsetup{type=figure}%
\end{figure}

\begin{figure}[H]
\centering
\captionsetup{type=table}%
\caption{%
  Comparison of node-level statistic distribution preservation for $\rho=0.1$ with $n=2000$ under the Latent Space Model (LSM) and Random Dot Product Graph (RDPG) model. Values are the average Wasserstein distance over 100 replications. Standard errors are omitted for brevity.
}
\label{tab:LSM-RDPG-2000-Wasserstein-rho1-noSE}
\begin{singlespace}{%
\footnotesize
  \begin{tabular}{ccc|ccc|ccc}
    \toprule
    \multirow{2}{*}{\raisebox{-0.5ex}{Metric}} & \multirow{2}{*}{\raisebox{-0.5ex}{$d$}} & \multirow{2}{*}{\raisebox{-0.5ex}{$\varepsilon$}} & \multicolumn{3}{c|}{LSM} & \multicolumn{3}{c}{RDPG} \\
    \cmidrule(lr){4-6} \cmidrule(lr){7-9}
    & & & Hat & GRAND & Laplace & Hat & GRAND & Laplace \\
    \midrule
    \multirow{8}{*}{\makecell{Node\\Degree}} & \multirow{4}{*}{3} & 1 & 0.008 & 0.031 & 1.632 & 0.009 & 0.026 & 1.511 \\
    & & 2 & 0.008 & 0.031 & 1.582 & 0.009 & 0.027 & 1.126 \\
    & & 5 & 0.008 & 0.026 & 1.333 & 0.009 & 0.023 & 0.518 \\
    & & 10 & 0.008 & 0.024 & 0.881 & 0.009 & 0.021 & 0.484 \\
    \cmidrule(lr){2-9}
    & \multirow{4}{*}{6} & 1 & 0.008 & 0.029 & 1.642 & 0.010 & 0.021 & 1.601 \\
    & & 2 & 0.008 & 0.030 & 1.630 & 0.010 & 0.019 & 1.486 \\
    & & 5 & 0.008 & 0.029 & 1.555 & 0.010 & 0.018 & 0.869 \\
    & & 10 & 0.008 & 0.021 & 1.336 & 0.010 & 0.017 & 0.567 \\
    \midrule
    \multirow{8}{*}{\makecell{V-shape\\Count}} & \multirow{4}{*}{3} & 1 & 0.016 & 0.063 & 3.277 & 0.019 & 0.053 & 3.036 \\
    & & 2 & 0.016 & 0.063 & 3.178 & 0.019 & 0.055 & 2.263 \\
    & & 5 & 0.016 & 0.053 & 2.677 & 0.019 & 0.047 & 1.047 \\
    & & 10 & 0.016 & 0.049 & 1.772 & 0.019 & 0.043 & 1.000 \\
    \cmidrule(lr){2-9}
    & \multirow{4}{*}{6} & 1 & 0.017 & 0.058 & 3.298 & 0.019 & 0.042 & 3.215 \\
    & & 2 & 0.017 & 0.060 & 3.274 & 0.019 & 0.039 & 2.984 \\
    & & 5 & 0.017 & 0.058 & 3.123 & 0.019 & 0.037 & 1.747 \\
    & & 10 & 0.017 & 0.043 & 2.683 & 0.019 & 0.035 & 1.145 \\
    \midrule
    \multirow{8}{*}{\makecell{Triangle\\Count}} & \multirow{4}{*}{3} & 1 & 0.015 & 0.094 & 4.799 & 0.019 & 0.066 & 4.788 \\
    & & 2 & 0.015 & 0.092 & 4.694 & 0.019 & 0.067 & 3.944 \\
    & & 5 & 0.015 & 0.080 & 4.142 & 0.019 & 0.057 & 1.853 \\
    & & 10 & 0.015 & 0.077 & 3.025 & 0.019 & 0.051 & 1.371 \\
    \cmidrule(lr){2-9}
    & \multirow{4}{*}{6} & 1 & 0.017 & 0.073 & 4.780 & 0.023 & 0.055 & 5.005 \\
    & & 2 & 0.017 & 0.077 & 4.754 & 0.023 & 0.050 & 4.753 \\
    & & 5 & 0.017 & 0.075 & 4.592 & 0.023 & 0.046 & 3.370 \\
    & & 10 & 0.017 & 0.057 & 4.101 & 0.023 & 0.042 & 1.908 \\
    \midrule
    \multirow{8}{*}{\makecell{Eigen\\Centrality}} & \multirow{4}{*}{3} & 1 & 0.011 & 0.031 & 0.164 & 0.019 & 0.023 & 0.137 \\
    & & 2 & 0.011 & 0.028 & 0.106 & 0.019 & 0.021 & 0.164 \\
    & & 5 & 0.011 & 0.031 & 0.043 & 0.019 & 0.020 & 0.209 \\
    & & 10 & 0.011 & 0.028 & 0.097 & 0.019 & 0.022 & 0.268 \\
    \cmidrule(lr){2-9}
    & \multirow{4}{*}{6} & 1 & 0.012 & 0.032 & 0.332 & 0.021 & 0.029 & 0.237 \\
    & & 2 & 0.012 & 0.033 & 0.226 & 0.021 & 0.027 & 0.120 \\
    & & 5 & 0.012 & 0.031 & 0.101 & 0.021 & 0.028 & 0.202 \\
    & & 10 & 0.012 & 0.032 & 0.045 & 0.021 & 0.026 & 0.278 \\
    \midrule
    \multirow{8}{*}{\makecell{Harmonic\\Centrality}} & \multirow{4}{*}{3} & 1 & 0.691 & 2.933 & 392.163 & 0.830 & 2.388 & 331.579 \\
    & & 2 & 0.691 & 2.927 & 372.992 & 0.830 & 2.484 & 213.128 \\
    & & 5 & 0.691 & 2.466 & 284.812 & 0.830 & 2.103 & \phantom{0}92.078 \\
    & & 10 & 0.691 & 2.299 & 157.138 & 0.830 & 1.887 & \phantom{0}67.140 \\
    \cmidrule(lr){2-9}
    & \multirow{4}{*}{6} & 1 & 0.716 & 2.575 & 398.107 & 0.877 & 1.978 & 382.348 \\
    & & 2 & 0.716 & 2.683 & 392.887 & 0.877 & 1.805 & 335.148 \\
    & & 5 & 0.716 & 2.589 & 361.959 & 0.877 & 1.687 & 165.261 \\
    & & 10 & 0.716 & 1.940 & 282.385 & 0.877 & 1.603 & \phantom{0}88.231 \\
    \bottomrule
  \end{tabular}%
}\end{singlespace}
\captionsetup{type=figure}%
\end{figure}

\begin{figure}[H]
\centering
\captionsetup{type=table}%
\caption{%
  Comparison of node-level statistic distribution preservation for $\rho=0.025$ with $n=4000$ under the Latent Space Model (LSM) and Random Dot Product Graph (RDPG) model. Values are the average Wasserstein distance over 100 replications. Standard errors are omitted for brevity.
}
\label{tab:LSM-RDPG-4000-Wasserstein-rho025-noSE}
\begin{singlespace}{%
\footnotesize
  \begin{tabular}{ccc|ccc|ccc}
    \toprule
    \multirow{2}{*}{\raisebox{-0.5ex}{Metric}} & \multirow{2}{*}{\raisebox{-0.5ex}{$d$}} & \multirow{2}{*}{\raisebox{-0.5ex}{$\varepsilon$}} & \multicolumn{3}{c|}{LSM} & \multicolumn{3}{c}{RDPG} \\
    \cmidrule(lr){4-6} \cmidrule(lr){7-9}
    & & & Hat & GRAND & Laplace & Hat & GRAND & Laplace \\
    \midrule
    \multirow{8}{*}{\makecell{Node\\Degree}} & \multirow{4}{*}{3} & 1 & 0.033 & 0.049 & 3.040 & 0.017 & 0.026 & 2.496 \\
    & & 2 & 0.033 & 0.042 & 2.956 & 0.017 & 0.025 & 1.613 \\
    & & 5 & 0.033 & 0.031 & 2.552 & 0.017 & 0.023 & 0.544 \\
    & & 10 & 0.033 & 0.035 & 1.787 & 0.017 & 0.025 & 0.504 \\
    \cmidrule(lr){2-9}
    & \multirow{4}{*}{6} & 1 & 0.040 & 0.053 & 3.050 & 0.019 & 0.021 & 2.908 \\
    & & 2 & 0.040 & 0.048 & 3.031 & 0.019 & 0.021 & 2.562 \\
    & & 5 & 0.040 & 0.036 & 2.910 & 0.019 & 0.022 & 1.308 \\
    & & 10 & 0.040 & 0.041 & 2.555 & 0.019 & 0.021 & 0.523 \\
    \midrule
    \multirow{8}{*}{\makecell{V-shape\\Count}} & \multirow{4}{*}{3} & 1 & 0.068 & 0.100 & 6.113 & 0.034 & 0.054 & 5.025 \\
    & & 2 & 0.068 & 0.085 & 5.945 & 0.034 & 0.050 & 3.253 \\
    & & 5 & 0.068 & 0.063 & 5.137 & 0.034 & 0.048 & 1.105 \\
    & & 10 & 0.068 & 0.071 & 3.603 & 0.034 & 0.051 & 1.060 \\
    \cmidrule(lr){2-9}
    & \multirow{4}{*}{6} & 1 & 0.082 & 0.108 & 6.132 & 0.039 & 0.043 & 5.846 \\
    & & 2 & 0.082 & 0.098 & 6.095 & 0.039 & 0.042 & 5.153 \\
    & & 5 & 0.082 & 0.073 & 5.852 & 0.039 & 0.044 & 2.638 \\
    & & 10 & 0.082 & 0.084 & 5.142 & 0.039 & 0.043 & 1.061 \\
    \midrule
    \multirow{8}{*}{\makecell{Triangle\\Count}} & \multirow{4}{*}{3} & 1 & 0.024 & 0.097 & 8.759 & 0.037 & 0.063 & 8.075 \\
    & & 2 & 0.024 & 0.088 & 8.587 & 0.037 & 0.058 & 5.977 \\
    & & 5 & 0.024 & 0.094 & 7.747 & 0.037 & 0.051 & 2.254 \\
    & & 10 & 0.024 & 0.116 & 6.052 & 0.037 & 0.051 & 1.340 \\
    \cmidrule(lr){2-9}
    & \multirow{4}{*}{6} & 1 & 0.024 & 0.091 & 8.783 & 0.057 & 0.054 & 8.954 \\
    & & 2 & 0.024 & 0.100 & 8.746 & 0.057 & 0.056 & 8.162 \\
    & & 5 & 0.024 & 0.120 & 8.496 & 0.057 & 0.054 & 5.015 \\
    & & 10 & 0.024 & 0.169 & 7.749 & 0.057 & 0.046 & 2.152 \\
    \midrule
    \multirow{8}{*}{\makecell{Eigen\\Centrality}} & \multirow{4}{*}{3} & 1 & 0.009 & 0.031 & 0.331 & 0.035 & 0.035 & 0.142 \\
    & & 2 & 0.009 & 0.031 & 0.292 & 0.035 & 0.034 & 0.217 \\
    & & 5 & 0.009 & 0.036 & 0.193 & 0.035 & 0.035 & 0.329 \\
    & & 10 & 0.009 & 0.037 & 0.063 & 0.035 & 0.034 & 0.297 \\
    \cmidrule(lr){2-9}
    & \multirow{4}{*}{6} & 1 & 0.012 & 0.037 & 0.466 & 0.049 & 0.038 & 0.247 \\
    & & 2 & 0.012 & 0.037 & 0.366 & 0.049 & 0.036 & 0.093 \\
    & & 5 & 0.012 & 0.040 & 0.247 & 0.049 & 0.040 & 0.309 \\
    & & 10 & 0.012 & 0.041 & 0.144 & 0.049 & 0.039 & 0.379 \\
    \midrule
    \multirow{8}{*}{\makecell{Harmonic\\Centrality}} & \multirow{4}{*}{3} & 1 & 21.910 & 23.201 & 1036.187 & 3.557 & 6.924 & 638.380 \\
    & & 2 & 21.910 & 20.024 & \phantom{0}976.434 & 3.557 & 6.257 & 278.025 \\
    & & 5 & 21.910 & 10.857 & \phantom{0}732.359 & 3.557 & 5.938 & 100.041 \\
    & & 10 & 21.910 & \phantom{0}6.869 & \phantom{0}417.292 & 3.557 & 6.585 & 138.617 \\
    \cmidrule(lr){2-9}
    & \multirow{4}{*}{6} & 1 & 23.680 & 24.162 & 1041.142 & 3.716 & 4.905 & 911.974 \\
    & & 2 & 23.680 & 21.104 & 1026.115 & 3.716 & 4.798 & 671.450 \\
    & & 5 & 23.680 & 12.597 & \phantom{0}933.849 & 3.716 & 5.245 & 218.088 \\
    & & 10 & 23.680 & \phantom{0}7.174 & \phantom{0}713.174 & 3.716 & 5.271 & \phantom{0}74.745 \\
    \bottomrule
  \end{tabular}%
}\end{singlespace}
\captionsetup{type=figure}%
\end{figure}

\begin{figure}[H]
\centering
\captionsetup{type=table}%
\caption{%
  Comparison of node-level statistic distribution preservation for $\rho=0.1$ with $n=4000$ under the Latent Space Model (LSM) and Random Dot Product Graph (RDPG) model. Values are the average Wasserstein distance over 100 replications. Standard errors are omitted for brevity.
}
\label{tab:LSM-RDPG-4000-Wasserstein-rho1-noSE}
\begin{singlespace}{%
  \footnotesize
  \begin{tabular}{ccc|ccc|ccc}
    \toprule
    \multirow{2}{*}{\raisebox{-0.5ex}{Metric}} & \multirow{2}{*}{\raisebox{-0.5ex}{$d$}} & \multirow{2}{*}{\raisebox{-0.5ex}{$\varepsilon$}} & \multicolumn{3}{c|}{LSM} & \multicolumn{3}{c}{RDPG} \\
    \cmidrule(lr){4-6} \cmidrule(lr){7-9}
    & & & Hat & GRAND & Laplace & Hat & GRAND & Laplace \\
    \midrule
    \multirow{8}{*}{\makecell{Node\\Degree}} & \multirow{4}{*}{3} & 1 & 0.004 & 0.020 & 1.635 & 0.005 & 0.020 & 1.516 \\
    & & 2 & 0.004 & 0.019 & 1.588 & 0.005 & 0.019 & 1.138 \\
    & & 5 & 0.004 & 0.017 & 1.341 & 0.005 & 0.016 & 0.511 \\
    & & 10 & 0.004 & 0.016 & 0.895 & 0.005 & 0.015 & 0.473 \\
    \cmidrule(lr){2-9}
    & \multirow{4}{*}{6} & 1 & 0.004 & 0.018 & 1.644 & 0.005 & 0.013 & 1.598 \\
    & & 2 & 0.004 & 0.017 & 1.634 & 0.005 & 0.013 & 1.471 \\
    & & 5 & 0.004 & 0.015 & 1.565 & 0.005 & 0.013 & 0.819 \\
    & & 10 & 0.004 & 0.014 & 1.358 & 0.005 & 0.011 & 0.593 \\
    \midrule
    \multirow{8}{*}{\makecell{V-shape\\Count}} & \multirow{4}{*}{3} & 1 & 0.008 & 0.041 & 3.276 & 0.010 & 0.040 & 3.039 \\
    & & 2 & 0.008 & 0.039 & 3.182 & 0.010 & 0.038 & 2.283 \\
    & & 5 & 0.008 & 0.035 & 2.689 & 0.010 & 0.032 & 1.027 \\
    & & 10 & 0.008 & 0.032 & 1.796 & 0.010 & 0.031 & 0.962 \\
    \cmidrule(lr){2-9}
    & \multirow{4}{*}{6} & 1 & 0.009 & 0.037 & 3.295 & 0.010 & 0.027 & 3.202 \\
    & & 2 & 0.009 & 0.035 & 3.274 & 0.010 & 0.026 & 2.948 \\
    & & 5 & 0.009 & 0.030 & 3.136 & 0.010 & 0.026 & 1.643 \\
    & & 10 & 0.009 & 0.028 & 2.722 & 0.010 & 0.023 & 1.192 \\
    \midrule
    \multirow{8}{*}{\makecell{Triangle\\Count}} & \multirow{4}{*}{3} & 1 & 0.008 & 0.062 & 4.796 & 0.010 & 0.050 & 4.788 \\
    & & 2 & 0.008 & 0.060 & 4.695 & 0.010 & 0.048 & 3.961 \\
    & & 5 & 0.008 & 0.055 & 4.148 & 0.010 & 0.040 & 1.845 \\
    & & 10 & 0.008 & 0.052 & 3.050 & 0.010 & 0.038 & 1.358 \\
    \cmidrule(lr){2-9}
    & \multirow{4}{*}{6} & 1 & 0.010 & 0.049 & 4.777 & 0.011 & 0.034 & 5.009 \\
    & & 2 & 0.010 & 0.043 & 4.754 & 0.011 & 0.034 & 4.735 \\
    & & 5 & 0.010 & 0.037 & 4.604 & 0.011 & 0.033 & 3.274 \\
    & & 10 & 0.010 & 0.037 & 4.139 & 0.011 & 0.028 & 1.920 \\
    \midrule
    \multirow{8}{*}{\makecell{Eigen\\Centrality}} & \multirow{4}{*}{3} & 1 & 0.007 & 0.028 & 0.164 & 0.013 & 0.016 & 0.133 \\
    & & 2 & 0.007 & 0.025 & 0.103 & 0.013 & 0.015 & 0.161 \\
    & & 5 & 0.007 & 0.023 & 0.044 & 0.013 & 0.017 & 0.214 \\
    & & 10 & 0.007 & 0.025 & 0.100 & 0.013 & 0.016 & 0.277 \\
    \cmidrule(lr){2-9}
    & \multirow{4}{*}{6} & 1 & 0.008 & 0.022 & 0.328 & 0.014 & 0.023 & 0.213 \\
    & & 2 & 0.008 & 0.024 & 0.228 & 0.014 & 0.021 & 0.137 \\
    & & 5 & 0.008 & 0.025 & 0.099 & 0.014 & 0.022 & 0.213 \\
    & & 10 & 0.008 & 0.022 & 0.048 & 0.014 & 0.021 & 0.287 \\
    \midrule
    \multirow{8}{*}{\makecell{Harmonic\\Centrality}} & \multirow{4}{*}{3} & 1 & 0.738 & 3.886 & 784.981 & 0.844 & 3.485 & 666.958 \\
    & & 2 & 0.738 & 3.722 & 747.901 & 0.844 & 3.310 & 438.446 \\
    & & 5 & 0.738 & 3.381 & 571.464 & 0.844 & 2.769 & 172.796 \\
    & & 10 & 0.738 & 3.109 & 317.947 & 0.844 & 2.600 & 114.992 \\
    \cmidrule(lr){2-9}
    & \multirow{4}{*}{6} & 1 & 0.765 & 3.337 & 796.937 & 0.909 & 2.489 & 761.823 \\
    & & 2 & 0.765 & 3.116 & 787.725 & 0.909 & 2.435 & 659.242 \\
    & & 5 & 0.765 & 2.683 & 729.535 & 0.909 & 2.404 & 316.235 \\
    & & 10 & 0.765 & 2.524 & 575.932 & 0.909 & 2.084 & 170.016 \\
    \bottomrule
  \end{tabular}%
}\end{singlespace}
\captionsetup{type=figure}%
\end{figure}

\begin{figure}[H]
  \centering
\includegraphics[width=0.7\textwidth]{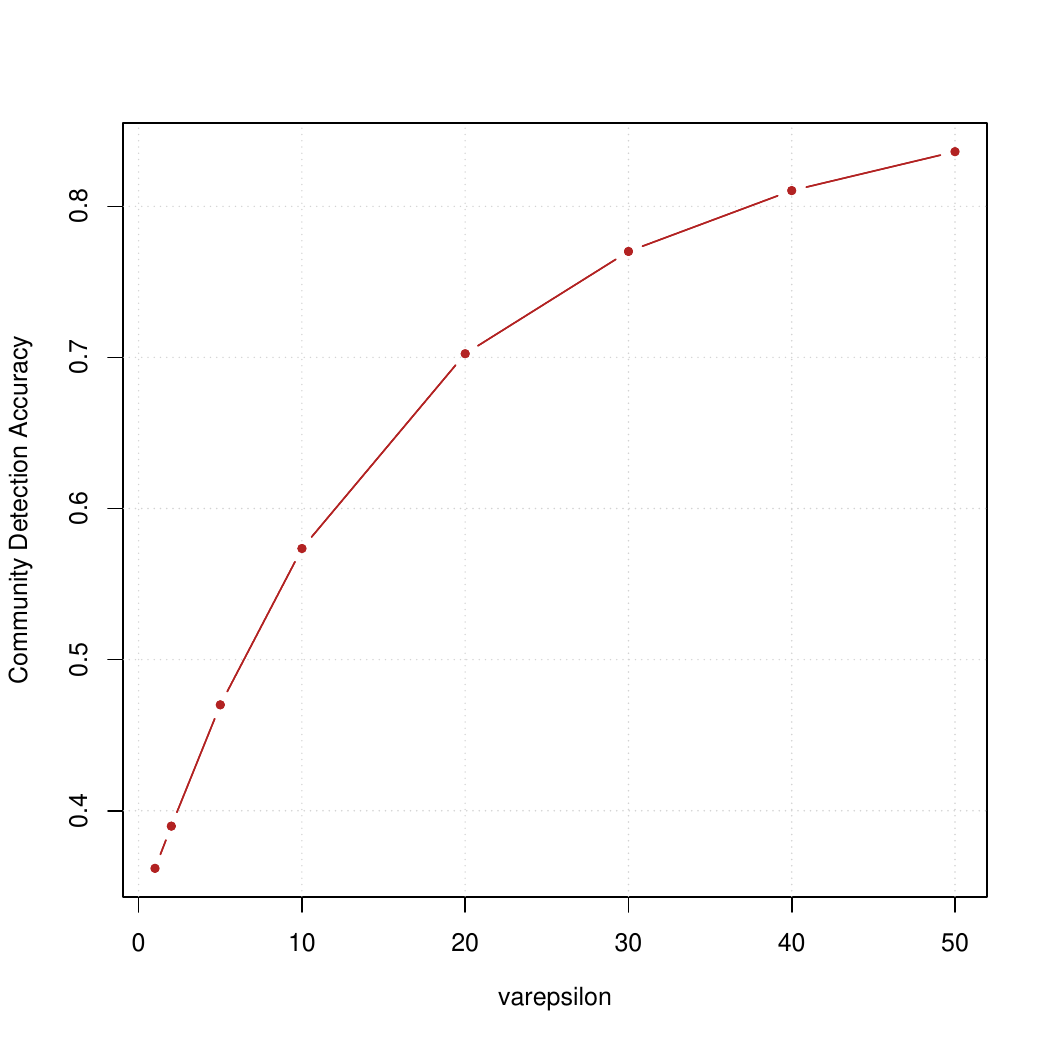}
\caption{\label{fig:SBM}The recovery accuracy of the true community labels from the GRAND-privatized networks. The network is generated from SBM with 3 communities, with size 4000 and average degree 200. }
\end{figure}

\end{appendix}

\end{document}